\documentclass[pdflatex,sn-mathphys-num]{sn-jnl}

\usepackage{graphicx}%
\usepackage{multirow}%
\usepackage{amsmath,amssymb,amsfonts}%
\usepackage{amsthm}%
\usepackage{mathrsfs}%
\usepackage[title]{appendix}%
\usepackage{xcolor}%
\usepackage{textcomp}%
\usepackage{manyfoot}%
\usepackage{booktabs}%
\usepackage{algorithm}%
\usepackage{algorithmicx}%
\usepackage{algpseudocode}%
\usepackage{listings}%
\usepackage{subfigure}
\usepackage{makecell}

\theoremstyle{thmstyleone}%
\newtheorem{theorem}{Theorem}
\newtheorem{proposition}[theorem]{Proposition}%
\newtheorem{corollary}[theorem]{Corollary}
\newtheorem{lemma}[theorem]{Lemma}

\theoremstyle{thmstyletwo}%
\newtheorem{remark}{Remark}%

\theoremstyle{thmstylethree}%
\newtheorem{definition}{Definition}%

\begin{document}

\title[Mathematical Modeling and Convergence Analysis of Deep Neural Networks with Dense Layer Connectivities in Deep Learning]{Mathematical Modeling and Convergence Analysis of Deep Neural Networks with Dense Layer Connectivities in Deep Learning}


\author[1]{\fnm{Jinshu} \sur{Huang}}\email{huangjsh@mail.nankai.edu.cn}

\author[1]{\fnm{Haibin} \sur{Su}}\email{hbsu@mail.nankai.edu.cn}

\author[2]{\fnm{Xue-Cheng} \sur{Tai}}\email{xtai@norceresearch.no}

\author*[1]{\fnm{Chunlin} \sur{Wu}}\email{wucl@nankai.edu.cn}

\affil[1]{\orgdiv{School of Mathematical Sciences}, \orgname{Nankai University}, \orgaddress{\city{Tianjin}, \country{China}}}

\affil[2]{\orgname{Norwegian Research Center (NORCE)}, \orgaddress{\street{Nyg\r{a}rdsgaten 112}, \postcode{5008}, \city{Bergen}, \country{Norway}}}



\abstract{In deep learning, dense layer connectivity has become a key design principle in deep neural networks (DNNs), enabling efficient information flow and strong performance across a range of applications. In this work, we model densely connected DNNs mathematically and analyze their learning problems in the deep-layer limit. For a broad applicability, we present our analysis in a framework setting of DNNs with densely connected layers and general non-local feature transformations (with local feature transformations as special cases) within layers, which is called dense non-local (DNL) framework and includes standard DenseNets and variants as special examples.
	In this formulation, the densely connected networks are modeled as nonlinear integral equations, in contrast to the ordinary differential equation viewpoint commonly adopted in prior works. We study the associated training problems from an optimal control perspective and prove convergence results from the network learning problem to its continuous-time counterpart. In particular, we show the convergence of optimal values and the subsequence convergence of minimizers, using a piecewise linear extension and $\Gamma$-convergence analysis.
	Our results provide a mathematical foundation for understanding densely connected DNNs and further suggest that such architectures can offer stability of training deep models.}




\maketitle

\section{Introduction}\label{sec1}

Deep learning technology has achieved significant breakthroughs in various fields \cite{goodfellow2016deep}. 
A fundamental and key factor behind these achievements is the design of deep neural network (DNN) architectures.
It determines how neurons are connected and how information flows through the neural network, which plays a crucial role in the expressive ability of DNNs and the efficiency of training algorithms.

The architecture of DNNs includes the layer connectivity methods and the computation schemes within layers. Regarding layer connections, the early proposed fully connected feedforward neural networks (FNNs) and convolutional neural networks (CNNs) \cite{goodfellow2016deep} utilize a straightforward sequential connection structure. 
Such structures are compact and computationally simple, but they may become harder to train with an increasing number of layers and less effective at capturing complex dependencies. Later on, the residual networks (ResNets) \cite{he2016deep} introduced skip connections and residual learning, which mitigate the vanishing gradient problem and enable the training of much deeper networks. After that, the dense convolutional network (DenseNet) \cite{Huang2016DenselyCC} further enhanced information flow by connecting each layer to every other layer in a feedforward manner. This dense connectivity structure promotes feature reuse and thus can reduce the number of parameters, leading to more efficient models. 
These layer connectivity methods have led to the development of many effective deep neural networks. For instance, 
UNet \cite{ronneberger2015u} extends the typical CNN by incorporating an encoder-decoder structure and skip connections; \cite{zhang2020dense, yao2022dense, ma2023denseformer} combined the dense layer connectivity methods with some attention mechanisms to get more effective DNN architectures.
For more research in this direction, see, e.g., ResUNet \cite{diakogiannis2020resunet}, DenseUNet \cite{cao2020denseunet}, PottsMGNet \cite{tai2024pottsmgnet}.

The computation scheme within each layer also plays an important role in the design of neural network architectures. There is a broad range of local and non-local feature transformations that can benefit the performances of DNNs.
Initially, FNNs use affine transformations combined with nonlinear activations. Their fully connected layers are straightforward and capture global information within layers. 
After that, CNN layers introduce sparse convolution operations, applying filters to local regions of the state, and can effectively capture spatial hierarchies through multiple layers.
	Beyond these classical designs, more general feature transformation schemes have been explored, such as nonlinear operations beyond activation functions \cite{vaswani2017attention, Wang2017NonlocalNN, jia2020nonlocal, meng2024learnable}. For example, the STD (Soft-Threshold-Dynamics) activation layer introduced in \cite{liu2022deep}, which offered a mechanism to incorporate many well-known variational models into activation functions. The Transformers \cite{vaswani2017attention} employ self-attention mechanisms to model global dependencies. In each self-attention function, the output is calculated as a weighted sum of the values obtained by an affine transformation of inputs, where the weight assigned to each value is computed by the {\rm softmax($\cdot$)} function \cite{gao2017properties}. Additionally, non-local neural networks (Non-local Nets) \cite{Wang2017NonlocalNN} expand the self-attention mechanism by computing interactions between all possible pairs of positions through a flexible non-local kernel, which can further capture global dependencies such as similarity.

Despite the great successes of these DNN architectures in applications, they were mainly handcrafted and lacked rigorous mathematical understanding. We note that the connectivity between layers and the computation scheme within each layer corresponds exactly to a time variable and a space variable, respectively. Naturally, we can consider their dynamical system modeling.
Such dynamical system modeling and analysis aim to study an associated continuous-time formulation in some detail, providing new perspectives and tools for theoretical research on DNN structures and network learning problems. 
Indeed, the works \cite{weinan2017proposal, haber2017stable} first established a connection between ResNets and ordinary differential equations (ODEs) by interpreting the forward propagation of the network as a time-discretization of an ODE, with each layer corresponding to a discrete time step.
Lu et al. \cite{Lu18Beyond} further modeled some effective networks, such as RevNet \cite{gomez2017reversible}, as different numerical discretizations of ODEs. Recently, Zhang and Schaeffer \cite{zhang2020forward} represented the ResNets as an ordinary differential inclusion (ODI) to analyze its global stability property. Thorpe and Van Gennip \cite{thorpe2018deep} demonstrated some convergence results from the discrete-time to continuous-time learning problems for ResNets. For more continuous-time modeling and discussion of classic DNNs, one can see, e.g., \cite{chen2018neural, lu2019understanding, sherstinsky2020fundamentals, song2021scorebased}.
We also mention that, for some deep unrolling/unfolding networks (see, e.g., \cite{monga2021algorithm} for a survey), it is natural to connect them as continuous-time systems and study the convergence of the associated learning problems in the deep-layer limit setting \cite{huang2024on, lin2022deep, lin2024deep}.
These dynamical system modeling and analysis methods not only provide a mathematical understanding but also aid in designing new network architectures through various numerical ordinary/partial differential equations or optimization schemes \cite{Lu18Beyond, haber2019imexnet, ruthotto2020deep, tai2024pottsmgnet}.

DNNs with dense layer connectivities, such as DenseNet \cite{Huang2016DenselyCC} and its variants \cite{zhang2020dense, cao2020denseunet, yao2022dense, ma2023denseformer}, have demonstrated excellent empirical performance across a wide range of deep learning applications. Despite their successes, a precise mathematical understanding of how such dense connections influence network behavior, particularly in the limit of infinite depth, remains unclear.
In this work, we model and analyze densely connected DNNs through the dynamical system approach. Our analysis is presented within a general framework, namely DNL framework. Note that our goal is not to propose a new architecture, but rather to provide a unified theoretical perspective that applies across this class of DNNs. \textit{Our central observation is that, with each network layer corresponding to a time step, dense layer connectivity yields a discrete accumulation across layers and naturally gives rise to a nonlinear integral equation in the deep-layer limit}. Our main contributions are as follows:
\begin{itemize}
	\item [(i)] We provide a dynamical system modeling for a class of densely connected DNNs. Specifically, within a general dense non-local framework, we formulate the DNNs in the deep-layer limit using nonlinear integral equations, which stands in contrast to prior works that typically model deep networks as ODEs or ODIs. We further cast the corresponding learning problems in both discrete- and continuous-time settings as optimal control problems with suitable regularization.

    \item [(ii)] We establish convergence results from the learning problem of the discrete-time DNL framework to its continuous-time counterpart via $\Gamma$-convergence. In particular, we show convergence of the optimal values and subsequence convergence of the optimal solutions using a piecewise linear extension technique.
\end{itemize}

These dynamical system modeling and convergence analysis offer a mathematical understanding and evidence of DNN architectures with dense layer connectivities from the perspective of integral equations. 
Our analyses suggest that densely connected architectures may offer enhanced stability when training deep networks. 
Moreover, our results can be directly applied to, or analogously extended to, various DNNs with dense layer connections, such as DenseNet, DenseFormer \cite{ma2023denseformer}, and Dense Residual Transformer \cite{yao2022dense}.

The remainder of this paper is organized as follows. 
Section~\ref{Sec:2} introduces the dense non-local framework of DNNs. In Section~\ref{Sec:3}, we present dynamical system modeling for the framework and our main theoretical result. Section~\ref{Sec:4} provides detailed proofs. In Section~\ref{Sec:5}, we present some simple numerical experiments to validate the theoretical result. Finally, Section~\ref{Sec:6} concludes the paper.

\section{A dense non-local (DNL) framework of deep neural networks}
\label{Sec:2}
We denote by $\mathbb{N}$ the set of natural numbers and by $\mathbb{R}$ the field of real numbers. 
Scalars are denoted with italic letters, like $L, n \in \mathbb{N}$. Vectors are represented by lowercase letters and matrices by uppercase letters, all in upright font, for example, ${\rm a} \in \mathbb{R}^L$ and ${\rm U} \in \mathbb{R}^{n \times n}$.
We also denote the $l^2$-norm for vectors in Euclidean space by $|\cdot|$ and the spectral norm for matrices by $\|\cdot\|_{2}$. Without confusion, we abbreviate $\|\cdot\|_2$ as $\|\cdot\|$.   
For a space $\mathbb{X}$, the notation $(\mathbb{X})^L$ represents the Cartesian product space of $\underbrace{\mathbb{X} \times \ldots \times \mathbb{X}}_{L}$. To facilitate the description of the network architecture, we define a general affine map.
\begin{definition}
    Given a sequence of matrix-vector pairs $\{(\mathrm{M}^k, \mathrm{v}^k )\}_{k=0}^{l} \subset \mathbb{R}^{n\times n} \times \mathbb{R}^{n}$, we define an affine map 
    $
    \bold{Aff}_{\{(\mathrm{M}^k, \mathrm{v}^k )\}_{k=0}^{l}}: \underbrace{\mathbb{R}^n \times \ldots \times \mathbb{R}^n}_{l+1} \to \mathbb{R}^n
    $ as
    \[
    \bold{Aff}_{\{(\mathrm{M}^k, \mathrm{v}^k )\}_{k=0}^{l}}(\mathrm{x}^0, \ldots, \mathrm{x}^l) := \sum_{k=0}^l \left( \mathrm{M}^k \mathrm{x}^k + \mathrm{v}^k \right).
    \]
\end{definition}
Such an affine map covers various types of linear operations commonly used in practice, including standard convolutions or matrix-vector multiplications in FNNs.

We now introduce the DNL framework and its associated learning problem.
Since our formulation is not restricted to specific data modalities, we assume without loss of generality that both the input and output lie in $\mathbb{R}^n$. 
Given a layer number $L \in \mathbb{N}$, the DNL framework is defined as a class of densely connected neural networks with $L$ layers. Each layer aggregates all preceding outputs through dense skip connections and applies a transformation that may be either local or non-local. Formally, the architecture is given by
\begin{equation}
	\begin{aligned}
 \mathrm{x}^{0}
		 & = \bold{Aff}_{ \{(\mathrm{V}_L^{0},\mathrm{b}_L^{0})\}} \Big( \pmb{\phi} \circ \bold{Aff}_{ \{(\mathrm{U}_L^{0},\mathrm{a}_L^{0}) \}} \big(\pmb{\mathcal{A}} _{\pmb{\kappa}} (\mathrm{T}_L^{0}; \mathrm{d})\big)  \Big) \\
   &= \mathrm{V}_L^{0} \pmb{\phi} \circ ( \mathrm{U}_L^{0} \pmb{\mathcal{A}} _{\pmb{\kappa} } (\mathrm{T}_L^{0}; \mathrm{d}) + \mathrm{a}_L^{0} ) + \mathrm{b}_L^{0}, \\
		\mathrm{x}^{l}
  & = \bold{Aff}_{\{(\mathrm{V}_L^{l},\mathrm{b}_L^{l})\}} \! \Big( \pmb{\phi} \circ   \bold{Aff}_{ \{ (\mathrm{U}_L^{l},\mathrm{a}_L^{l}) \} \cup \{(\tau\mathrm{W}_L^{l,k+1},\tau\mathrm{c}_L^{l,k+1})\}_{k=0}^{l-1} } \\
  & \qquad \qquad \qquad \qquad \qquad  \big(\pmb{\mathcal{A}} _{\pmb{\kappa} } (\mathrm{T}_L^{l};\mathrm{d}), \pmb{\mathcal{A}} _{\pmb{\kappa} } (\mathrm{T}_L^{l}; \mathrm{x}^0), \ldots, \pmb{\mathcal{A}} _{\pmb{\kappa} } (\mathrm{T}_L^{l}; \mathrm{x}^{l-1}) \big) \! \Big)  \\
		 & = \mathrm{V}_L^{l}  \pmb{\phi} \!\circ \!\!\Big(\mathrm{U}_L^{l} \pmb{\mathcal{A}} _{\pmb{\kappa}} (\mathrm{T}_L^{l};\mathrm{d}) \!+ \!\mathrm{a}_L^{l}\! + \!\tau \sum_{k=0}^{l-1} \!\big[\mathrm{W}_L^{l,k+1}\pmb{\mathcal{A}} _{\pmb{\kappa} }(\mathrm{T}_L^{l};\mathrm{x}^k) \!+  \! \mathrm{c}_L^{l,k+1} \big] \Big)\!  + \! \mathrm{b}_L^{l}, 
	\end{aligned}
	\label{equation: DNLF}
\end{equation}
\begin{figure}[htbp]
	\centering
	\includegraphics[scale=0.42]{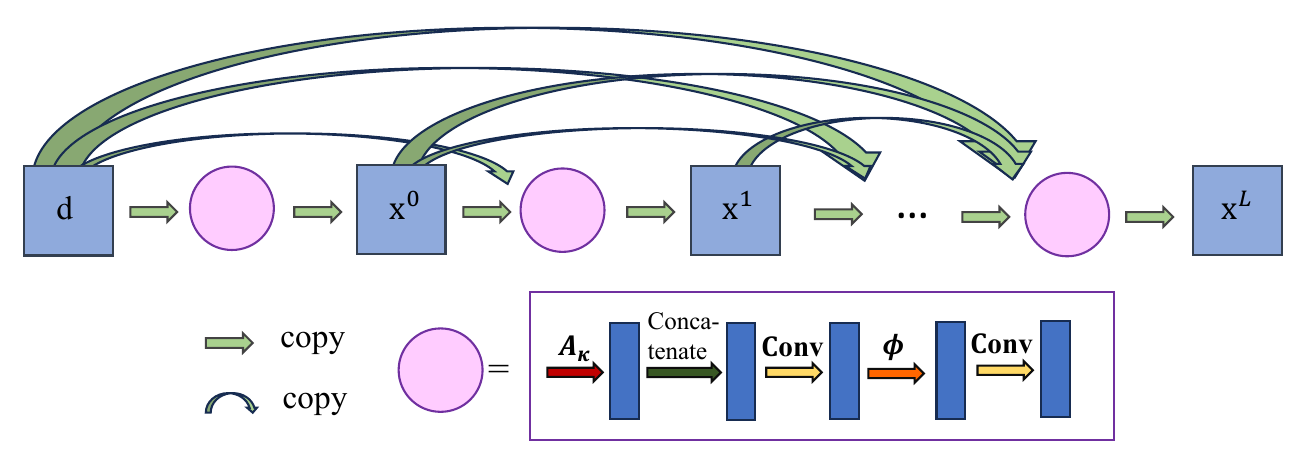}
	\caption{Network architecture associated with \eqref{equation: DNLF}. }
		\label{figure: net arch}
\end{figure}
for $ 1 \leq l \leq L$. Therein, the input ${\rm d} \in \mathbb{R}^{n}$; the affine parameters include $(\mathrm{U}^{l}_L , \mathrm{a}^{l}_L) \in \mathbb{R}^{n\times n} \times \mathbb{R}^{n},\
(\mathrm{V}^{l}_L, \mathrm{b}^{l}_L) \in  \mathbb{R}^{n\times n} \times \mathbb{R}^{n}$ ($0 \leq l \leq L$)  and $(\mathrm{W}^{l}_L, \mathrm{c}^{l}_L) \in  (\mathbb{R}^{n\times n})^{l} \times (\mathbb{R}^{n} )^{l}$ ($1 \leq l \leq L$),  where $\mathrm{W}_L^l = (\mathrm{W}_L^{l,1}, \ldots, \mathrm{W}_L^{l,l}) \in (\mathbb{R}^{n\times n})^l$ and $\mathrm{c}_L^l = (\mathrm{c}_L^{l,1}, \ldots, \mathrm{c}_L^{l,l}) \in (\mathbb{R}^n)^l$ due to the dense connectivity. Since $(\mathrm{W}^{l}_L,\mathrm{c}^{l}_L)$ is learnable,
we extract a step size $\tau>0$ and still denote it as $(\mathrm{W}^{l}_L,\mathrm{c}^{l}_L)$ without loss of generality. The activation function $\pmb{\phi}$ is applied component-wise, and $\pmb{\mathcal{A}} _{\pmb{\kappa} }: (\mathbb{R}^{n\times n})^3 \times \mathbb{R}^n \rightarrow \mathbb{R}^n$ is a parametrized general non-local transformation with kernel $\pmb{\kappa}$ \cite{buades2005review, Wang2017NonlocalNN}. 
For any vector $\mathrm{z} \in \mathbb{R}^n$, 
\begin{equation}
    \pmb{\mathcal{A}}_{\pmb{\kappa}}(\mathrm{T}_L^{l}; \mathrm{z})[i] 
  = \frac{1}{\text{Normalize}(\mathrm{T}_L^{l};\mathrm{z})[i]} \sum_{j=1}^{n} \pmb{\kappa} (((\mathrm{T}_1)_{L}^{l}\mathrm{z})[i], ((\mathrm{T}_2)_{L}^{l}\mathrm{z})[j])\cdot ((\mathrm{T}_3)_{L}^{l}\mathrm{z})[j],
  \nonumber
\end{equation}
for $1 \leq i \leq n$, where $\mathrm{T}^{l}_L\!=\!((\mathrm{T}_1)^{l}_{L}, (\mathrm{T}_2)^{l}_{L}, (\mathrm{T}_3)^{l}_{L}) \in (\mathbb{R}^{n\times n})^3$ is the learnable parameter, the non-negative kernel $\pmb{\kappa}(\cdot, \cdot)$ is a scalar value that quantifies the degree of relation or relevance between the features at positions $i$ and $j$. The normalization factor $\text{Normalize}(\cdot; \cdot) \not = 0$, which is commonly taken as $\text{Normalize}(\mathrm{T}_L^{l};\mathrm{z})[i] = \sum_{j=1}^{n} \pmb{\kappa} (((\mathrm{T}_1)_{L}^{l}\mathrm{z})[i], ((\mathrm{T}_2)_{L}^{l}\mathrm{z})[j])$. Figure~\ref{figure: net arch} illustrates the DNN architecture in (\ref{equation: DNLF}).

Note that the DNL framework is quite general due to the flexibility of kernel $\pmb{\kappa}$ and parameters $\mathrm{W}^{l}_L, \mathrm{c}^{l}_L$ ($1\leq l \leq L$). It characterizes a broad class of densely connected neural networks, encompassing several important examples and closely related variants as follows
\begin{itemize}
    \item If $\pmb{\kappa}(\mathrm{z}[i], \mathrm{z}[j]) = \pmb{\delta}_{ij}$ with $\pmb{\delta}_{i,j}$ being the Kronecker function, then $\pmb{\mathcal{A}} _{\pmb{\kappa}}$ is a local operation, and the DNL framework reduces to a DenseNet \cite{Huang2016DenselyCC}.
  \item If $\pmb{\kappa} (((\mathrm{T}_1)_{ L}^{l}\mathrm{z})[i],  ((\mathrm{T}_2)_{L}^{l}\mathrm{z})[j]) = \exp{\big( ((\mathrm{T}_1)_{L}^{l}\mathrm{z})[i] \cdot ((\mathrm{T}_2)_{L}^{l}\mathrm{z})[j] /\sqrt{n} \big)}$ and \\
  $\text{Normalize}(\mathrm{T}_L^{l};\mathrm{z})[i] = \sum_{j=1}^{n} \!\pmb{\kappa} (((\mathrm{T}_1)_{L}^{l}\mathrm{z})[i], ((\mathrm{T}_2)_{L}^{l}\mathrm{z})[j])$, the function $\pmb{\mathcal{A}}_{\pmb{\kappa}}$ admits the self-attention operation in Transformer, and the DNL framework is a kind of Dense-Attention structure, which is closely similar to \cite{zhang2020dense, yao2022dense, ma2023denseformer}.
    \item If $(\mathrm{W}_L^{l,k},\mathrm{c}_L^{l,k}) \equiv \textbf{0}$, $1\leq k \leq \lceil l/2 \rceil -1, 2\leq l \leq L$, then the DNL framework coincides with a partially densely connected neural network.
\end{itemize}

To facilitate the subsequent analysis, we package the learnable parameters together and introduce the notation ${\Theta}_L$ and the discrete learnable parameter set $\Omega_{{\Theta}; L}$ as follows
    \begin{align}
 {\Theta}_L := & (\mathrm{T}_L, \mathrm{U}_L, \mathrm{a}_L, \mathrm{V}_L, \mathrm{b}_L, \mathrm{W}_L,\mathrm{c}_L) \in \Omega_{{\Theta};L},  \label{notation: discrete learnable parameter set} \\
    \Omega_{{\Theta};L}  := & ((\mathbb{R}^{n\times n})^3)^{L+1} \! \times \! (\mathbb{R}^{n\times n})^{L+1} \! \times \! (\mathbb{R}^{n})^{L+1}  \! \times  \!
(\mathbb{R}^{n\times n})^{L+1} \! \times \!(\mathbb{R}^{n})^{L+1} \nonumber\\
& \times \! (\mathbb{R}^{n\times n})^{(1/2)L(L+1)} \! \times \! (\mathbb{R}^{n})^{(1/2)L(L+1)}, \nonumber
\end{align}
where  $\mathrm{T}_L = ((\mathrm{T}_1)_L, (\mathrm{T}_2)_L, \ (\mathrm{T}_3)_L), (\mathrm{T}_i)_L = ((\mathrm{T}_i)_L^0, \ldots, (\mathrm{T}_i)_L^L) \ (i=1,2,3)$, $\mathrm{U}_L = (\mathrm{U}_L^0, \ldots, \mathrm{U}_L^L)$, $\mathrm{a}_L = (\mathrm{a}_L^0, \ldots, \mathrm{a}_L^L)$, $\mathrm{V}_L = (\mathrm{V}_L^0, \ldots, \mathrm{V}_L^L)$, $\mathrm{b}_L = (\mathrm{b}_L^0,\ldots, \mathrm{b}_L^L)$, $\mathrm{W}_L = (\mathrm{W}_L^1, \ldots, \mathrm{W}_L^L)$ and $\mathrm{c}_L = (\mathrm{c}_L^1, \ldots, \mathrm{c}_L^L)$. It is natural to define a norm for $\Omega_{{\Theta};L}$ as 
\begin{equation}
	\|{\Theta}_{L}\|_{\Omega_{{\Theta};L}} := \max \{\|\mathrm{T}_L\|_{\mathrm{T};L}, \!
 \|\mathrm{U}_L\|_{\mathrm{U};L}, \!	\|\mathrm{a}_L\|_{\mathrm{a};L}, \!
 \|\mathrm{V}_L\|_{\mathrm{V};L}, \! 	\|\mathrm{b}_L\|_{\mathrm{b};L},\! \|\mathrm{W}_L\|_{\mathrm{W};L},\|\mathrm{c}_L\|_{\mathrm{c};L}\},
\end{equation}
where 
$$
\begin{aligned}
& \|\mathrm{T}_L\|_{\mathrm{T};L} = \max_{ i \in \{1, 2, 3\} } \max_{l \in \{0, 1, \ldots, L\}} \| (\mathrm{T}_i)_{L}^l\| , \|\mathrm{U}_L\|_{\mathrm{U};L} = \max_{l \in \{0, 1, \ldots, L\}} \| \mathrm{U}_{{L}}^l\|, \\
&\|\mathrm{a}_L\|_{\mathrm{a};L} = \max_{l \in \{0, 1, \ldots, L\}} | \mathrm{a}_{{L}}^l|, \|\mathrm{V}_L\|_{\mathrm{V};L} = \max_{l \in \{0, 1, \ldots, L\}} \| \mathrm{V}_{{L}}^l\|, \ \|\mathrm{b}_L\|_{\mathrm{b};L} = \max_{l \in \{0, 1, \ldots, L\} } | \mathrm{b}_{{L}}^l|, \\
& \|\mathrm{W}_L\|_{\mathrm{W};L} = \max_{l \in \{1, \ldots, L\}} \max_{ k \in \{1, \ldots, l\}}  \| \mathrm{W}_{{L}}^{l,k}\|, \ \|\mathrm{c}_L\|_{\mathrm{c};L} = \max_{l \in \{1, \ldots, L\}} \max_{ k \in \{1, \ldots, l\}} | \mathrm{c}_{{L}}^{l,k}|.
\end{aligned}
$$
Without ambiguity, we will sometimes use $\|{\Theta}_{L}\|$ to denote $\|{\Theta}_{L}\|_{\Omega_{{\Theta};L}}$ for simplicity. 

For the DNL framework, we now introduce its learning problem.
Denote $\mathrm{x}^l(\mathrm{d}; \Theta_L)$ as the output of $l$-th ($1 \leq l \leq L$) layer of (\ref{equation: DNLF}) with input data $\mathrm{d} \in \mathbb{R}^n$ and learnable parameter $\Theta_L$.
Let $\pmb{\ell} : \mathbb{R}^{n} \times \mathbb{R}^{n} \rightarrow \mathbb{R}^+$ be the data loss function.
Deep learning aims to minimize a loss function ${\pmb{\mathfrak{L}}_{\mathcal{S};L}}({\Theta_L})$ on a training dataset $\mathcal{S} := \{(\mathrm{d}_m, \mathrm{g}_m)\}_{m=1}^M$, where $\mathrm{d}_m$ is the observed signal and $\mathrm{g}_m$ is the corresponding ground truth. 
Consequently, we have the following optimal control problem
\begin{equation}
	\begin{aligned}
	(\mathcal{P}_L):	\left\{ \begin{array}{l}
				\inf \limits_{\Theta_L \in \Omega_{\Theta;L}} \left\{ {\pmb{\mathfrak{L}}_{\mathcal{S};L}}({\Theta}_L) = \frac{1}{M} \sum\limits_{m=1}^{M} \pmb{\ell} (\mathrm{x}^L(\mathrm{d}_m; \Theta_L), \mathrm{g}_m) + \pmb{\mathcal{R}}_{L}(\Theta_L) \right\} \\
				\text{ subject to:} \\
			   \{\mathrm{x}^l(\mathrm{d}_m; \Theta_L) \}_{l=0}^{L}  {\rm \ satisfy \ Eq. (\ref{equation: DNLF}) \ with \ input \ data \ } \mathrm{d}_m { \rm \ and \ parameter \ } \Theta_L,
		\end{array} \right.\\
	\end{aligned}
	\label{discrete-time-control problem P_L}
\end{equation} 
where $\pmb{\mathcal{R}}_{L}: \Omega_{{\Theta};L} \rightarrow [0, \infty)$ is a regularization function. 
In neural network training, it is common to employ regularization strategies \cite{thorpe2018deep, wei2019regularization, zhang2020forward, huang2024on}. To clearly define the regularization term $\pmb{\mathcal{R}}_{L}(\Theta_L)$, we introduce the following definitions.
\begin{definition}
Let $\mathbb{X}$ be a vector space. Given $\mathrm{\Xi}_L \in \mathbb{Y}_{L} :=  \mathbb{X} \times (\mathbb{X})^ 2 \times \cdots \times (\mathbb{X})^L$, we define an operator $\bold{flip}: \mathbb{Y}_{L} \rightarrow (\mathbb{X})^{(L+1) \times (L+1)}$ by
\begin{equation}
\begin{aligned}
\text{flipping: } & (\bold{flip}(\mathrm{\Xi}_L))^{k,l} = (\bold{flip}(\mathrm{\Xi}_L))^{l,k} = \mathrm{\Xi}_L^{l,k},  1 \leq l \leq L,  1 \leq k \leq l;  \\ 
\text{with } & (\bold{flip}(\mathrm{\Xi}_L))_L^{0,  0} = \mathrm{\Xi}_L^{1,1};  (\bold{flip}(\mathrm{\Xi}_L))^{l,0} = \mathrm{\Xi}_L^{l,1} ;   (\bold{flip}(\mathrm{\Xi}_L))^{0, k} = \mathrm{\Xi}_L^{k,1},  1 \leq k,l \leq L.
\end{aligned}
\nonumber
\end{equation}  
\end{definition}
We give an example in Figure~\ref{figure: flip} to help understand the $\bold{flip}(\cdot)$ operation. 
\begin{figure}[htbp]
	\centering
	\includegraphics[scale=0.22]{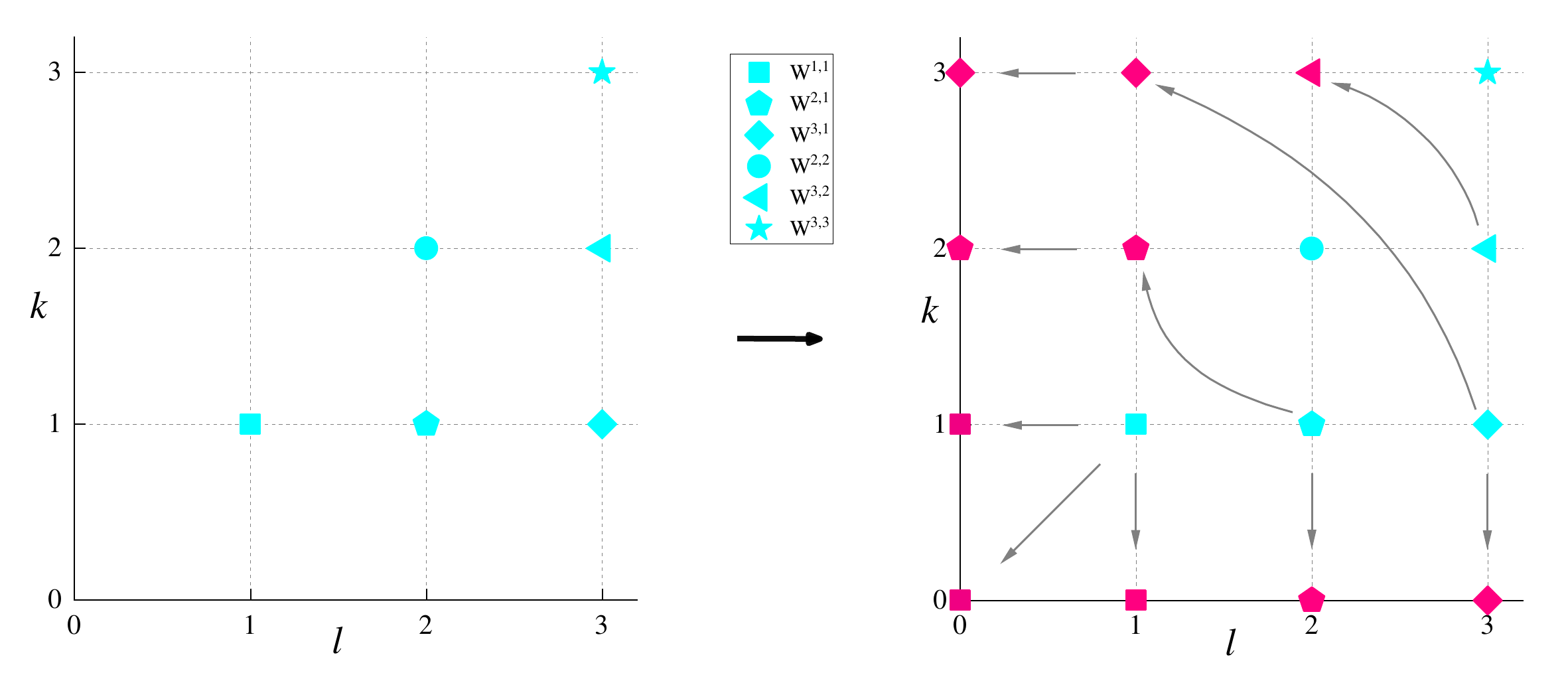}
	\caption{An illustration of the $\bold{flip}(\cdot)$ operation. The left and right pictures show the parameters $\bold{W}_3$ and $\bold{flip}(\bold{W}_3)$, respectively. The elements are arranged on the grid points according to their superscripts. The gray arrow represents the ``copy" operation.}
		\label{figure: flip}
\end{figure} 

Let $\Theta_L = (\mathrm{T}_L, \mathrm{U}_L, \mathrm{a}_L, \mathrm{V}_L, \mathrm{b}_L, \mathrm{W}_L, \mathrm{c}_L ) \in \Omega_{\Theta; L}$ and $\bar{\mathrm{W}}_L = \bold{flip}(\mathrm{W}_L), \bar{\mathrm{c}}_L = \bold{flip}(\mathrm{c}_L)$. 
We consider the following regularization term
\begin{equation}
	\begin{aligned}
		\pmb{\mathcal{R}}_L(\Theta_L)  = &  \sum_{j=1}^{3} \pmb{\mathcal{R}}_L^{(1)}({(\mathrm{T}_j)}_{L})+
		\pmb{\mathcal{R}}_L^{(1)}(\mathrm{U}_L) + \pmb{\mathcal{R}}_L^{(2)}(\mathrm{a}_L) + \pmb{\mathcal{R}}_L^{(1)}(\mathrm{V}_L) + \pmb{\mathcal{R}}_L^{(2)}(\mathrm{b}_L) \\
		& + \pmb{\mathcal{R}}_L^{(3)}(\bar{\mathrm{W}}_L) + \pmb{\mathcal{R}}_L^{(4)}(\bar{\mathrm{c}}_L),
	\end{aligned}
	\label{equation: dis regularization}
\end{equation}
where
\begin{equation}
	\begin{aligned}
 &\pmb{\mathcal{R}}_L^{(1)}(\mathrm{U}_L)
		 := \! \tau \sum_{l=1}^{L}\|\mathrm{U}_{L}^{l}\|^{2} \! + \!  \tau^{-1} \sum_{l=1}^{L}\|\mathrm{U}_{L}^{l} \! - \! \mathrm{U}_{L}^{l-1}\|^{2}, \\
         &\pmb{\mathcal{R}}_L^{(2)}(\mathrm{a}_L)
		:=  \tau \sum_{l=1}^{L}\|\mathrm{a}_{L}^{l}\|^{2} \!+ \!  \tau^{-1} \sum_{l=1}^{L}\|\mathrm{a}_{L}^{l} \! - \! \mathrm{a}_{L}^{l-1}\|^{2}, \\
	&\pmb{\mathcal{R}}_L^{(3)}(\bar{\mathrm{W}}_L) 
	:=	\!  \tau^2 \sum_{l=1}^{L} \sum_{k=1}^{L}\|\bar{\mathrm{W}}_{L}^{l,k}\|^{3}  \\
    & \qquad \qquad \qquad  + \!  \tau^{-1} \Big(\!\sum_{l=1}^{L} \sum_{k=1}^{L}  \|\bar{\mathrm{W}}_{L}^{l,k}\! - \!\bar{\mathrm{W}}_{L}^{l-1,k}\|^{3} \! + \! \sum_{l=1}^{L} \sum_{k=1}^{L}  \|\bar{\mathrm{W}}_{L}^{l,k} \!-\! \bar{\mathrm{W}}_{L}^{l,k-1}\|^{3} \!\Big), \\
	&\pmb{\mathcal{R}}_L^{(4)}(\bar{\mathrm{c}}_L)
	:=  \tau^2  \sum_{l=1}^{L} \!\sum_{k=1}^{L}\|\bar{\mathrm{c}}_{L}^{l,k}\|^{3} \\
    &\qquad \qquad \qquad  +   \tau^{-1} \Big(\!\sum_{l=1}^{L} \sum_{k=1}^{L}  \|\bar{\mathrm{c}}_{L}^{l,k} \!-\! \bar{\mathrm{c}}_{L}^{l-1,k}\|^{3} \! + \!\sum_{l=1}^{L} \sum_{k=1}^{L}  \|\bar{\mathrm{c}}_{L}^{l,k}\! -\! \bar{\mathrm{c}}_{L}^{l,k-1}\|^{3} \! \Big).
	\end{aligned}
	\label{regularizations}
\end{equation}
As can be seen, the definitions in (\ref{regularizations}) mimic some norms in some specific Sobolev spaces. We use the above regularization instead of standard $l^2$ or $l^1$ penalties to ensure compactness, as it controls both the magnitude and variation of parameters across layers, similar to the approach in \cite{thorpe2018deep, Esteve2020LargetimeAI}.

\section{Deep-layer limit of the DNL framework: mathematical modeling and main convergence result}
\label{Sec:3}
In this section, 
we consider the limit of the DNL framework (\ref{equation: DNLF}) as the total layer number $L \rightarrow \infty$.
We show dynamic system modeling for the DNL framework and present our main result on the convergence from the learning problem of the discrete-time DNL framework to that of the continuous-time DNL framework.

\subsection{The continuous-time dynamical system modeling for DNL framework}
We now model the forward propagation of the DNL framework (\ref{equation: DNLF}) as a dynamical system in a continuous-time setting, and present its learning problem. For this, we give some notations in a continuous-time setting.
The space of functions that are continuous on $\Omega \subset \mathbb{R}^d$ is denoted by $\mathcal{C}(\Omega)$.
The Sobolev space of functions that are $q$-times weakly differentiable and each weak derivative is $\mathcal{L}^p$ integrable on $\Omega$ is denoted by $\mathcal{W}^{q, p}(\Omega)$. Specially, $\mathcal{H}^q(\Omega) := \mathcal{W}^{q,2}(\Omega)$. 
The functions are denoted with bold letters.
{For a vector-valued (resp., matrix-valued) function $\bold{a}:[0, 1]\to\mathbb{R}^{n}$ (resp., $\bold{U}:[0, 1]\to\mathbb{R}^{n\times n}$), we denote its supremum norm (if exists) as $\|\bold{a}\|_{C}$ (resp., $\|\bold{U}\|_{C}$) and denote its $\mathcal{L}^{\infty}$ norm (if exists) as $\|\bold{a}\|_{\mathcal{L}^{\infty}}$ (resp., $\|\bold{U}\|_{\mathcal{L}^{\infty}}$).}
For a two-variable function $\bold{f}(t,s)$, we denote $\pmb{\mathcal{D}}_{t}(\bold{f})$ and $\pmb{\mathcal{D}}_{tt}(\bold{f})$ as its first-order and second-order weak partial derivative for variable $t$ (if it exists). The notations $\pmb{\mathcal{D}}_{s}(\bold{f})$ and $\pmb{\mathcal{D}}_{ss}(\bold{f})$ are similarly defined. The identity operator is denoted by $\bold{Id}$. 

Without loss of generality, we assume that the time interval for the continuous-time system is $[0,1]$, and thus step size $\tau = 1/L$.
We introduce the $ansatz$ $\mathrm{x}^{l} \approx \bold{x}(l\tau), \ \mathrm{T}^{l} \approx \bold{T}(l\tau), \ (\mathrm{U}_L^{l},\mathrm{a}_L^{l})  \approx (\bold{U}(l\tau), \bold{a}(l\tau)),\ (\mathrm{V}_L^{l},\mathrm{b}_L^{l})  \approx (\bold{V}(l\tau), \bold{b}(l\tau))$ ($0 \leq l \leq L$), and  $(\mathrm{W}_L^{l,k},\mathrm{c}_L^{l,k}) \approx (\bold{W}(l\tau, k\tau),\bold{c}(l\tau, k\tau))$ ($1 \leq k\leq l \leq L$), 
for some smooth curves $\bold{x}(t)$, $\bold{T}(t)$, $\bold{U}(t)$, $\bold{a}(t)$, $\bold{V}(t)$, $\bold{b}(t)$ defined for $0 \leq t \leq 1$ and $\bold{W}(t, s), \bold{c}(t, s)$ defined for $0 \leq t,s \leq 1$. Let $t_L^l=l \tau$, $0 \leq l \leq L$. For small $\tau$, we have $\bold{x}(t_L^l) \approx  \mathrm{x}^{l}$, $\bold{T}(t_L^l) \approx  \mathrm{T}^{l}$, $(\bold{U}(t_L^l), \bold{a}(t_L^l)) \approx (\mathrm{U}_L^{l},\mathrm{a}_L^{l}),\  (\bold{V}(t_L^l), \bold{b}(t_L^l)) \approx (\mathrm{V}_L^{l},\mathrm{b}_L^{l})$ and $ (\bold{W}(t_L^l, t_L^k), \bold{c}(t_L^l, t_L^k))  \approx(\mathrm{W}_L^{l,k},\mathrm{c}_L^{l,k})$. 
In this way, equation \eqref{equation: DNLF} implies that 
\begin{equation}
	\begin{aligned}
		\bold{x}(t_L^l) 
		& = \bold{V}(t_L^l) \pmb{\phi} \!\circ \!\Big( \bold{U}(t_L^{l}) \pmb{\mathcal{A}} _{\pmb{\kappa} }(\bold{T}(t_L^l); \mathrm{d}) +\bold{a}(t_L^{l}) \\
        & \qquad \qquad \quad \ + \! \sum_{k=0}^{l-1}\! \int_{t_L^{k}}^{t_L^{k\!+\!1}} \! [\bold{W}(t_L^{l}, t_L^{k+1}) \pmb{\mathcal{A}} _{\pmb{\kappa}}(\bold{T}(t_L^l); \bold{x}(t_L^{k})) \!+\! \bold{c}(t_L^{l}, t_L^{k+1})] {d}s \Big) \!+ \!\bold{b}(t_L^{l}).
	\end{aligned}
 \nonumber
\end{equation}
This formulation can be seen as a discrete approximation of the following integral equation
\begin{equation}
	\bold{x}(t) = \bold{V}(t) \pmb{\phi} \circ \Big( \bold{U}(t) \pmb{\mathcal{A}} _{\pmb{\kappa} }(\bold{T}(t);\mathrm{d}) +\bold{a}(t) + \int_{0}^{t} \Big[\bold{W}(t, s) \pmb{\mathcal{A}}_{\pmb{\kappa}}(\bold{T}(t); \bold{x}(s)) + \bold{c}(t, s) \Big] {d}s \Big) + \bold{b}(t),
	\label{equation: continuous-time DNLF}
\end{equation}
where $t \in [0,1]$, $\mathrm{d} \in \mathbb{R}^n$ is the input data. It is noted that when $\pmb{\phi}= \bold{Id}$, Eq.(\ref{equation: continuous-time DNLF}) is a Volterra integral equation of the second kind \cite{brunner2017volterra}, which is often utilized to characterize the dynamic behavior of systems with memory properties or problems involving delay effects. {We mention that this type of integral equation with memory properties is consistent with some original intention of DenseNets.}

In the following, we refer to the integral equation 
(\ref{equation: continuous-time DNLF}) as 
continuous-time DNL framework. Similar to the discrete case, we package the continuous-time learnable parameters together
and denote $\pmb{\Theta} := (\bold{T}, \bold{U}, \bold{a}, \bold{V}, \bold{b}, \bold{W}, \bold{c})$, where $\bold{T}=(\bold{T}_1, \bold{T}_2, \bold{T}_3)$. For convenience, we let the parameter space 
\begin{equation}
	\begin{aligned}
			\mathcal{C}_{\pmb{\Theta}}:= & \mathcal{C}([0, 1]; (\mathbb{R}^{n\times n})^3) \! \times  \! \big(\mathcal{C}([0, 1]; \mathbb{R}^{n\times n})  \! \times \! \mathcal{C}([0, 1]; \mathbb{R}^{n})\big)^2  \\
			&	\times \{\bold{W} \in  \mathcal{L}^{\infty}([0, 1] \times [0, 1]; \mathbb{R}^{n\times n}) : \bold{W}(t,s) = \bold{W}(s,t), \ 0\leq s,t \leq 1 \} \\
			& \times \{\bold{c} \in  \mathcal{L}^{\infty}([0, 1] \times [0, 1]; \mathbb{R}^{n}) : \bold{c}(t,s) = \bold{c}(s,t), \ 0\leq s,t \leq 1 \}, 
	\end{aligned}
	\label{parameter space C}
\end{equation}
and introduce a norm $\| \cdot \|_{	\mathcal{C}_{\pmb{\Theta}}}$ for $\pmb{\Theta} \in 	\mathcal{C}_{\pmb{\Theta}}$ as 
\begin{equation}
	\|\pmb{\Theta}\|_{	\mathcal{C}_{\pmb{\Theta}}} \!:= \! \max \{\|\bold{T}_1\|_{C}, \! \|\bold{T}_2\|_{C}, \!\|\bold{T}_3\|_{C}, \!\|\bold{U}\|_{C}, \!\|\bold{a}\|_{C}, \!\|\bold{V}\|_{C}, \!\|\bold{b}\|_{C}, \!\|\bold{W}\|_{\mathcal{L}^\infty}, \!\|\bold{c}\|_{\mathcal{L}^\infty}\}.
	\label{equation: norm of continuous-time p}
\end{equation}
We sometimes abbreviate $\|\pmb{\Theta}\|_{	\mathcal{C}_{\pmb{\Theta}}}$ as $\|\pmb{\Theta}\|$. Here, to be consistent with the $\textbf{flip}(\cdot)$ operation in discrete systems, we leverage symmetry to define the parameters $\bold{W}$ and $\bold{c}$ on $[0,1]\times[0,1]$, thus avoiding some complex boundary condition discussions in subsequent analysis. 
We also introduce a parameter set 
\begin{equation}
	\begin{aligned}
		\Omega_{\pmb{\Theta}} := \mathcal{C}_{\pmb{\Theta}} \cap \Big( & \mathcal{H}^1((0, 1); (\mathbb{R}^{n\times n})^3) \times \big(\mathcal{H}^1((0, 1); \mathbb{R}^{n\times n}) \times \mathcal{H}^1((0, 1); \mathbb{R}^{n}) \big)^2 \\
		& \times \mathcal{W}^{1,3}((0, 1)\times(0, 1); \mathbb{R}^{n \times n}) \times \mathcal{W}^{1,3}((0, 1) \times(0, 1); \mathbb{R}^{n}) \Big)
	\end{aligned}
	\label{continuous parameter set}
\end{equation}
to define regularization terms for the continuous-time learning problem and enforce some compactness.

Let $\bold{x}(\cdot; \mathrm{d};\pmb{\Theta})$ be the trajectory of Eq.(\ref{equation: continuous-time DNLF}) for parameter $\pmb{\Theta}$ and input $\mathrm{d}$. The learning problem of the continuous-time DNL framework can be given as follows
\begin{equation}
	\begin{aligned}
		(\mathcal{P}): \	\left\{ \begin{array}{l}
			 	\inf \limits_{\pmb{\Theta} \in \Omega_{\pmb{\Theta}}} \left\{ {\pmb{{\mathfrak{L}}}_{\mathcal{S}}}({\pmb{\Theta}}) = \frac{1}{M} \sum\limits_{m=1}^{M} \pmb{\ell} (\bold{x}(1; \mathrm{d}_m;\pmb{\Theta}) , \mathrm{g}_m) + \pmb{\mathcal{R}}(\pmb{\Theta}) \right\} \\
			\text{ subject to:} \\
		   \bold{x}(\cdot; \mathrm{d}_m;\pmb{\Theta})  \text{ satisfies Eq.} (\ref{equation: continuous-time DNLF}) \text{ with input data } \mathrm{d}_m \text{ and parameter } \pmb{\Theta},
		\end{array} \right.\\
	\end{aligned}
	\label{continuous-time: control problem P}
\end{equation}
where 
\begin{equation}
    \begin{aligned}
        	&\pmb{\mathcal{R}}(\pmb{\Theta}) = \sum_{j=1}^{3} \|\bold{T}_j\|^2_{\mathcal{H}^1((0, 1); \mathbb{R}^{n\times n})} \! + \! \|\bold{U}\|^2_{\mathcal{H}^1((0, 1); \mathbb{R}^{n\times n})} \! + \! \|\bold{a}\|^2_{ \mathcal{H}^1((0, 1); \mathbb{R}^{n})} \! + \! \|\bold{V}\|^2_{\mathcal{H}^1((0, 1); \mathbb{R}^{n\times n})} \\
    & \qquad \qquad + \! \|\bold{b}\|^2_{ \mathcal{H}^1((0, 1); \mathbb{R}^{n})} + \|\bold{W}\|^3_{\mathcal{W}^{1,3}((0, 1) \times (0, 1); \mathbb{R}^{n\times n})} + \|\bold{c}\|^3_{\mathcal{W}^{1,3}((0, 1) \times (0, 1); \mathbb{R}^{n})}.
    \end{aligned}
    \label{equation: continuous regularization}
\end{equation}
The regularization term $\pmb{\mathcal{R}}_L$ defined in Eq.(\ref{equation: dis regularization}) can be regarded as a discrete approximation of $\pmb{\mathcal{R}}$.

\subsection{Main convergence result}
In this subsection, we present the main result of this work on the convergence from the learning problem $(\mathcal{P}_L)$ of discrete-time DNL framework to the learning problem $(\mathcal{P})$ of continuous-time DNL framework, by using some extension operators defined as follows.
\begin{definition}
\label{definition: linear extention}
	Let $\mathbb{X}$ be a vector space. Partition the time interval $[0,1]$ into $L$ intervals $\{[(l-1)\tau,l\tau] \}_{l=1}^{L}$, where $\tau = 1/L$. We define the piecewise constant extension operator $\bar{\pmb{\mathcal{I}}}_{L}: (\mathbb{X})^{L+1}  \rightarrow \mathcal{M}_L([0,1]; \mathbb{X})$ and
 the piecewise linear extension operator $\hat{\pmb{\mathcal{I}}}_{L}: (\mathbb{X})^{L+1} \rightarrow \mathcal{C}([0,1]; \mathbb{X})$ for $\mathrm{\Xi}_L=(\mathrm{\xi}_L^0, \mathrm{\xi}_L^1, \ldots, \mathrm{\xi}_L^{L}) \in (\mathbb{X})^{L+1}$ as
	\begin{align*}
 	(\bar{\pmb{\mathcal{I}}}_{L} \mathrm{\Xi}_L) (t) &=
		\left \{
		\begin{aligned}
		  & \mathrm{\xi}_L^{l}, \ t \in ((l-1)\tau,l\tau],
		\ 1\leq l \leq L, \\
		 & \mathrm{\xi}_L^{0}, \ t =0;
	\end{aligned}
	\right. \\
		\nonumber
	(\hat{\pmb{\mathcal{I}}}_{L} \mathrm{\Xi}_L) (t)  &= 
	\mathrm{\xi}_L^{l-1} + \Big(t-(l-1) \tau \Big)\frac{\mathrm{\xi}_L^{l} - \mathrm{\xi}_L^{l-1}}{\tau}, \ t \in [(l-1)\tau,l\tau], \ 1\leq l \leq L, 
		\nonumber
	\end{align*}
where $\mathcal{M}_L([0,1]; \mathbb{X}) $ is the piecewise constant function space with $L$ pieces. 
\end{definition}

\begin{definition}
\label{definition: bilinear extension}
Let $\mathbb{X}$ be a vector space. Divide the domain $[0,1] \times [0,1]$ equally into $L^2$ squares $\{[(l-1)\tau,l\tau] \times [(k-1)\tau,k\tau]\}_{l,k=1}^{L}$, where $\tau = 1/L$. Define the piecewise constant extension operator $\bar{\pmb{\mathcal{BI}}}_{L}: (\mathbb{X})^{(L+1)\times (L+1)}  \rightarrow \mathcal{M}([0,1] \times [0,1]; \mathbb{X})$ and the piecewise bilinear extension operator $\hat{\pmb{\mathcal{BI}}}_{L}: (\mathbb{X})^{(L+1)\times (L+1)} \rightarrow \mathcal{C}([0,1] \times [0,1]; \mathbb{X})$  for $\mathrm{\Xi}_L=(\mathrm{\xi}_L^{l,k})_{0 \leq l,k\leq L} \in (\mathbb{X})^{(L+1)\times (L+1)} $ as
\begin{align*}
(\bar{\pmb{\mathcal{BI}}}_{L}\mathrm{\Xi}_L) (t,s)=&
\left \{
\begin{aligned}
	& \mathrm{\xi}_L^{l,k}, \ t \in ((l-1)\tau,l\tau], s \in ((k-1)\tau,k\tau],
	\ 1\leq l,k \leq L, \\
	& \mathrm{\xi}_L^{0, k}, \ t = 0, s \in ((k-1)\tau,k\tau], \ 1\leq k \leq L, \\
	& \mathrm{\xi}_L^{l, 0}, \ s = 0, t \in ((l-1)\tau,l\tau], \ 1\leq l \leq L, \\
	& \mathrm{\xi}_L^{0, 0}, \ t = 0, s=0.
\end{aligned}
\right.
\nonumber 
\\
	(\hat{\pmb{\mathcal{BI}}}_{L} \mathrm{\Xi}_L) (t,s)  =
		&\frac{k\tau -s}{\tau}\! \big( \frac{l\tau-t}{\tau} \mathrm{\xi}_L^{l-1,k-1} \!+\! \frac{t\!-\!(l-1)\tau}{\tau}  \mathrm{\xi}_L^{l,k-1} \big) \!+\! \frac{s\!-\!(k-1)\tau}{\tau} \big( \frac{l\tau-t}{\tau} \mathrm{\xi}_L^{l-1,k} \\ 
		&+\! \frac{t-(l-1)\tau}{\tau}  \mathrm{\xi}_L^{l,k} \big), t \in [(l-1)\tau,l\tau], s \in [(k-1)\tau,k\tau], 1\leq k,l \leq L,  
\nonumber
\end{align*}
where $\mathcal{M}([0,1] \times [0,1]; \mathbb{X}) $ is the piecewise constant function space with $L\times L$ pieces.
\end{definition}
These extension operators connect the learnable parameters of the discrete-time DNL framework \eqref{equation: DNLF} and continuous-time DNL framework \eqref{equation: continuous-time DNLF}. 
In particular, we denote ${\hat{\pmb{\mathcal{I}}}}_{L}\Theta_{L}\! = \!\big(\hat{\pmb{\mathcal{I}}}_{L} (\mathrm{T}_1)_{L}, \hat{\pmb{\mathcal{I}}}_{L} (\mathrm{T}_2)_{L}, \hat{\pmb{\mathcal{I}}}_{L} (\mathrm{T}_3)_{L}, \hat{\pmb{\mathcal{I}}}_{L} \mathrm{U}_L,$ $ \hat{\pmb{\mathcal{I}}}_{L} {\mathrm{a}}_{L}, \hat{\pmb{\mathcal{I}}}_{L} \mathrm{V}_L, \hat{\pmb{\mathcal{I}}}_{L} {\mathrm{b}}_{L}, \hat{\pmb{\mathcal{BI}}}_{L} (\bold{flip}(\mathrm{W}_L))$, $ \hat{\pmb{\mathcal{BI}}}_{L} (\bold{flip}(\mathrm{c}_L))\big)$ and ${\bar{\pmb{\mathcal{I}}}}_{L}\Theta_{L} = \big(\bar{\pmb{\mathcal{I}}}_{L} (\mathrm{T}_1)_{L}, \bar{\pmb{\mathcal{I}}}_{L} (\mathrm{T}_2)_{L}, \bar{\pmb{\mathcal{I}}}_{L} (\mathrm{T}_3)_{L}, \bar{\pmb{\mathcal{I}}}_{L} \mathrm{U}_L, \bar{\pmb{\mathcal{I}}}_{L} {\mathrm{a}}_{L}, \bar{\pmb{\mathcal{I}}}_{L} \mathrm{V}_L, \bar{\pmb{\mathcal{I}}}_{L} {\mathrm{b}}_{L},$
$ \bar{\pmb{\mathcal{BI}}}_{L} (\bold{flip}(\mathrm{W}_L)), $ $ \bar{\pmb{\mathcal{BI}}}_{L} (\bold{flip}(\mathrm{c}_L) )  \big)$.

We will use the following assumptions:
\begin{itemize}
	\item [(${A}_1$)] The activation function $\pmb{\phi}: \mathbb{R} \rightarrow \mathbb{R}$ is $L_{\pmb{\phi}}$-Lipschitz, increasing, acts point-wise and takes $0$ at $0$.
	\item  [(${A}_2$)]  The transformation $\pmb{\mathcal{A}} _{\pmb{\kappa}} : (\mathbb{R}^{n\times n})^3 \times \mathbb{R}^n  \rightarrow \mathbb{R}^n$ satisfies a local growth condition and a local Lipschitz condition in the following sense. There exist continuous functions $\pmb{\mathcal{G}}_{\pmb{\mathcal{A}}} (\cdot): (\mathbb{R}^{n\times n})^3 \rightarrow \mathbb{R}^+$ and $\pmb{\mathcal{L}}_{\pmb{\mathcal{A}}}(\cdot, \cdot, \cdot): (\mathbb{R}^{n\times n})^3 \times \mathbb{R}^n \times \mathbb{R}^n \rightarrow \mathbb{R}^+$ such that 
 for all $\mathrm{\Xi}= (\mathrm{\Xi}_1, \mathrm{\Xi}_2, \mathrm{\Xi}_3), \mathrm{\Xi}^{\prime}= (\mathrm{\Xi}_1^{\prime}, \mathrm{\Xi}_2^{\prime}, \mathrm{\Xi}_3^{\prime}) \in (\mathbb{R}^{n\times n})^3$ and $\mathrm{z}, \mathrm{z}^{\prime} \in \mathbb{R}^{n}$
 $$
 \begin{aligned}
 |\pmb{\mathcal{A}} _{\pmb{\kappa}}(\mathrm{\Xi}, \mathrm{z})| & \leq \pmb{\mathcal{G}}_{\pmb{\mathcal{A}}}(\mathrm{\Xi}) \cdot |\mathrm{z}|, \\
     |\pmb{\mathcal{A}} _{\pmb{\kappa}}(\mathrm{\Xi}, \mathrm{z}) - \pmb{\mathcal{A}} _{\pmb{\kappa}}(\mathrm{\Xi}^{\prime}, \mathrm{z}^{\prime})| & \leq \pmb{\mathcal{L}}_{\pmb{\mathcal{A}}}(\mathrm{\Xi}, \mathrm{z}, \mathrm{z}^{\prime}) \cdot \Big(\sum_{i=1}^{3}\|\mathrm{\Xi}_i - \mathrm{\Xi}^{\prime}_i\| + |\mathrm{z} - \mathrm{z}^{\prime}|\Big).
 \end{aligned}
 $$
  \item  [(${A}_3$)] The loss function $\pmb{\ell} : \mathbb{R}^{n} \times \mathbb{R}^{n} \rightarrow \mathbb{R}^+$ is continuous in its first argument.
\end{itemize}
\begin{remark}
\label{remark: remark for assumptions}
These assumptions are mild and commonly used in deep learning theory. Many widely used activation functions, such as ReLU, parametric ReLU, and tanh \cite{goodfellow2016deep}, satisfy (${A}_1$). Common loss functions such as mean squared error and cross-entropy satisfy (${A}_3$). 
Moreover, a broad class of existing architectures satisfy (${A}_2$). For example:
\begin{itemize}
    \item[(i)] Affine map in DenseNet: the kernel $\pmb{\kappa}((\mathrm{\Xi}_1\mathrm{z})[i], (\mathrm{\Xi}_2\mathrm{z})[j]) = \pmb{\delta}_{ij}$ and   
    $\pmb{\mathcal{A}}_{\pmb{\kappa}}(\mathrm{\Xi}, \mathrm{z}) = \mathrm{z} \mathrm{\Xi}_3$, where $\pmb{\delta}_{ij}$ is the Kronecker function. Then (${A}_2$) is satisfied with 
    $$
    \begin{aligned}
    \pmb{\mathcal{G}}_{\pmb{\mathcal{A}}}(\mathrm{\Xi}) =  \|\mathrm{\Xi}_3\|_2, \ 
    \pmb{\mathcal{L}}_{\pmb{\mathcal{A}}} (\mathrm{\Xi}, \mathrm{z}, \mathrm{z}^{\prime}) = \|\mathrm{\Xi}_3\|_2 + |\mathrm{z}^{\prime}| .
    \end{aligned}
    $$
    \item[(ii)] {Self-attention in Transformers} \cite{vaswani2017attention}: the kernel 
    $\pmb{\kappa}((\mathrm{\Xi}_1\mathrm{z})[i], (\mathrm{\Xi}_2\mathrm{z})[j]) = \exp((\mathrm{\Xi}_1\mathrm{z})[i] \cdot (\mathrm{\Xi}_2\mathrm{z})[j]/\sqrt{n})$, and   
    $\pmb{\mathcal{A}}_{\pmb{\kappa}}(\mathrm{\Xi}, \mathrm{z}) = \mathrm{softmax}\left(\frac{ \mathrm{\Xi}_1 \mathrm{z} \cdot ( \mathrm{\Xi}_2 \mathrm{z})^{\top}}{\sqrt{n}}\right)  \mathrm{\Xi}_3\mathrm{z}$.
    One can verify that assumption (${A}_2$) holds with 
    $$
    \begin{aligned}
    \pmb{\mathcal{G}}_{\pmb{\mathcal{A}}}(\mathrm{\Xi}) &= \sqrt{n} \|\mathrm{\Xi}_3\|_2, \\ 
    \pmb{\mathcal{L}}_{\pmb{\mathcal{A}}} (\mathrm{\Xi}, \mathrm{z}, \mathrm{z}^{\prime}) &= \sqrt{n} \max \big \{ \|\mathrm{\Xi}\|^2 |\mathrm{z}|^3,  \|\mathrm{\Xi}\|^3  (|\mathrm{z}|^2+ |\mathrm{z}||\mathrm{z}^{\prime}|) + \|\mathrm{\Xi}\|, \|\mathrm{\Xi}\|^2 |\mathrm{z}|^2 |\mathrm{z}^{\prime}|, |\mathrm{z}^{\prime}| \big \},
    \end{aligned}
    $$
    where $\|\mathrm{\Xi}\| = \max \{\|\mathrm{\Xi}_1\|_2, \|\mathrm{\Xi}_2\|_2, \|\mathrm{\Xi}_3\|_2 \}$.
    
    \item[(iii)] {Gaussian non-local operator in Non-local Nets} \cite{Wang2017NonlocalNN}: the kernel is given by $\pmb{\kappa}((\mathrm{\Xi}_1\mathrm{z})[i], (\mathrm{\Xi}_2\mathrm{z})[j]) = \exp(\mathrm{z}[i] \cdot \mathrm{z}[j])$, and 
    $\pmb{\mathcal{A}}_{\pmb{\kappa}}(\mathrm{\Xi}, \mathrm{z}) = \mathrm{softmax}\left({\mathrm{z} \mathrm{z}^{\top}}\right) \mathrm{\Xi}_3\mathrm{z}$.
    In this case, (${A}_2$) is satisfied with
    $$
    \begin{aligned}
    \pmb{\mathcal{G}}_{\pmb{\mathcal{A}}} (\mathrm{\Xi}) &= \sqrt{n} \|\mathrm{\Xi}_3\|_2, \\ 
    \pmb{\mathcal{L}}_{\pmb{\mathcal{A}}}(\mathrm{\Xi}, \mathrm{z}, \mathrm{z}^{\prime}) &= \max \big \{ {\|\mathrm{\Xi}\| |\mathrm{z}|}(|\mathrm{z}| + |\mathrm{z}^{\prime}| ) + \sqrt{n}\|\mathrm{\Xi}\| ,  \sqrt{n}|\mathrm{z}^{\prime}| \big \}.
    \end{aligned}
    $$
\end{itemize}
See the supplementary material for detailed derivations.
\end{remark}

Our main result establishes both the convergence of optimal values and the subsequence convergence of optimal solutions for the discrete-time learning problem as the number of layers approaches infinity.
\begin{theorem}
(Convergence from the discrete to  continuous  learning problem) Consider the problems $(\mathcal{P}_{L}), (\mathcal{P})$ defined in (\ref{discrete-time-control problem P_L}) and (\ref{continuous-time: control problem P}), respectively. Let the parameter sets $\Omega_{\Theta;L}$ and $\Omega_{\pmb{\Theta}}$ be given by (\ref{notation: discrete learnable parameter set}) and (\ref{continuous parameter set}). If assumptions $(A_1)-(A_3)$ hold, then minimizers of $(\mathcal{P}_L)$ and $(\mathcal{P})$ exist. Moreover, for any sequence $\{\Theta_{L}^{*}\}_L \subset \Omega_{\Theta;L}$, where $\Theta_{L}^{*} = (\mathrm{T}^*_L,\mathrm{U}^*_L,\mathrm{a}^*_L,\mathrm{V}^*_L,\mathrm{b}^*_L,\mathrm{W}^*_L,\mathrm{c}^*_L)$ is a minimizer of $(\mathcal{P}_L)$, we have 
	$$
\min _{\Theta_L \in \Omega_{\Theta;L}} \pmb{\mathfrak{L}}_{\mathcal{S};L}(\Theta_L) =\pmb{\mathfrak{L}}_{\mathcal{S};L} \left(\Theta_{L}^{*}\right) \rightarrow \min _{\pmb{\Theta} \in \Omega_{\pmb{\Theta}}} \pmb{\mathfrak{L}}_{\mathcal{S}} (\pmb{\Theta}), \quad \text { as } L \rightarrow \infty.
$$
Let $\pmb{\Theta}^*_L: ={\hat{\pmb{\mathcal{I}}}}_{L}\Theta^*_{L}$, $L\ge 1$. Then $\{\pmb{\Theta}^*_L\}_L$ is relatively compact in $\mathcal{C}_{\pmb{\Theta}}$ with respect to the norm $\| \cdot \|_{	\mathcal{C}_{\pmb{\Theta}}}$, 
and any limit point of $\{\pmb{\Theta}_L^*\}_L$ in $\Omega_{\pmb{\Theta}}$ is a minimiser of $(\mathcal{P})$.
	\label{theorem: main results}
\end{theorem}
Theorem \ref{theorem: main results} establishes that the training process of the discrete-time DNL framework can be viewed as a consistent discretization of a well-posed continuous-time learning problem, which may provide a theoretical justification in designing dynamic system learning approaches for networks with densely connected layers.
While our focus is on theoretical guarantees, these results imply that densely connected architectures can offer stability when training very deep models. The proof of the theorem is given in the subsequent sections.

\section{Proof details}
\label{Sec:4}
In this section, we provide a detailed proof for Theorem~\ref{theorem: main results}, primarily utilizing the properties of $\Gamma$-convergence. We begin by presenting some useful preliminary results on the forward dynamics of the DNL framework. Recall that $\tau = 1/L$, $t_L^l := l \cdot \tau$, $1\leq l\leq L$.

\begin{lemma} (Bound estimation for network states)
 \label{lemma: dis bound}
Let assumptions $(A_1)(A_2)$ hold. {Given a layer number $L \in \mathbb{N}$,} a learnable parameter $\Theta_L \in \Omega_{\Theta; L}$ and an input $\mathrm{d} \in \mathbb{R}^n$. Then the state variables $\{\mathrm{x}^l(\mathrm{d}; {\Theta_L})\}_{l=0}^L$ of DNL framework (\ref{equation: DNLF}) exist uniquely and satisfy
	\begin{equation}
 \begin{aligned}
     |\mathrm{x}^l(\mathrm{d}; {\Theta_L})| \leq & L_{\pmb{\phi}} \| \Theta_L\|^2 \bar{M}_{\pmb{\mathcal{A}}} \big[ L_{\pmb{\phi}} \| \Theta_L\|^2( \bar{M}_{\pmb{\mathcal{A}}}  |\mathrm{d}| + 2) + \| \Theta_L\| \big] \bold{exp}(L_{\pmb{\phi}} \| \Theta_L\|^2 \bar{M}_{\pmb{\mathcal{A}}} ) \\
      & + L_{\pmb{\phi}} \| \Theta_L\|^2( \bar{M}_{\pmb{\mathcal{A}}}  |\mathrm{d}| + 2) + \| \Theta_L\|,
 \end{aligned}
	\end{equation}
where  ${\bar{M}_{\pmb{\mathcal{A}}}} = \max_{0 \leq l \leq L} \{\pmb{\mathcal{G}}_{\pmb{\mathcal{A}}}(\mathrm{T}_L^l)\}$. Moreover, $\{\mathrm{x}^l(\mathrm{d}; {\Theta_L})\}_{l=0}^L$ continuously depend on the learnable parameter $\Theta_L$.
\end{lemma}
\begin{proof}
	The proof is standard and is supplied in supplementary materials.
\end{proof}

The next proposition shows that under the assumptions of (${A}_1$) and (${A}_2$), the continuous-time DNL framework
(\ref{equation: continuous-time DNLF}) is well defined.
\begin{proposition}
\label{proposition: forward-existence}
	(Existence, uniqueness, and bound estimation for the continuous architecture) Let $\mathrm{d} \in \mathbb{R}^n$ be a given input data. Let assumptions (${A}_1$) and (${A}_2$) hold true. For any given parameter $\pmb{\Theta} = (\bold{T}, \bold{U}, \bold{a}, \bold{V}, \bold{b}, \bold{W}, \bold{c}) \in  \mathcal{C}_{\pmb{\Theta}}$, the continuous-time DNL framework (\ref{equation: continuous-time DNLF}) has a unique continuous solution on $[0, 1]$ which satisfies
	$$
		|\bold{x}(t)| \leq \Big[L_{\pmb{\phi}} \|\bold{V}\|_C (\hat{M}_{\pmb{\mathcal{A}} } \|\bold{U}\|_C |\mathrm{d}| + \|\bold{a}\|_C +\|\bold{b}\|_C ) \! + \! \|\bold{c}\|_{\mathcal{L}^\infty} \! \Big] \! \cdot \exp  (L_{\pmb{\phi}}\hat{M}_{\pmb{\mathcal{A}}}\|\bold{V}\|_C \|\bold{W}\|_{\mathcal{L}^\infty} \! ), 
		$$
  where $t\in[0,1]$, ${\hat{M}_{\pmb{\mathcal{A}}}} = \sup_{t \in [0,1]} \pmb{\mathcal{G}}_{\pmb{\mathcal{A}}}(\bold{T}(t))< \infty$.  Moreover, the solution $\bold{x}(\cdot)$ continuously depend on the learnable parameter $\pmb{\Theta}$.
\end{proposition}
\begin{proof}
	The proof is given in {supplementary materials}.
\end{proof}

The next proposition shows that when the layer number $L \rightarrow \infty$, the state variables $\{\mathrm{x}^l(\mathrm{d}; {\Theta_L})\}_l$ of the discrete-time DNL framework \eqref{equation: DNLF} converge to the trajectory of the continuous-time DNL framework \eqref{equation: continuous-time DNLF}. Such a result plays an important role in the proof of Theorem~\ref{theorem: main results}, as it infers the convergence of the data loss of ($\mathcal{P}_L$).
\begin{proposition}
(Forward convergence) Let assumptions $(A_1) (A_2)$ hold. Given data $\mathrm{d} \in \mathbb{R}^{n}$,  parameters $\pmb{\Theta} = (\bold{T},\! \bold{U}, \! \bold{a}, \!\bold{V}, \!\bold{b}, \!\bold{W}, \! \bold{c}) \in \mathcal{C}_{\pmb{\Theta}}$ and $\Theta_L = (\mathrm{T}_L, \! \mathrm{U}_L,\! \mathrm{a}_L, \!\mathrm{V}_L,\!\mathrm{b}_L, $ $\mathrm{W}_L,\mathrm{c}_L) \in \Omega_{\Theta; L}$, $L\ge 1$. 
	 If 
  \begin{equation}
  \begin{aligned}
  {\bar{\pmb{\mathcal{I}}}}_{L}\Theta_{L}
      \overset{C_{\pmb{\Theta}}}{\longrightarrow}
 \pmb{\Theta} , 
  \end{aligned}
      \label{convergence of parameter}
  \end{equation}
as $L \rightarrow \infty$, then $\sup_{ {1 \leq l \leq L} } \sup_{t\in [t_{L}^{l-1}, t_{L}^{l}]} | \mathrm{x}^l(\mathrm{d}; {\Theta_L}) - \bold{x}(t; \mathrm{d}; \pmb{\Theta}) | \rightarrow 0, \text{ as } L \rightarrow \infty.$
\label{forward convergence}
\end{proposition}
\begin{proof}

Since $\pmb{\Theta} \in \Omega_{\pmb{\Theta}}$ is given, $\{ \Theta_L \}_L$ is bounded by (\ref{convergence of parameter}). Hence, there exist constants 
$$
\begin{aligned}
    M_{\Theta}:= & \max\{ \sup_{L\ge 1} \|\Theta_L\|, \|\pmb{\Theta} \|_{\Omega_{\pmb{\Theta}}} \} < +\infty, \\
    M_{\pmb{\mathcal{A}} }:= & \max \Big\{ \sup_{L\ge 1} \max_{0\leq l \leq L} \pmb{\mathcal{G}}_{\pmb{\mathcal{A}}}(\mathrm{T}_L^{l}), \ \sup_{0 \leq t \leq 1} \pmb{\mathcal{G}}_{\pmb{\mathcal{A}}}(\bold{T}(t)) \Big\} < \infty,
\end{aligned}
$$
by assumption $(A_2)$.
Additionally, there exist constants $B_{\mathrm{x}}>0$ and $B_{\bold{x}}>0$ such that
\[
\begin{aligned}
    |\mathrm{x}^l(\mathrm{d}; {\Theta_L})| \leq  B_{\mathrm{x}}, \ \forall 1\leq l\leq L, \ L\ge 1; \ 
    |\bold{x}(t; \mathrm{d}; \pmb{\Theta})| \leq  B_{\bold{x}}, \ \forall t \in [0,1],
\end{aligned}
\]
as a consequence of Lemma \ref{lemma: dis bound} and Proposition \ref{proposition: forward-existence}, respectively. Therefore, we can let 
$$L_{\pmb{\mathcal{A}}}:= \max \Big\{ \sup_{ \substack{L\ge 1, \\ |\mathrm{x}|, |\mathrm{x}^{\prime}| \leq B_{\mathrm{x}} } } 
\max_{0\leq l \leq L} \pmb{\mathcal{L}}_{\pmb{\mathcal{A}}}(\mathrm{T}_L^{l}, \mathrm{x}, \mathrm{x}^{\prime}), \ \sup_{ \substack{L\ge 1, \\ |\mathrm{x}|, |\mathrm{x}^{\prime}| \leq B_{\mathrm{x}} } }  \{\pmb{\mathcal{L}}_{\pmb{\mathcal{A}}}(\bold{T}(t), \mathrm{x}, \mathrm{x}^{\prime})\} \Big\} < \infty.$$

For simplicity, we denote $\mathrm{x}_L^{l} = \mathrm{x}^l(\mathrm{d}; {\Theta_L}), \ \bold{x}(t^l_L)= \bold{x}(t_L^l; \mathrm{d}; \pmb{\Theta})$.
Subtracting Eq.(\ref{equation: DNLF}) from Eq.(\ref{equation: continuous-time DNLF}) and using the assumptions $(A_1)(A_2)$ give
    \begin{align}
		&	\Big| \mathrm{x}_L^{l} - \bold{x}(t^l_L) \Big|    \label{forward convergence: 1} \\ 
                = &  \Big| \mathrm{V}_L^{l} \pmb{\phi} \circ \big(\mathrm{U}_L^{l} \pmb{\mathcal{A}} _{\pmb{\kappa} } (\mathrm{T}_L^{l}; \mathrm{d}) + \mathrm{a}_L^{l} + \tau \sum_{k=0}^{l-1} [\mathrm{W}_L^{l, k+1} \pmb{\mathcal{A}} _{\pmb{\kappa} } (\mathrm{T}_L^{l}; \mathrm{x}_L^k) + \mathrm{c}_L^{l,k+1}]\big) + \mathrm{b}_L^{l}  \nonumber \\
                & \ - \bold{V}(t_L^{l}) \pmb{\phi} \circ \Big( \mathrm{U}_L^{l} \pmb{\mathcal{A}} _{\pmb{\kappa} } (\mathrm{T}_L^{l}; \mathrm{d}) + \mathrm{a}_L^{l} + \tau \sum_{k=0}^{l-1} [\mathrm{W}_L^{l,k+1} \pmb{\mathcal{A}} _{\pmb{\kappa} }(\mathrm{T}_L^{l}; \mathrm{x}_L^k) + \mathrm{c}_L^{l,k+1}] \Big)  \nonumber \\
                & \ + \bold{V}(t_L^{l}) \pmb{\phi} \circ \Big( \mathrm{U}_L^{l} \pmb{\mathcal{A}} _{\pmb{\kappa} } (\mathrm{T}_L^{l}; \mathrm{d}) + \mathrm{a}_L^{l} + \tau \sum_{k=0}^{l-1} [\mathrm{W}_L^{l,k+1} \pmb{\mathcal{A}} _{\pmb{\kappa} }(\mathrm{T}_L^{l}; \mathrm{x}_L^k) + \mathrm{c}_L^{l,k+1}] \Big) \nonumber \\
                & \ -\bold{V}(t_L^{l}) \pmb{\phi} \circ \Big[ \bold{U}(t_L^{l}) \pmb{\mathcal{A}} _{\pmb{\kappa} }(\bold{T}(t_L^{l}); \mathrm{d}) \!+ \!\bold{a}(t_L^{l}) \nonumber \\
                & \qquad \qquad \qquad  + \int_{0}^{t_L^{l}} (\bold{W}(t_L^{l}, s) \pmb{\mathcal{A}} _{\pmb{\kappa} }(\bold{T}(t_L^{l}); \bold{x}(s)) \! + \! \bold{c}(t_L^{l}, s)) \mathrm{d}s \Big] \! - \! \bold{b}(t_L^{l})\Big| \nonumber \\ 
			\leq & L_{\pmb{\phi}}\|\mathrm{V}_L^{l} - \bold{V}(t^l_L)  \| \cdot \Big| \mathrm{U}_L^{l} \pmb{\mathcal{A}} _{\pmb{\kappa} } (\mathrm{T}_L^l; \mathrm{d}) + \mathrm{a}_L^{l} + \tau \sum_{k=0}^{l-1} [\mathrm{W}_L^{l,k+1}\pmb{\mathcal{A}} _{\pmb{\kappa} }(\mathrm{T}_L^l; \mathrm{x}_L^k) + \mathrm{c}_L^{l,k+1}] \Big|  \nonumber \\
                &  + \! | \mathrm{b}_L^{l} \! - \! \bold{b}(t^l_L)| \! + \! L_{\pmb{\phi}} \|\bold{V}(t_L^{l}) \| \cdot \big| \mathrm{U}_L^{l} \pmb{\mathcal{A}} _{\pmb{\kappa} } (\mathrm{T}_L^l; \mathrm{d})\! + \!\mathrm{a}_L^{l} \!-v \bold{U}(t_L^l) \pmb{\mathcal{A}} _{\pmb{\kappa} }(\bold{T}(t_L^{l}); \mathrm{d}) \!- \!\bold{a}(t_L^l) \big| \nonumber  \\
		    &  + L_{\pmb{\phi}} \|\bold{V}(t_L^{l}) \| \! \cdot \! \Big|\tau \sum_{k=0}^{l-1} \Big[\mathrm{W}_L^{l,k+1}\pmb{\mathcal{A}} _{\pmb{\kappa} }(\mathrm{T}_L^l; \mathrm{x}_L^k) \!+\! \mathrm{c}_L^{l,k+1}\Big] \nonumber \\
            &\qquad \qquad \qquad \quad  -\! \int_{0}^{t^l_L} \!\Big[\bold{W}(t^l_L, s) \pmb{\mathcal{A}} _{\pmb{\kappa} }(\bold{T}(t_L^{l}); \bold{x}(s)) + \bold{c}(t^l_L, s)\Big] {d}s \Big|, \nonumber
\end{align}
where $l\ge 1$. The first term on the right-hand side of Eq.(\ref{forward convergence: 1}) can be bounded as 
\begin{equation}
	\begin{aligned}
		&L_{\pmb{\phi}}\|\mathrm{V}_L^{l} - \bold{V}(t^l_L)  \| \cdot \Big| \mathrm{U}_L^{l} \pmb{\mathcal{A}} _{\pmb{\kappa} } (\mathrm{T}_L^l; \mathrm{d}) + \mathrm{a}_L^{l} + \tau \sum_{k=0}^{l-1} [\mathrm{W}_L^{l,k+1}\pmb{\mathcal{A}} _{\pmb{\kappa}}(\mathrm{T}_L^l; \mathrm{x}_L^k) + \mathrm{c}_L^{l,k+1}] \Big| \\
	\leq	&  L_{\pmb{\phi}} \| \bar{\pmb{\mathcal{I}}}_{L} \mathrm{V}_L - \bold{V} \|_C \Big[ \|\Theta_L\| (M_{\pmb{\mathcal{A}} } |\mathrm{d}| +1 ) + \tau \sum_{k=0}^{l-1} (\|\Theta_L\| M_{\pmb{\mathcal{A}} } |\mathrm{x}_L^k |+ \|{\Theta_L}\|) \Big]  \\
	\leq	&  L_{\pmb{\phi}} \| \bar{\pmb{\mathcal{I}}}_{L} \mathrm{V}_L - \bold{V} \|_C \cdot M_{\Theta} \big(M_{\pmb{\mathcal{A}} } |\mathrm{d}|   + M_{\pmb{\mathcal{A}} }B_{\mathrm{x}} + 2 \big),
	\end{aligned}
 \nonumber
\end{equation}
where the first inequality is due to Definition (\ref{definition: linear extention}) and assumption (${A}_2$).
In addition, we may bound the last two terms in the right-hand side of Eq.(\ref{forward convergence: 1}) by
\begin{equation}
	\begin{aligned}
	&	L_{\pmb{\phi}} \|\bold{V}(t_L^{l}) \| \big| \mathrm{U}_L^{l} \pmb{\mathcal{A}} _{\pmb{\kappa} } (\mathrm{T}_L^l; \mathrm{d}) + \mathrm{a}_L^{l} - \bold{U}(t_L^l) \pmb{\mathcal{A}} _{\pmb{\kappa} }(\bold{T}(t_L^l); \mathrm{d}) - \bold{a}(t_L^l) \big| \\
\leq 	& L_{\pmb{\phi}} M_{\Theta} \big( M_{\pmb{\mathcal{A}} } |\mathrm{d}| \cdot 	\| \bar{\pmb{\mathcal{I}}}_{L} \mathrm{U}_L - \bold{U} \|_C + M_{\Theta} L_{\pmb{\mathcal{A}}} \| \bar{\pmb{\mathcal{I}}}_{L} \mathrm{T}_L - \bold{T} \|_C  +	\| \bar{\pmb{\mathcal{I}}}_{L} \mathrm{a}_L - \bold{a} \|_C \big),
	\end{aligned}
 \nonumber
\end{equation}
and 
\begin{align}
	& L_{\pmb{\phi}} \|\bold{V}(t_L^{l}) \| \cdot  \Big|\tau \sum_{k=0}^{l-1} [\mathrm{W}_L^{l,k+1}\pmb{\mathcal{A}} _{\pmb{\kappa} }(\mathrm{T}_L^l; \mathrm{x}_L^k)\! + \!\mathrm{c}_L^{l,k+1}] \nonumber \\
    & \qquad \qquad \qquad \qquad \qquad - \!\int_{0}^{t^l_L}\! [\bold{W}(t^l_L, s) \pmb{\mathcal{A}} _{\pmb{\kappa} }(\bold{T}(t_L^l); \bold{x}(s)) \!+\! \bold{c}(t^l_L, s)] {d}s \Big| \nonumber \\
		\leq & L_{\pmb{\phi}} \|\pmb{\Theta}\| \Biggl\{ \sum_{k=0}^{l-1} \! \int_{t_L^k}^{t_L^{k+1}} \!
		\big\|\mathrm{W}_L^{l,k+1}\big\| \big|\pmb{\mathcal{A}} _{\pmb{\kappa} }(\mathrm{T}_L^l; \mathrm{x}_L^k) \!-\! \pmb{\mathcal{A}} _{\pmb{\kappa} }(\bold{T}(t_L^l); \bold{x}(s))\big|\!  \nonumber \\
		& \ + 
		\big\| \mathrm{W}_L^{l,k+1}  \!-\! \bold{W}(t^l_L, s) \big\| \big| \pmb{\mathcal{A}} _{\pmb{\kappa} }(\bold{T}(t_L^l); \bold{x}(s)) \big| ds \! + \! \int_{0}^{t^l_L} \Big| (\bar{\pmb{\mathcal{BI}}}_{L} \bar{\mathrm{c}}_L) (t^l_L, s) - \bold{c}(t^l_L, s) \Big|ds \Biggr\} \nonumber \\
		\leq & 
		L_{\pmb{\phi}} \|\pmb{\Theta}\| \Biggl\{ \sum_{k=0}^{l-1}  
		\|{\Theta}_L\| L_{\pmb{\mathcal{A}} } \int_{t_L^k}^{t_L^{k+1}} \! \Big(\big| \mathrm{x}_L^k -  \bold{x}(s)\big| \! + \! \big\|\mathrm{T}_L^l - \bold{T}(t_L^l)\big\| \Big) ds \!  \nonumber \\
		& \ + \! M_{\pmb{\mathcal{A}} } B_{\bold{x}} \int_{t_L^k}^{t_L^{k+1}} \| \mathrm{W}_L^{l,k+1}  - \bold{W}(t^l_L, s) \|  ds \!+\! \int_{0}^{t^l_L} \Big| (\bar{\pmb{\mathcal{BI}}}_{L} \bar{\mathrm{c}}_L) (t^l_L, s) - \bold{c}(t^l_L, s) \Big|ds \Biggr\} \nonumber \\
		\leq & 
		L_{\pmb{\phi}} \|\pmb{\Theta}\|  \|{\Theta}_L\| L_{\pmb{\mathcal{A}} }  \sum_{k=0}^{l-1}  
		 \int_{t_L^k}^{t_L^{k+1}} \Big( \big| \mathrm{x}_L^k - \bold{x}(t_L^k) \big| \! + \! \big| \bold{x}(t_L^k) - \bold{x}(s)\big| \!+\! \big\|\mathrm{T}_L^l - \bold{T}(t_L^l)\big\| \Big)ds \! \nonumber \\
		 & +\! L_{\pmb{\phi}} \|\pmb{\Theta}\|  M_{\pmb{\mathcal{A}} } B_{\bold{x}} \! \int_{0}^{t^l_L} \big\| (\bar{\pmb{\mathcal{BI}}}_{L} \bar{\mathrm{W}}_L) (t^l_L, s)\! - \!\bold{W}(t^l_L, s) \big\|ds \nonumber \\
         & + \!L_{\pmb{\phi}} \|\pmb{\Theta}\|  \int_{0}^{t^l_L} \big| (\bar{\pmb{\mathcal{BI}}}_{L} \bar{\mathrm{c}}_L) (t^l_L, s) \!- \!\bold{c}(t^l_L, s) \big|ds \nonumber \\
		 \leq & \tau L_{\pmb{\phi}} L_{\pmb{\mathcal{A}} } M_{\Theta}^2 \sum_{k=0}^{l-1}  | \mathrm{x}_L^k \!- \!\bold{x}(t_L^k)|\! + \! L_{\pmb{\phi}} L_{\pmb{\mathcal{A}} } M_{\Theta}^2 \pmb{\omega}_{\bold{x}}(\tau) \! + \!L_{\pmb{\phi}}L_{\pmb{\mathcal{A}} } M_{\Theta}^2 \| \bar{\pmb{\mathcal{I}}}_{L} \mathrm{T}_L \!- \!\bold{T} \|_C \nonumber  \\
		 & + L_{\pmb{\phi}} M_{\Theta} M_{\pmb{\mathcal{A}}} B_{\bold{x}} \big\| \bar{\pmb{\mathcal{BI}}}_{L} \bar{\mathrm{W}}_L  - \bold{W} \big\|_{\mathcal{L}^\infty} + L_{\pmb{\phi}} M_{\Theta} \big\| \bar{\pmb{\mathcal{BI}}}_{L} \bar{\mathrm{c}}_L - \bold{c} \big\|_{\mathcal{L}^\infty}, \nonumber
\end{align}
where $\bar{\mathrm{W}}_L = \bold{flip}(\mathrm{W}_L), \bar{\mathrm{c}}_L = \bold{flip}(\mathrm{c}_L)$, $\pmb{\omega}_{\bold{x}}(\cdot)$ is the modulus of continuity of $\bold{x}$. 
Combining the above four inequalities, we obtain for $1\leq l \leq L$ that
\begin{equation}
	\begin{aligned}
		&| \mathrm{x}^{l}_L - \bold{x}(t^l_L)| \\
        \leq 
		& L_{\pmb{\phi}} \| \bar{\pmb{\mathcal{I}}}_{L} \mathrm{V}_L - \bold{V} \|_C \cdot M_{\Theta} \big(M_{\pmb{\mathcal{A}} } |\mathrm{d}| + M_{\pmb{\mathcal{A}} }B_{\rm x} + 2 \big) + \| \bar{\pmb{\mathcal{I}}}_{L} \mathrm{b}_L - \bold{b} \|_C \\
		& + L_{\pmb{\phi}} M_{\Theta} \big( M_{\pmb{\mathcal{A}} } |\mathrm{d}| \cdot 	\| \bar{\pmb{\mathcal{I}}}_{L} \mathrm{U}_L - \bold{U} \|_C + M_{\Theta} L_{\pmb{\mathcal{A}}} \| \bar{\pmb{\mathcal{I}}}_{L} \mathrm{T}_L - \bold{T} \|_C  +	\| \bar{\pmb{\mathcal{I}}}_{L} \mathrm{a}_L - \bold{a} \|_C \big) \\
		& + \tau L_{\pmb{\phi}} L_{\pmb{\mathcal{A}} } M_{\Theta}^2 \sum_{k=0}^{l-1}  | \mathrm{x}_L^k - \bold{x}(t_L^k)| + L_{\pmb{\phi}} L_{\pmb{\mathcal{A}}} M_{\Theta}^2 \pmb{\omega}_{\bold{x}}(\tau) + L_{\pmb{\phi}}L_{\pmb{\mathcal{A}} } M_{\Theta}^2 \| \bar{\pmb{\mathcal{I}}}_{L} \mathrm{T}_L - \bold{T} \|_C  \\
		& + L_{\pmb{\phi}}   M_{\Theta} M_{\pmb{\mathcal{A}} } B_{\bold{x}} \| \bar{\pmb{\mathcal{BI}}}_{L} \bar{\mathrm{W}}_L - \bold{W} \|_{\mathcal{L}^\infty}  + L_{\pmb{\phi}}  M_{\Theta} \| \bar{\pmb{\mathcal{BI}}}_{L} \bar{\mathrm{c}}_L - \bold{c} \|_{\mathcal{L}^\infty}  \\
         = & \tau L_{\pmb{\phi}} L_{\pmb{\mathcal{A}} } M_{\Theta}^2 \sum_{k=0}^{l-1} |\mathrm{x}_L^k - \bold{x}(t_L^k)| + C_L,
  \end{aligned}
 \label{equation: forward-convergence-proof}
\end{equation}
{with} $C_L = L_{\pmb{\phi}} L_{\pmb{\mathcal{A}}} M_{\Theta}^2 \pmb{\omega}_{\bold{x}}(\tau) + 2 L_{\pmb{\phi}}L_{\pmb{\mathcal{A}} } M_{\Theta}^2 \| \bar{\pmb{\mathcal{I}}}_{L} \mathrm{T}_L - \bold{T} \|_C + L_{\pmb{\phi}} M_{\Theta}  M_{\pmb{\mathcal{A}} } |\mathrm{d}| \| \bar{\pmb{\mathcal{I}}}_{L} \mathrm{U}_L - \bold{U} \|_C + L_{\pmb{\phi}} M_{\Theta}  \| \bar{\pmb{\mathcal{I}}}_{L} \mathrm{a}_L - \bold{a} \|_C + \! L_{\pmb{\phi}} M_{\Theta} \big(M_{\pmb{\mathcal{A}} } |\mathrm{d}| + M_{\pmb{\mathcal{A}} }B_{\rm x} + 2 \big) \| \bar{\pmb{\mathcal{I}}}_{L} \mathrm{V}_L \!-\! \bold{V} \|_C \!  +\!  \| \bar{\pmb{\mathcal{I}}}_{L} \mathrm{b}_L \!- \!\bold{b} \|_C \!+\! L_{\pmb{\phi}}   M_{\Theta} M_{\pmb{\mathcal{A}} } B_{\bold{x}} \| \bar{\pmb{\mathcal{BI}}}_{L} \bar{\mathrm{W}}_L \!-\! \bold{W} \|_{\mathcal{L}^\infty} \!+\!
L_{\pmb{\phi}}   M_{\Theta} \| \bar{\pmb{\mathcal{BI}}}_{L} \bar{\mathrm{c}}_L \!-\! \bold{c} \|_{\mathcal{L}^\infty}$. 
 
Note that 
 $$
 \begin{aligned}
 	| \mathrm{x}^{0}_L \!-\! \bold{x}(0)| 
 \leq & L_{\pmb{\phi}} \| \bar{\pmb{\mathcal{I}}}_{L} \mathrm{V}_L \!- \!\bold{V} \|_C M_{\Theta} \big(M_{\pmb{\mathcal{A}} } |\mathrm{d}|\! +\! 1 \big)\! +\! \| \bar{\pmb{\mathcal{I}}}_{L} \mathrm{b}_L \!-\! \bold{b} \|_C \!+ \! M_{\Theta} L_{\pmb{\phi}}  \big( \| \bar{\pmb{\mathcal{I}}}_{L} \mathrm{a}_L \! - \! \bold{a} \|_C \\
 	&  + M_{\pmb{\mathcal{A}} } |\mathrm{d}|	\| \bar{\pmb{\mathcal{I}}}_{L} \mathrm{U}_L \!-\! \bold{U} \|_C \!+\! M_{\Theta} L_{\pmb{\mathcal{A}}} \| \bar{\pmb{\mathcal{I}}}_{L} \mathrm{T}_L \! -\! \bold{T} \|_C  \big) \leq C_L.
 \end{aligned}
 $$
{Therefore, applying discrete Gronwall's inequality \cite[Lemma 100]{dragomir2002some} to Eq.(\ref{equation: forward-convergence-proof}) and using $C_L \rightarrow 0$ as $L \rightarrow  \infty$, we get}
$$
\begin{aligned}
	| \mathrm{x}^{l}_L - \bold{x}(t^l_L)| 
	 \leq C_L \Big[ L_{\pmb{\phi}} L_{\pmb{\mathcal{A}} } M_{\Theta}^2 \exp(L_{\pmb{\phi}} L_{\pmb{\mathcal{A}} } M_{\Theta}^2) + 1 \Big] \rightarrow 0, \ \text{as } L \rightarrow  \infty.
\end{aligned}
$$
Hence, $\sup_{t\in [t_{L}^{l-1}, t_{L}^{l}]} | \mathrm{x}^{l}_L - \bold{x}(t)| \leq | \mathrm{x}^{l}_L - \bold{x}(t_L^l) | + \pmb{\omega}_{\bold{x}}(\tau) \rightarrow 0$ as $L \rightarrow \infty$, and then we complete the proof.
\end{proof}

\begin{corollary}
\label{corollary: convergent rate}
Under the assumptions of Proposition~\ref{forward convergence}, suppose in addition that $\pmb{\Theta}\in \Omega_{\pmb{\Theta}}$ and that the discrete parameter $\Theta_L \in \Omega_{\Theta;L}$ is obtained by the following sampling strategy,
\[
\begin{aligned}
    (\mathrm{T}_L^l, \mathrm{U}_L^l, \mathrm{a}_L^l,  \mathrm{V}_L^l,  \mathrm{b}_L^l) 
= (\bold{T}(t_L^l), \bold{U}(t_L^l),  \bold{a}(t_L^l), \bold{V}(t_L^l), \bold{b}(t_L^l)), \ 0\le l\le L; \\
(\mathrm{W}_L^{l,k}, \mathrm{c}_L^{l,k}) = \Big(\frac{1}{\tau ^2} \int_{t_L^{l-1}}^{t_L^{l}} \int_{t_L^{k-1}}^{t_L^{k}} \bold{W}(t,s)dt ds,  \frac{1}{\tau ^2} \int_{t_L^{l-1}}^{t_L^{l}} \int_{t_L^{k-1}}^{t_L^{k}} \bold{c}(t,s)dt ds \Big), \ 1\le k \le l\le L.
\end{aligned}
\]
Then there exists a constant $C>0$ such that
$$
|\mathrm{x}^l(\mathrm{d}; {\Theta_L})- \bold{x}(t_L^l; \mathrm{d}; \pmb{\Theta}) | \leq C \tau^{1/3}, \ 1 \leq l \leq L.
$$
In particular, $\sup_{ {1 \leq l \leq L} } \sup_{t\in [t_{L}^{l-1}, t_{L}^{l}]} | \mathrm{x}^l(\mathrm{d}; {\Theta_L}) - \bold{x}(t; \mathrm{d}; \pmb{\Theta}) |$ converges to $0$ at rate $O(\tau^{1/3})$ as $L\to\infty$.
\end{corollary}
\begin{proof}
The proof is a small modification of the above proposition and is provided in the supplementary materials.
\end{proof}

Next, we study the learning problem of the DNL frameworks by using the results derived from the forward systems.
\begin{proposition} (Existence of solutions for learning problems)
Under the assumptions of Theorem~\ref{theorem: main results}, the minimizers of $(\mathcal{P}_L)$ and $(\mathcal{P})$ exist in $\Omega_{\Theta;L}$ and $\Omega_{\pmb{\Theta}}$, respectively. 
	\label{proposition: existence}
\end{proposition}
\begin{proof}
The proof is straightforward by the compactness of the learnable parameters, the weak lower semicontinuity of the loss functions, and the continuous dependency of the DNL framework on learnable parameters in Lemma~\ref{lemma: dis bound} and Proposition~\ref{proposition: forward-existence}. Here, we omit the details.
\end{proof}

We note that the functionals $\pmb{\mathfrak{L}}_{\mathcal{S};L}$, $L \in \mathbb{N}$, are defined on different function spaces. To study the convergence of optimal solutions of $(\mathcal{P}_L)$ to those of $(\mathcal{P})$ via  $\Gamma$-convergence (Definition~\ref{definition: gamma convergence} in Appendix), we need to
expand their feasible sets to a common space $\mathcal{C}_{\pmb{\Theta}}$ (defined in (\ref{parameter space C})). Inspired by \cite{ huang2024on, braides2002gamma}, 
we define the discrete-to-continuum extension functional $\tilde{\pmb{\mathfrak{L}}}_{\mathcal{S};L}: \mathcal{C}_{\pmb{\Theta}} \rightarrow [0, +\infty]$ as follows
\begin{equation}
	\tilde{\pmb{\mathfrak{L}}}_{\mathcal{S};L}(\pmb{\Theta})  =
	\left \{
	\begin{aligned}
		&\pmb{\mathfrak{L}}_{\mathcal{S};L}(\Theta_{L}),  
            \text{ if } \pmb{\Theta} = \hat{\pmb{\mathcal{I}}}_{L}\Theta_{L}, \\
		& +\infty, \text{ otherwise},
	\end{aligned}
	\right.
	\label{problem tilde P_L}
\end{equation}
and the functional $\tilde{\pmb{\mathfrak{L}}}: \mathcal{C}_{\pmb{\Theta}} \rightarrow [0, +\infty]$ as
\begin{equation}
	\tilde{\pmb{\mathfrak{L}}}(\pmb{\Theta})  =
	\left \{
	\begin{aligned}
		&\pmb{\mathfrak{L}}(\pmb{\Theta}),  \ \text{if } \pmb{\Theta} \in \Omega_{\pmb{\Theta}},  \\
		& +\infty, \text{ otherwise}.
	\end{aligned}
	\right.
	\label{problem tilde P}
\end{equation}
The following lemma shows the relationship between the optimal solution of $(\mathcal{P}_L)$ and the minimizer of $\tilde{\pmb{\mathfrak{L}}}_{\mathcal{S}; L}$, as well as the relationship between  $(\mathcal{P})$ and $\tilde{\pmb{\mathfrak{L}}}$. This enables us to study the $\Gamma$-convergence of $\tilde{\pmb{\mathfrak{L}}}_{\mathcal{S}; L}$ to $\tilde{\pmb{\mathfrak{L}}}$ to establish the relationship between the optimal solutions of $(\mathcal{P}_L)$ and $(\mathcal{P})$.

\begin{lemma} 
Consider the problems $(\mathcal{P}_{L}), (\mathcal{P})$ defined in (\ref{discrete-time-control problem P_L}) and (\ref{continuous-time: control problem P}), respectively. Let $\tilde{\pmb{\mathfrak{L}}}_{\mathcal{S};L}$ and $\tilde{\pmb{\mathfrak{L}}}$ be given in (\ref{problem tilde P_L}) and (\ref{problem tilde P}). Let $\pmb{\Theta}^* = (\bold{T}^*, \! \bold{U}^*, \! \bold{a}^*,\! \bold{V}^*, \!\bold{b}^*, \! \bold{W}^*, \!\bold{c}^*) \! \in \! \Omega_{\pmb{\Theta}}$, $\Theta_{L}^{*} = (\mathrm{T}_L^*, \mathrm{U}_L^*, \mathrm{a}_L^*, \mathrm{V}_L^*, \mathrm{b}_L^*, \mathrm{W}_L^*, \mathrm{c}_L^*) \in \Omega_{\Theta; L}$ for $L \ge 1$.  
Then 
\begin{itemize}
	\item[(i)] $\Theta_{L}^{*}$ is an optimal solution of $(\mathcal{P}_L)$ if and only if $\hat{\pmb{\mathcal{I}}}_{L}\Theta_L^*$ minimizes $\tilde{\pmb{\mathfrak{L}}}_{\mathcal{S}; L}$;
 \item[(ii)] $\pmb{\Theta}^*$ is an optimal solution of $(\mathcal{P})$ if and only if $\pmb{\Theta}^*$ minimizes $\tilde{\pmb{\mathfrak{L}}}$.
	\label{lemma: solution tilde_L and L_L}
 \end{itemize}
\end{lemma} 

\begin{proof} 
	The proof of this lemma is similar to \cite[Lemma 3]{huang2024on} and we omit it.
\end{proof}

We then verify the compactness of parameter set $\Omega_{\Theta; L}$ after linear extension operation.
\begin{lemma}
	Given $\{\Theta_L \in \Omega_{\Theta;L}: L \in \mathbb{N}\}$ such that 
	 $\sup_{L\in \mathbb{N}} \{ \pmb{\mathcal{R}}_L(\Theta_L)\}
 < + \infty,$ where $\pmb{\mathcal{R}}_L(\Theta_L)$ is defined in \eqref{equation: dis regularization}.
{Then there exists a subsequence of $\{\hat{\pmb{\mathcal{I}}}_{L}\Theta_{L} \}_{L \in \mathbb{N}}$ converging to a $\pmb{\Theta} \in \Omega_{\pmb{\Theta}}$ in $\mathcal{C}_{\pmb{\Theta}}$.}  
	\label{lemma: relative compactness}
\end{lemma}
\begin{proof}
We only need to prove the cases of $\{\hat{\pmb{\mathcal{I}}}_L\mathrm{U}_L\}_L$ and $\{ \hat{\pmb{\mathcal{BI}}}_L\bold{flip}(\mathrm{W}_L)\}_L$. 
Since $ \sup_{L\in \mathbb{N}} \{\pmb{\mathcal{R}}_L(\Theta_L)\} < + \infty$, there exists a constant $M<+\infty$ such that, for all $L\ge 1$,
$$
\begin{aligned}
    \max \Big\{ & \pmb{\mathcal{R}}_L^{(1)} ((\mathrm{T}_1)_{L}),  \pmb{\mathcal{R}}_L^{(1)} ((\mathrm{T}_2)_{L}),  \pmb{\mathcal{R}}_L^{(1)} ((\mathrm{T}_3)_{L}), 
 \pmb{\mathcal{R}}_L^{(1)} (\mathrm{U}_L),  \\
 & \ \ \pmb{\mathcal{R}}_L^{(2)} (\mathrm{a}_L),  \pmb{\mathcal{R}}_L^{(1)}(\mathrm{V}_L),  \pmb{\mathcal{R}}_L^{(2)} (\mathrm{b}_L), \pmb{\mathcal{R}}_L^{(3)} (\bar{\mathrm{W}}_L),  \pmb{\mathcal{R}}_L^{(4)} (\bar{\mathrm{c}}_L)   \Big \}   \leq   M,
\end{aligned}
$$
 where $\bar{\mathrm{W}}_L := \bold{flip}(\mathrm{W}_L), \bar{\mathrm{c}}_L := \bold{flip}(\mathrm{c}_L)$. By the definition of  $\{ \hat{\pmb{\mathcal{I}}}_L\mathrm{U}_L\}_L$, Jensen's inequality and $\tau \leq \frac{1}{\tau}$, we have 
\begin{equation}
	\begin{aligned}
		\| \hat{\pmb{\mathcal{I}}}_L\mathrm{U}_L \|^2_{\mathcal{L}^{2}([0,1]; \mathbb{R}^{n\times n})}
		\leq & \sum_{l=1}^{L}  \int_{t_L^{l-1}}^{t_L^{l}}  \frac{t_L^l -t}{\tau} \| \mathrm{U}_L^{l-1} \|^2 + \frac{t -t_L^{l-1}}{\tau} \| \mathrm{U}_L^{l}\|^2  {d}t\\
		\leq & \tau \sum_{l=1}^{L} \| \mathrm{U}_L^{l}\|^2 + \tau \| \mathrm{U}_L^{1}\|^2 + \frac{1}{\tau} \| \mathrm{U}_L^{1} - \mathrm{U}_L^{0}\|^2  , \\
		\| \pmb{\mathcal{D}}_t (\hat{\pmb{\mathcal{I}}}_L\mathrm{U}_L) \|^2_{\mathcal{L}^{2}([0,1]; \mathbb{R}^{n\times n})}  = & \sum_{l=1}^{L}  \int_{t_L^{l-1}}^{t_L^{l}} \Big\| \frac{\mathrm{U}_L^{l} -\mathrm{U}_L^{l-1}}{\tau} \Big\|^2 {d}t = \frac{1}{\tau} \sum_{l=1}^{L}\|\mathrm{U}_{L}^{l} - \mathrm{U}_{L}^{l-1}\|^{2}.
	\end{aligned}
	\nonumber
\end{equation}
Adding the above two inequalities together, we obtain $
\| \hat{\pmb{\mathcal{I}}}_L\mathrm{U}_L \|^2_{\mathcal{H}^{1}((0,1); \mathbb{R}^{n\times n})} \! \leq \! 2M,  \forall L\in \mathbb{N}$.	
Therefore, by Rellich-Kondrachov theorem \cite[Theorem 6.3]{adams2003sobolev}, {there exists a subsequence of $\{\hat{\pmb{\mathcal{I}}}_L\mathrm{U}_L\}_L$ converging to a $\bold{U} \in \mathcal{H}^{1}((0,1); \mathbb{R}^{n\times n})$ in $\mathcal{C}([0,1]; \mathbb{R}^{n\times n})$.}

As for the $\{\hat{\pmb{\mathcal{BI}}}_L \bold{flip}(\mathrm{W}_L)\}_L$. By Jensen's inequality again, we have
\begin{equation}
	\begin{aligned}
		&\| \hat{\pmb{\mathcal{BI}}}_L\bar{\mathrm{W}}_L \|^3_{\mathcal{L}^{3}([0,1] \times [0,1]; \mathbb{R}^{n\times n})} \\
        =& \sum_{l=1}^{L} \! \sum_{k=1}^{L} \! \int_{t_L^{l\!-\!1}}^{t_L^{l}}\int_{t_L^{k\!-\!1}}^{t_L^{k}}  \| (\hat{\pmb{\mathcal{BI}}}_L\bar{\mathrm{W}}_L)(t,s) \|^3 {d}t{d}s \! +\! \frac{s\!-\!t^{k-1}_L}{\tau} \big(\!  \frac{t^l_L\!-\!t}{\tau} \bar{\mathrm{W}}_L^{l\!-\!1,k} \!+\! \frac{t-t^{l-1}_L}{\tau} \! \bar{\mathrm{W}}_L^{l,k} \!\big)  \big\| ^3 {d}t{d}s   \\
		\leq & \sum_{l=1}^{L} \sum_{k=1}^{L} \int_{t_L^{l-1}}^{t_L^{l}}\int_{t_L^{k-1}}^{t_L^{k}} \frac{t^k_L -s}{\tau} \big(    \frac{t^l_L-t}{\tau} \| \bar{\mathrm{W}}_L^{l-1,k-1}\|^3 + \frac{t-t^{l-1}_L}{\tau}  \|\bar{\mathrm{W}}_L^{l,k-1}\|^3\big) \\
		& \qquad \qquad \qquad \qquad \quad + \frac{s-t^{k-1}_L}{\tau} \big(   \frac{t^l_L-t}{\tau} \|\bar{\mathrm{W}}_L^{l-1,k}\|^3 + \frac{t-t^{l-1}_L}{\tau}  \|\bar{\mathrm{W}}_L^{l,k}\|^3 \big)   {d}t{d}s   \\
		= & \frac{\tau^2}{4} \sum_{l=1}^{L} \sum_{k=1}^{L} (\|\bar{\mathrm{W}}_L^{l-1,k-1}\|^3 + \|\bar{\mathrm{W}}_L^{l,k-1}\|^3 + \|\bar{\mathrm{W}}_L^{l-1,k}\|^3 + \|\bar{\mathrm{W}}_L^{l,k}\|^3).
	\end{aligned}
	\nonumber
\end{equation}		
Similarly, we can estimate the weak derivation of $\hat{\pmb{\mathcal{BI}}}_L\bar{\mathrm{W}}_L$ as follows,
\begin{equation}
	\begin{aligned}
		& \| \pmb{\mathcal{D}}_s (\hat{\pmb{\mathcal{BI}}}_L\bar{\mathrm{W}}_L) \|^3_{\mathcal{L}^{3}([0,1] \times [0,1]; \mathbb{R}^{n\times n})} \\ 
		= & \tau^{-3} \sum_{l=1}^{L} \sum_{k=1}^{L} \int_{t_L^{l-1}}^{t_L^{l}} \int_{t_L^{k-1}}^{t_L^{k}} \big\|  \! \frac{t_L^{l} -t }{\tau} (\bar{\mathrm{W}}_{L}^{l-1,k} \! - \! \bar{\mathrm{W}}_{L}^{l-1,k-1}) \! +\!  \frac{t \! -\!  t_L^{l-1}}{\tau} (\bar{\mathrm{W}}_{L}^{l,k} \! - \! \bar{\mathrm{W}}_{L}^{l,k-1}) \big\| ^3\!  dt ds \\
		\leq & \frac{1}{2} \tau^{-1} \sum_{l=1}^{L} \sum_{k=1}^{L} ( \|\bar{\mathrm{W}}_{L}^{l-1,k} - \bar{\mathrm{W}}_{L}^{l-1,k-1}\|^3 + \|\bar{\mathrm{W}}_{L}^{l,k} - \bar{\mathrm{W}}_{L}^{l,k-1}\|^3  ), \\
		& \| \pmb{\mathcal{D}}_t (\hat{\pmb{\mathcal{BI}}}_L\bar{\mathrm{W}}_L) \|^3_{\mathcal{L}^{3}([0,1] \times [0,1]; \mathbb{R}^{n\times n})} \\
        \leq&  \frac{1}{2} \tau^{-1} \sum_{l=1}^{L} \sum_{k=1}^{L} ( \|\bar{\mathrm{W}}_{L}^{l,k-1} - \bar{\mathrm{W}}_{L}^{l-1,k-1}\|^3 +  \|\bar{\mathrm{W}}_{L}^{l,k} - \bar{\mathrm{W}}_{L}^{l-1,k}\|^3  ).
	\end{aligned}
 \nonumber
\end{equation}
Adding the above three inequalities together, we have	$
\| \hat{\pmb{\mathcal{BI}}}_L\! \bar{\mathrm{W}}_L \|^3_{\mathcal{W}^{1,3}((0,1) \!\times\! (0,1); \mathbb{R}^{n\times n})} \!\leq 4M$
uniformly for $L \in \mathbb{N}$. Then there exists a subsequence of $\{\hat{\pmb{\mathcal{BI}}}_L\bold{flip}(\mathrm{W}_L)\}_L$ converging to a $\bold{W} \in \mathcal{W}^{1,3}((0,1) \times (0,1); \mathbb{R}^{n\times n})$ in $\mathcal{L}^{\infty}([0,1] \times [0,1]; \mathbb{R}^{n\times n})$ owing to Rellich-Kondrachov theorem again.
\end{proof}

The next two lemmas will be used in proving the $\Gamma$-convergence of $\tilde{\pmb{\mathfrak{L}}}_{\mathcal{S}; L}$.
\begin{lemma} \cite[Proposition 4.8]{thorpe2018deep} 
	 Let $\bold{f}_L \in \mathcal{L}^2\left([0,1] ; \mathbb{R}^d\right), \bold{f} \in \mathcal{L}^2\left([0,1] ; \mathbb{R}^d\right)$ and $\varepsilon_L \rightarrow 0^{+}$ as $L \rightarrow \infty$. Assume that $\bold{f}_L \rightarrow \bold{f}$ in $\mathcal{L}^2\left([0,1] ; \mathbb{R}^d \right)$. If
	$$
	\liminf _{L \rightarrow \infty} \frac{1}{\varepsilon_L^2} \int_{\varepsilon_L}^1|\bold{f}_L(t)-\bold{f}_L\left(t-\varepsilon_L\right)|^2 {d} t<+\infty,
	$$
	then $\bold{f} \in \mathcal{H}^1\left((0,1) ; \mathbb{R}^d\right)$ and
	$$
	\liminf _{L \rightarrow \infty} \frac{1}{\varepsilon_L^2} \int_{\varepsilon_L}^1|\bold{f}_L(t)-\bold{f}_L\left(t-\varepsilon_L\right)|^2 {d}t \geq \int_0^1|\pmb{\mathcal{D}}_t(\bold{f})(t)|^2 {d}t .
	$$
	\label{lemma: limit-inf-1}
\end{lemma}

The following lemma generalizes Lemma~\ref{lemma: limit-inf-1} in a subtle way.
\begin{lemma}
	Let $\bold{f}_L \in \mathcal{L}^3\left([0,1]\times[0,1] ; \mathbb{R}^{d}\right), \bold{f} \in  \mathcal{L}^3\left([0,1]\times[0,1]; \mathbb{R}^{d}\right)$ and $\varepsilon_L \rightarrow 0^{+}$ as $L \rightarrow \infty$. Assume that $\bold{f}_L \rightarrow \bold{f}$ in $ \mathcal{L}^3\left([0,1]\times[0,1] ; \mathbb{R}^{d}\right)$. If
 \begin{equation}
 \begin{aligned}
      \liminf _{L \rightarrow \infty} \Big\{  & \frac{1}{\varepsilon_L^3} \int_{\varepsilon_L}^1 \int_{0}^1 |\bold{f}_L(t,s)\!-\!\bold{f}_L\left(t-\varepsilon_L, s\right)|^3 \! dt ds  \\
      & \ +  \frac{1}{\varepsilon_L^3}  \int_{0}^1 \int_{\varepsilon_L}^1|\bold{f}_L(t,s)\!-\!\bold{f}_L\left(t, s-\varepsilon_L\right)|^3 \! dt ds \Big\} <+\infty, 
 \end{aligned}
     \label{proposition-limit-inf:condition}
 \end{equation}
	then $\bold{f} \in \mathcal{W}^{1,3}\left((0,1)\times(0,1) ; \mathbb{R}^d \right)$ and
 \begin{equation}
     \begin{aligned}
	&	\liminf _{L \rightarrow \infty} \! \Big\{   \frac{1}{\varepsilon_L^3} \!\int_{\varepsilon_L}^1 \!\int_{0}^1 \!|\bold{f}_L\!(t,s)\!-\!\bold{f}_L\!\left(t\!-\!\varepsilon_L, s\right)\!|^3 \!dt  ds \!+\! \frac{1}{\varepsilon_L^3} \int_{0}^1 \int_{\varepsilon_L}^1 |\bold{f}_L\!(t,s)\!-\!\bold{f}_L \!\left(t, s\!-\!\varepsilon_L \right)\!|^3\! {d}t{d}s  \Big\}  \\
    & \geq \int_0^1\int_0^1 |\pmb{\mathcal{D}}_t{(\bold{f})(t,s)}|^3 + |\pmb{\mathcal{D}}_s{(\bold{f})(t,s)}|^3 dt ds.
	\end{aligned}
 \label{proposition-limit-inf:result}
 \end{equation}
	\label{proposition: limit-inf-2}
\end{lemma}

\begin{proof}
We first show the following inequalities
\begin{equation}
	\int_{\delta^{\prime}}^{1\!-\!\delta^{\prime}} \!\int_{\delta}^{1\!-\!\delta}\! |(J_\delta * \tilde{\bold{g}})(t,s)\!- \!(J_\delta * \tilde{\bold{g}})\left(t-\varepsilon_L, s\right)\!|^3 \!dt ds 
	\!\leq \!
	\int_{\varepsilon_L}^1\! \int_{0}^1  |\tilde{\bold{g}}(t, s)\!-\!\tilde{\bold{g}} \left(t-\varepsilon_L ,s \right)\!|^3 dt ds,
	\label{liminf: pro-1}
\end{equation}
\begin{equation}
	\int_{\delta}^{1-\delta} \int_{\delta^{\prime}}^{1-\delta^{\prime}} |(J_\delta * \tilde{\bold{g}})(t,s)\!-\!(J_\delta * \tilde{\bold{g}})\left(t, s-\varepsilon_L\right)|^3 dt ds 
	\leq 
	\int_{0}^1 \int_{\varepsilon_L}^1\!   |\tilde{\bold{g}}(t, s)\!-\!\tilde{\bold{g}} \left(t ,s\!-\!\varepsilon_L \right)|^3 dt ds,
	\label{liminf: pro-3}
\end{equation}
for any $\tilde{\bold{g}} \in \mathcal{L}^3\left([0,1]\times[0,1] ; \mathbb{R}^{d}\right)$ and any $\delta, \delta^{\prime}>0$ that satisfy $\varepsilon_L+\delta<\delta^{\prime}$, where $J_\delta$ is a standard 2D mollifier \cite{adams2003sobolev}; and
\begin{equation}
	\int_{2 \delta^{\prime}}^{1\!-\!2 \delta^{\prime}}\! \int_{2 \delta^{\prime}}^{1\!-\!2 \delta^{\prime}} \!|(\pmb{\mathcal{D}}_t \bold{g}) (t,s) |^3  dt ds \!\leq\! \liminf _{L \rightarrow \infty} \frac{1}{\varepsilon_L^3} \! \int_{2 \delta^{\prime}}^{1\!-\!2 \delta^{\prime}}
	\int_{2 \delta^{\prime}}^{1\!-\!2 \delta^{\prime}} \! |\bold{g}_L(t,s)\!-\!\bold{g}_L\left(t-\varepsilon_L,s\right)\!|^3 dt ds
	\label{liminf: pro-2}
\end{equation}
\begin{equation}
	\int_{2 \delta^{\prime}}^{1\!-\!2 \delta^{\prime}}  \int_{2 \delta^{\prime}}^{1\!-\!2 \delta^{\prime}} \! |\pmb{\mathcal{D}}_s(\bold{g}) (t,s) |^3  dt ds \!\leq\! \liminf _{L \rightarrow \infty} \!\frac{1}{\varepsilon_L^3} \!	\int_{2 \delta^{\prime}}^{1\!-\!2 \delta^{\prime}}  \int_{2 \delta^{\prime}}^{1\!-\!2 \delta^{\prime}}
	|\bold{g}_L(t,s)\!-\!\bold{g}_L\left(t,s-\varepsilon_L\right)\!|^3 dt ds,
	\label{liminf: pro-4}
\end{equation}
for any $\bold{g}, \bold{g}_L \in C^{\infty}\!\left(\left[\delta^{\prime}, 1\!-\!\delta^{\prime}\right]\!\times\!\left[\delta^{\prime}, 1\!-\!\delta^{\prime}\right] ;\! \mathbb{R}^d \right)$ with $\pmb{\mathcal{D}}_t(\bold{g}_L) \!\rightarrow \!\pmb{\mathcal{D}}_t(\bold{g})$, $ \pmb{\mathcal{D}}_s(\bold{g}_L) \!\rightarrow \!\pmb{\mathcal{D}}_s(\bold{g})$ in $\mathcal{L}^{\infty}\!\left(\left[\delta^{\prime}, 1\!-\!\delta^{\prime}\right]\!\times\!\left[\delta^{\prime}, 1\!-\!\delta^{\prime}\right] ;\! \mathbb{R}^d\right)$ and $\sup _L\!\left\|\pmb{\mathcal{D}}_{tt} (\bold{g}_L)\right\|_{\mathcal{L}^{\infty}}\!<\!\infty$, $ \sup _L\!\left\|\pmb{\mathcal{D}}_{ss} (\bold{g}_L)\right\|_{\mathcal{L}^{\infty}}\!<\!\infty$.

We only need to prove (\ref{liminf: pro-1}) and (\ref{liminf: pro-2}). Note that $\varepsilon_L+\delta<\delta^{\prime}$. To show (\ref{liminf: pro-1}), we have, by Minkowski's inequality for integrals \cite[Theorem 2.9]{adams2003sobolev} and the property of mollifier, that
\begin{equation}
	\begin{aligned}
		& \Big( \int_{\delta^{\prime}}^{1-\delta^{\prime}} \int_{\delta}^{1-\delta}|(J_\delta * \tilde{\bold{g}})(t,s)-(J_\delta * \tilde{\bold{g}})\left(t-\varepsilon_L, s\right)|^3 dt ds \Big) ^{1/3}  \\
		\leq & \Big( \int_{\delta^{\prime}}^{1-\delta^{\prime}} \int_{\delta}^{1-\delta} \Big|  \iint_{B(0,\delta)}  J_\delta (u,v) [  \tilde{\bold{g}}(t\!-\!u,s\!-\!v) \!- \!\tilde{\bold{g}}\left(t-\varepsilon_L\!-\!u, s\!-\!v\right) ] {d}u{d} v \Big|^3 dt ds \Big) ^{1/3} \\
		\leq &    \iint_{B(0,\delta)}  J_\delta (u,v) \Big( \!\int_{\delta^{\prime}}^{1-\delta^{\prime}}\! \int_{\delta}^{1-\delta}\! |  \tilde{\bold{g}}(t\!-\!u,s\!-\!v)\!- \!\tilde{\bold{g}}\left(t\!-\!\varepsilon_L\!-\!u, s\!-\!v\right)|^3 dt ds \Big) ^{1/3} {d} u {d} v  \\
		\leq &   \iint_{B(0,\delta)}  J_\delta (u,v) \Big( \int_{\varepsilon_L}^1 \int_{0}^1 |  \tilde{\bold{g}}(t,s)\!- \!\tilde{\bold{g}}\left(t\!-\!\varepsilon_L, s\right)  |^3 dt ds \Big)^{1/3} \! {d} u {d} v  \\
		= & \Big( \int_{\varepsilon_L}^1 \int_{0}^1  |\tilde{\bold{g}}(t, s)-\tilde{\bold{g}} \left(t-\varepsilon_L ,s \right)|^3 dt ds \Big) ^{1/3}.
	\end{aligned}
	\nonumber
\end{equation}

For inequality (\ref{liminf: pro-2}), the Taylor's theorem and the smoothness of $\bold{g}_L$ give
$$
\bold{g}_L(t,s)-\bold{g}_L\left(t-\varepsilon_L, s\right) = \varepsilon_L \pmb{\mathcal{D}}_t (\bold{g}_L)(t,s) - \varepsilon_L^2 \pmb{\mathcal{D}}_{tt}(\bold{g}_L)(r,s)  \text { for some } r \in\left[t-\varepsilon_L, t\right].
$$
Therefore, for $t \in\left[2 \delta^{\prime}, 1-2 \delta^{\prime}\right]$, $s \in [0, 1]$ and $\varepsilon_L<\delta^{\prime}$, we have
$$
\frac{|\bold{g}_L(t,s)-\bold{g}_L\left(t-\varepsilon_L,s\right)|}{\varepsilon_L} \geq |\pmb{\mathcal{D}}_t (\bold{g}_L)(t,s)|-\varepsilon_L\left\|\pmb{\mathcal{D}}_{tt}(\bold{g}_L)\right\|_{\mathcal{L}^{\infty}},
$$
due to the triangle inequality. For any $\eta>0$, there exists $C_\eta>0$ such that $|a+b|^3 \leq(1+\eta)|a|^3+C_\eta|b|^3$ for any $a, b \in \mathbb{R}$ by Young's inequality. Hence
$$
\begin{aligned}
	|\pmb{\mathcal{D}}_t(\bold{g}_L)(t,s)|^3 & \leq 
	\left(\frac{|\bold{g}_L(t,s)-\bold{g}_L\left(t-\varepsilon_L,s\right)|}{\varepsilon_L} +  \varepsilon_L\left\|\pmb{\mathcal{D}}_{tt}(\bold{g}_L)\right\|_{\mathcal{L}^{\infty}} \right)^3 \\
	& \leq (1+\eta)\Big|\frac{\bold{g}_L(t,s)-\bold{g}_L\left(t-\varepsilon_L,s\right)}{\varepsilon_L}\Big|^3+C_\eta \varepsilon_L^3\left\|\pmb{\mathcal{D}}_{tt}(\bold{g}_L)\right\|_{\mathcal{L}^{\infty}}^3 .
\end{aligned}
$$
Due to $\sup _{L \in \mathbb{N}}\left\|\pmb{\mathcal{D}}_{tt}(\bold{g}_L) \right\|_{\mathcal{L}^{\infty}} \!<\!\infty$ and Lebesgue's dominated convergence theorem, we see
$$
\begin{aligned}
	&\int_{2 \delta^{\prime}}^{1-2 \delta^{\prime}} \int_{2 \delta^{\prime}}^{1-2 \delta^{\prime}} |\pmb{\mathcal{D}}_t(\bold{g})(t,s)|^3 dt ds  =\lim _{L \rightarrow \infty} \int_{2 \delta^{\prime}}^{1-2 \delta^{\prime}} \int_{2 \delta^{\prime}}^{1-2 \delta^{\prime}} |\pmb{\mathcal{D}}_t(\bold{g}_L)(t,s)|^3 dt ds \\
	 \leq& (1+\eta) \liminf _{L \rightarrow \infty} \int_{2 \delta^{\prime}}^{1-2 \delta^{\prime}}\int_{2 \delta^{\prime}}^{1-2 \delta^{\prime}} \Big|\frac{\bold{g}_L(t,s)-\bold{g}_L\left(t-\varepsilon_L,s\right)}{\varepsilon_L}\Big|^3 dt ds.
\end{aligned}
$$
Taking $\eta \rightarrow 0$ yields (\ref{liminf: pro-2}). 

We next use (\ref{liminf: pro-1}), (\ref{liminf: pro-2}), (\ref{liminf: pro-3}), and (\ref{liminf: pro-4}) to prove the existence of weak derivatives of function $\bold{f}$. 
By (\ref{liminf: pro-1}) and (\ref{proposition-limit-inf:condition}), there exists a constant $M$ such that 
\begin{equation}
\begin{aligned}
   & \liminf _{L \rightarrow \infty} \frac{1}{\varepsilon_L^3} \int_{\delta^{\prime}}^{1-\delta^{\prime}} \int_{\delta}^{1-\delta} |(J_\delta * \bold{f}_L)(t,s)-(J_\delta * \bold{f}_L) \left(t-\varepsilon_L, s\right)|^3 dt ds \\
 \leq   &  \liminf _{L \rightarrow \infty} \frac{1}{\varepsilon_L^3} \int_{\varepsilon_L}^{1} \int_{0}^{1} |\bold{f}_L(t,s)-\bold{f}_L \left(t-\varepsilon_L, s\right)|^3 dt ds \leq M.
\end{aligned}
\label{equation: lemma-inf-M}
\end{equation}
Furthermore, by the property of mollifiers and H{\"o}lder inequality, we have, for $(t,s) \in [\delta^{\prime}, 1-\delta^{\prime}]\times[\delta^{\prime}, 1-\delta^{\prime}]$, that
$$
\begin{aligned}
	 | \pmb{\mathcal{D}}_t (J_\delta \!* \!\bold{f}_L)(t,s) \!-\! \pmb{\mathcal{D}}_t (J_\delta \!*\! \bold{f})(t,s) |\! =& \Big| \! \int_{0}^{1}\!\int_{0}^{1} \pmb{\mathcal{D}}_t (J_\delta)(t\!-\!u, s\!-\!v) (\bold{f}_L(u,v) \!-\! \bold{f}(u,v)) {d}u {d}v \Big| \\
   \leq & \|\pmb{\mathcal{D}}_t (J_\delta)\| _{\mathcal{L}^{3/2}(\mathbb{R}^2)} \|\bold{f}_L -\bold{f} \|_{\mathcal{L}^{3}([0, 1]\times[0, 1])}, \\
	 | \pmb{\mathcal{D}}_{tt} (J_\delta * \bold{f}_L)| \leq & \|\pmb{\mathcal{D}}_{tt}(J_\delta)\| _{\mathcal{L}^{3/2}(\mathbb{R}^2)} \|\bold{f}_L\|_{\mathcal{L}^{3}([0, 1]\times[0, 1])}.
\end{aligned}
$$
We can hence apply (\ref{liminf: pro-2}) with $\bold{g} = J_\delta * \bold{f}$ and $\bold{g}_L = J_\delta * \bold{f}_L$ and use (\ref{equation: lemma-inf-M}) to get 
$$
\int_{2 \delta^{\prime}}^{1-2 \delta^{\prime}} \int_{2 \delta^{\prime}}^{1-2 \delta^{\prime}}|\pmb{\mathcal{D}}_t(J_\delta*{\bold{f}}) (t,s)|^3  dt ds \leq M.
$$
Therefore, there exists an $\bold{h} \in \mathcal{L}^3([2 \delta^{\prime}, 1-2\delta^{\prime}]\times[2 \delta^{\prime}, 1-2\delta^{\prime}]; \mathbb{R}^{d})$ and a subsequence of $\{\pmb{\mathcal{D}}_t(J_\delta*{\bold{f}})\}_{\delta}$ (not relabeled) such that 
$$
 \pmb{\mathcal{D}}_t(J_\delta*{\bold{f}}) \rightharpoonup \bold{h} \text{ in } \mathcal{L}^{3}([2 \delta^{\prime}, 1-2\delta^{\prime}]\times[2 \delta^{\prime}, 1-2\delta^{\prime}]; \mathbb{R}^{d}), \text{ as } \delta \rightarrow 0^+
$$
by the reflexivity of $\mathcal{L}^3$. Hence, for any differentiable function $\pmb{\psi}$ with compact support in $[2\delta^{\prime}, 1-2\delta^{\prime}]\times [2\delta^{\prime}, 1-2\delta^{\prime}]$, we have
$$
\begin{aligned}
	& \int_{2 \delta^{\prime}}^{1-2 \delta^{\prime}}  \int_{2 \delta^{\prime}}^{1-2 \delta^{\prime}} \pmb{\psi}\bold{h} 
	\leftarrow \int_{2 \delta^{\prime}}^{1-2 \delta^{\prime}} \int_{2 \delta^{\prime}}^{1-2 \delta^{\prime}} \pmb{\psi} \pmb{\mathcal{D}}_t(J_\delta*{\bold{f}}) = - \int_{2 \delta^{\prime}}^{1-2 \delta^{\prime}} \int_{2 \delta^{\prime}}^{1-2 \delta^{\prime}}  \pmb{\mathcal{D}}_t (\pmb{\psi}) \cdot J_\delta*{\bold{f}} \\
    \rightarrow &   - \int_{2 \delta^{\prime}}^{1-2 \delta^{\prime}} \int_{2 \delta^{\prime}}^{1-2 \delta^{\prime}}  \pmb{\mathcal{D}}_t (\pmb{\psi}) \cdot \bold{f} =  \int_{2 \delta^{\prime}}^{1-2 \delta^{\prime}} \int_{2 \delta^{\prime}}^{1-2 \delta^{\prime}} \pmb{\psi}  \pmb{\mathcal{D}}_t (\bold{f}),
\end{aligned}
$$
where we use the property of mollifiers in the second line. This shows $\pmb{\mathcal{D}}_t(\bold{f}) = \bold{h}$ and in particular $\pmb{\mathcal{D}}_t (\bold{f}) \in \mathcal{L}^3([2 \delta^{\prime}, 1-2\delta^{\prime}]\times[2 \delta^{\prime}, 1-2\delta^{\prime}]).$ A similar discussion as above gives the existence of $\pmb{\mathcal{D}}_s(\bold{f})\in \mathcal{L}^3([2 \delta^{\prime}, 1-2\delta^{\prime}]\times[2 \delta^{\prime}, 1-2\delta^{\prime}])$.

Now, we prove (\ref{proposition-limit-inf:result}).
Applying (\ref{liminf: pro-2})  with $\bold{g} = J_\delta * \bold{f}$ and $\bold{g}_L = J_\delta * \bold{f}_L$, and by Eq.(\ref{liminf: pro-1}), we have
$$
\begin{aligned}
	& \int_{2 \delta^{\prime}}^{1-2 \delta^{\prime}}  \int_{2 \delta^{\prime}}^{1-2 \delta^{\prime}} | \pmb{\mathcal{D}}_t(J_\delta*{\bold{f}}) |^3 dt ds \\
	 \leq & \liminf _{L \rightarrow \infty} 	\int_{2 \delta^{\prime}}^{1-2 \delta^{\prime}}  \int_{2 \delta^{\prime}}^{1-2 \delta^{\prime}} \frac{|(J_\delta*{\bold{f}}_L)(t,s)-(J_\delta*{\bold{f}}_L)\left(t-\varepsilon_L,s\right)|^3}{\varepsilon_L^3}  dt ds \\
	 \leq & \liminf _{L \rightarrow \infty} 	\int_{\varepsilon_L}^{1}  \int_{0}^{1} \frac{|\bold{f}_L(t,s)-\bold{f}_L\left(t-\varepsilon_L,s\right)|^3}{\varepsilon_L^3}  dt ds.
\end{aligned}
$$
Similarly, we can get
$$
\begin{aligned}
  &\int_{2 \delta^{\prime}}^{1-2 \delta^{\prime}} 	\int_{2 \delta^{\prime}}^{1-2 \delta^{\prime}} | \pmb{\mathcal{D}}_s(J_\delta*{\bold{f}}) |^3 dt ds \\
	\leq &  \liminf _{L \rightarrow \infty}  \int_{2 \delta^{\prime}}^{1-2 \delta^{\prime}} 	\int_{2 \delta^{\prime}}^{1-2 \delta^{\prime}}  \frac{|(J_\delta*{\bold{f}}_L)(t,s)-(J_\delta*{\bold{f}}_L)\left(t,s-\varepsilon_L\right)|^3}{\varepsilon_L^3}  dt ds \\
	\leq & \liminf _{L \rightarrow \infty} 	  \int_{0}^{1} \int_{\varepsilon_L}^{1} \frac{|\bold{f}_L(t,s)-\bold{f}_L\left(t,s-\varepsilon_L\right)|^3}{\varepsilon_L^3}  dt ds.
\end{aligned}
$$
By the property of mollifiers, $\pmb{\mathcal{D}}_t(J_\delta*{\bold{f}})$ converges strongly to $\pmb{\mathcal{D}}_t (\bold{f})$ in $\mathcal{L}^3([2 \delta^{\prime}, 1-2\delta^{\prime}]\times[2 \delta^{\prime}, 1-2\delta^{\prime}]; \mathbb{R}^{d})$, $\pmb{\mathcal{D}}_s(J_\delta*{\bold{f}})$ converges strongly to $\pmb{\mathcal{D}}_s (\bold{f})$ in $\mathcal{L}^3([2 \delta^{\prime}, 1-2\delta^{\prime}]\times[2 \delta^{\prime}, 1-2\delta^{\prime}]; \mathbb{R}^{d})$ as $\delta \rightarrow 0+$. 
Hence, for any $\delta^{\prime} >0$, we have
$$ 
\begin{aligned}
    \int_{2 \delta^{\prime}}^{1-2 \delta^{\prime}}  \int_{2 \delta^{\prime}}^{1-2 \delta^{\prime}} | \pmb{\mathcal{D}}_t (\bold{f}) |^3 dt ds & \leq  \liminf _{L \rightarrow \infty} 	\int_{\varepsilon_L}^{1}  \int_{0}^{1} \frac{|\bold{f}_L(t,s)-\bold{f}_L\left(t-\varepsilon_L,s\right)|^3}{\varepsilon_L^3}  dt ds , \\
    \int_{2 \delta^{\prime}}^{1-2 \delta^{\prime}}\int_{2 \delta^{\prime}}^{1-2 \delta^{\prime}}  |\pmb{\mathcal{D}}_s(\bold{f})|^3 dt ds & \leq \liminf _{L \rightarrow \infty} 	\int_{0}^{1} \int_{\varepsilon_L}^{1}   \frac{|\bold{f}_L(t,s)-\bold{f}_L\left(t,s-\varepsilon_L\right)|^3}{\varepsilon_L^3}  dt ds ,
\end{aligned}
$$
since the additional constraint imposed by $\delta^{\prime} >\delta + \varepsilon_L$ vanishes when taking $\delta \rightarrow 0$ and $\varepsilon_L \rightarrow 0 $.
Adding the above two inequalities together and taking $\delta^{\prime} \rightarrow 0^+$, we complete the proof.
\end{proof}

Based on these preceding results, we can verify the ${\rm \liminf}$ condition of $\Gamma$-convergence for $\{\tilde{\pmb{\mathfrak{L}}}_{\mathcal{S}; L}\}_L$. 

\begin{lemma}
(Liminf condition)	Let the assumptions in Theorem~\ref{theorem: main results} hold. Then, for every learnable parameter function $\pmb{\Theta} \in \mathcal{C}_{\pmb{\Theta}}$ and every sequence $\{\pmb{\Theta}_L\}_{L \in \mathbb{N}}$ converging to $\pmb{\Theta}$ in $\mathcal{C}_{\pmb{\Theta}}$, we have 
$$ \liminf_{L \rightarrow \infty}\tilde{\pmb{\mathfrak{L}}}_{\mathcal{S};L}(\pmb{\Theta}_L) \ge \tilde{\pmb{\mathfrak{L}}}(\pmb{\Theta}),$$
where $\tilde{\pmb{\mathfrak{L}}}_{\mathcal{S};L}$ and $\tilde{\pmb{\mathfrak{L}}}$ are given in (\ref{problem tilde P_L}) and (\ref{problem tilde P}), respectively.
\label{lemma: liminf condition}
\end{lemma}

\begin{proof}
	If $\tilde{\pmb{\mathfrak{L}}}(\pmb{\Theta}) = +\infty$, i.e., $\pmb{\Theta} \notin \Omega_{\pmb{\Theta}}$, we show $\liminf _{L \rightarrow \infty} \tilde{\pmb{\mathfrak{L}}}_{\mathcal{S};L}(\pmb{\Theta}_L) = +\infty$ by contradiction. If $\liminf _{L \rightarrow \infty} \tilde{\pmb{\mathfrak{L}}}_{\mathcal{S};L}(\pmb{\Theta}_L) < +\infty$, then there exists a subsequence $\{\pmb{\Theta}_{L_{k}}\}_{L_{k}} \subset \{\pmb{\Theta}_{L}\}_{L}$ such that
$\lim _{L_k \rightarrow \infty}\tilde{\pmb{\mathfrak{L}}}_{\mathcal{S};L}(\pmb{\Theta}_{L_k}) < +\infty$. Without loss of generality, we can assume
$\pmb{\Theta}_{L_k} = \hat{\pmb{\mathcal{I}}}_{L_k}\Theta_{L_k}$ for some $ \Theta_{L_k} \in \Omega_{\Theta; L}$ by the definition of $\tilde{\pmb{\mathfrak{L}}}_{\mathcal{S};L}$ in (\ref{problem tilde P_L}). 
Then, there exists a subsequence $\{\pmb{\Theta}_{L_{k_r}}\}_{L_{k_r}}$ converging to some $\hat{\pmb{\Theta}} \in \Omega_{\pmb{\Theta}}$ in $\mathcal{C}_{\pmb{\pmb\Theta}}$ due to Lemma~\ref{lemma: relative compactness}. This contradicts $\pmb{\Theta}_{L_k} \rightarrow \pmb{\Theta} \notin \Omega_{\pmb{\Theta}}$. 

If $\tilde{\pmb{\mathfrak{L}}}(\pmb{\Theta}) < +\infty$, we only need to consider the case when $\liminf_{L \rightarrow \infty} \tilde{\pmb{\mathfrak{L}}}_{\mathcal{S};L}(\pmb{\Theta}_L) < +\infty$.
Assume $\pmb{\Theta}_{L} = \hat{\pmb{\mathcal{I}}}_{L}\Theta_{L}$ for some $\Theta_{L} \in \Omega_{\Theta; L}$ by the definition of $\tilde{\pmb{\mathfrak{L}}}_{\mathcal{S};L}$.
Due to the continuity of $\pmb{\ell}$ (assumption (${A}_3$)) and Proposition~\ref{forward convergence}, we next prove
\begin{equation}
	\frac{1}{M} \sum\limits_{m=1}^{M} \pmb{\ell} (\mathrm{x}^L(\mathrm{d}_m; \Theta_L), \mathrm{g}_m) \rightarrow \frac{1}{M} \sum\limits_{m=1}^{M} \pmb{\ell} (\bold{x}(1; \mathrm{d}_m;\pmb{\Theta}) , \mathrm{g}_m)  \text{ as } L \rightarrow \infty,
	\label{equation: convergence of data loss}
\end{equation}
by verifying the condition (\ref{convergence of parameter}). Note that we only need to show 
\begin{equation}
	\bar{\pmb{\mathcal{I}}}_{L} {\mathrm{U}}_L \rightarrow \bold{U} \text{ in } \mathcal{C}([0,1];\mathbb{R}^{n\times n}); \ \bar{\pmb{\mathcal{BI}}}_{L} (\bold{flip}(\mathrm{W}_{L})) \rightarrow \bold{W} \text{ in } \mathcal{L}^{\infty}([0,1]\times[0,1];\mathbb{R}^{n\times n}).
	\label{equation: convergence of bar U&W}
\end{equation}

For any $\epsilon >0$, there exists an $\check{L} \in \mathbb{N}$ such that $\| \hat{\pmb{\mathcal{I}}}_{L} {\mathrm{U}}_L - \bold{U} \|_C < \epsilon/2$, $\forall L>\check{L}$, because $\{\pmb{\Theta}_L\}_{L \in \mathbb{N}}$ converges to $\pmb{\Theta}$ in $\mathcal{C}_{\pmb{\Theta}}$. Besides, there exists an $\tilde{L} \in \mathbb{N}$ such that $\max_{|t_1 - t_2| \leq 1/L} {\| \bold{U}(t_1) - \bold{U}(t_2) \| } \leq \epsilon/2$ for $L\ge \tilde{L}$ due to the uniform continuity of $\bold{U}$ in $[0, 1]$. Consequently, for all $L \ge \max \{\check{L}, \tilde{L} \}$, we have
$$
\begin{aligned} 
	\| (\bar{\pmb{\mathcal{I}}}_{L} {\mathrm{U}}_L) (0) - \bold{U}(0) \| &= \| (\hat{\pmb{\mathcal{I}}}_{L}) {\mathrm{U}}_L (0) - \bold{U}(0) \| < \epsilon, \\
	\| (\bar{\pmb{\mathcal{I}}}_{L} {\mathrm{U}}_L) (t) - \bold{U}(t) \| & 
	\leq \| (\hat{\pmb{\mathcal{I}}}_{L} {\mathrm{U}}_L) (t_L^l) - \bold{U}(t_L^l) \| + \| \bold{U}(t_L^l) - \bold{U}(t) \|  < \epsilon, \  \ t \in (t_L^{l-1}, t_L^l], 
\end{aligned}
$$
where $1 \leq l \leq L$.
Denote $\bar{\mathrm{W}}_L = \bold{flip}(\mathrm{W}_{L})$. Since $\liminf_{L \rightarrow \infty} \tilde{\pmb{\mathfrak{L}}}_{\mathcal{S};L}(\pmb{\Theta}_L) < +\infty$, there exists a constant $M_d$ such that for all $L \in \mathbb{N}$, 
\begin{equation}
\begin{aligned}
&	\tau^{-1} \sum_{l=1}^{L}\|\mathrm{U}_{L}^{l} \! - \! \mathrm{U}_{L}^{l-1}\|^{2}  \leq M_d, \\
&	\tau^{-1} \Big(\sum_{l=1}^{L} \sum_{k=1}^{L}  \|\bar{\mathrm{W}}_{L}^{l,k} - \bar{\mathrm{W}}_{L}^{l-1,k}\|^{3} \! + \! \sum_{l=1}^{L} \sum_{k=1}^{L}  \|\bar{\mathrm{W}}_{L}^{l,k} - \bar{\mathrm{W}}_{L}^{l,k-1}\|^{3} \Big)  \leq M_d,
	\end{aligned}
	\label{equation: bounded difference of U&W}
\end{equation}
due to $\liminf_{L \rightarrow \infty} \tilde{\pmb{\mathfrak{L}}}_{\mathcal{S};L}(\pmb{\Theta}_L) < +\infty$. This yields, for any $1\leq k,l \leq L$, that
\begin{equation}
	\begin{aligned}
	& \sup_{L \in  \mathbb{N}}	\sup_{{1 \leq k,l \leq L}} \max\{\|\bar{\mathrm{W}}_{L}^{l,k} - \bar{\mathrm{W}}_{L}^{l-1,k}\|, \|\bar{\mathrm{W}}_{L}^{l,k} - \bar{\mathrm{W}}_{L}^{l,k-1}\|\} \leq  (M_d\tau)^{1/3}. \\
	\end{aligned}
	\nonumber
\end{equation}
Hence, we derive
$$
\begin{aligned} 
	& \| (\bar{\pmb{\mathcal{BI}}}_{L} \bar{\mathrm{W}}_L) (t,s) - \bold{W}(t,s) \| \\
	\leq & \| (\hat{\pmb{\mathcal{BI}}}_{L} \bar{\mathrm{W}}_L) (t_L^l,t_L^k) - (\hat{\pmb{\mathcal{BI}}}_{L} \bar{\mathrm{W}}_L) (t,s) \| +  \|(\hat{\pmb{\mathcal{BI}}}_{L} \bar{\mathrm{W}}_L) (t,s) - \bold{W}(t,s) \| \\
 \leq & \Big\|\bar{\mathrm{W}}_L^{l,k} - \bar{\mathrm{W}}_L^{l,k-1} \| +  \Big\|\bar{\mathrm{W}}_L^{l,k-1} - \bar{\mathrm{W}}_L^{l-1,k-1} \| +  \Big\|\bar{\mathrm{W}}_L^{l,k} - \bar{\mathrm{W}}_L^{l,k-1} \| +  \Big\|\bar{\mathrm{W}}_L^{l,k} - \bar{\mathrm{W}}_L^{l-1,k} \| \\
 &   + \|(\hat{\pmb{\mathcal{BI}}}_{L} \bar{\mathrm{W}}_L) (t,s) - \bold{W}(t,s) \| \\
 \leq & 4 (M_d)^{1/3}\tau^{1/3} + \|(\hat{\pmb{\mathcal{BI}}}_{L} \bar{\mathrm{W}}_L) (t,s) - \bold{W}(t,s) \|,
\end{aligned}
$$
where $(t,s) \in (t_L^{l-1}, t_L^l] \times (t_L^{k-1}, t_L^k]$ ($1\leq l,k \leq L$) is the Lebesgue points of $\bold{W}$, and we use $(\bar{\pmb{\mathcal{BI}}}_{L} \bar{\mathrm{W}}_L) (t,s) = (\hat{\pmb{\mathcal{BI}}}_{L} \bar{\mathrm{W}}_L) (t_L^l,t_L^k)$ for $(t,s) \in (t_L^{l-1}, t_L^l] \times (t_L^{k-1}, t_L^k]$ in the first inequality.
Therefore, $\bar{\pmb{\mathcal{BI}}}_{L} \bar{\mathrm{W}}_L \rightarrow \bold{W}$ in $\mathcal{L}^{\infty}([0,1]\times[0,1];\mathbb{R}^{n\times n})$ due to  $\hat{\pmb{\mathcal{BI}}}_{L} \bar{\mathrm{W}}_L \rightarrow \bold{W}$ in $\mathcal{L}^{\infty}([0,1]\times[0,1];\mathbb{R}^{n\times n})$.

To prove $ \liminf_{L \rightarrow \infty} \!\tilde{\pmb{\mathfrak{L}}}_{\mathcal{S};L}(\!\pmb{\Theta}_L\!) \!\ge\! \tilde{\pmb{\mathfrak{L}}}(\!\pmb{\Theta}\!)$, we now only need to show $ \liminf_{L \rightarrow \infty}\!\pmb{\mathcal{R}}_L(\!\Theta_L\!) \!\ge \pmb{\mathcal{R}}(\pmb{\Theta})$ since (\ref{equation: convergence of data loss}) holds.
To achieve this goal, we show the following inequalities: \\
 (i) $
\liminf_{L \rightarrow \infty} \pmb{\mathcal{R}}_L^{(1)}((\mathrm{T}_j)_{L})  \ge \|(\bold{T})_j\|^2_{\mathcal{H}^{1}((0, 1); \mathbb{R}^{n\times n})}$, $j=1,2,3$; \\
(ii) $
\liminf_{L \rightarrow \infty} \pmb{\mathcal{R}}_L^{(1)}(\mathrm{U}_L)  \ge \|\bold{U}\|^2_{\mathcal{H}^{1}((0, 1); \mathbb{R}^{n\times n})}$, $\liminf_{L \rightarrow \infty} \pmb{\mathcal{R}}_L^{(2)}(\mathrm{a}_L)  \ge \|\bold{a}\|^2_{\mathcal{H}^{1}((0, 1); \mathbb{R}^{n})}$; \\
(iii)$\liminf_{L \rightarrow \infty} \pmb{\mathcal{R}}_L^{(1)}(\mathrm{V}_L)  \ge \|\bold{V}\|^2_{\mathcal{H}^{1}((0, 1); \mathbb{R}^{n\times n})}$, $\liminf_{L \rightarrow \infty} \pmb{\mathcal{R}}_L^{(2)}(\mathrm{b}_L)  \ge \|\bold{b}\|^2_{\mathcal{H}^{1}((0, 1); \mathbb{R}^{n})}$; \\
(iv) $\liminf_{L \rightarrow \infty} \pmb{\mathcal{R}}_L^{(3)}(\bold{flip}(\mathrm{W}_{L}))  \ge \|\bold{W}\|^3_{\mathcal{W}^{1,3}((0, 1) \times (0, 1); \mathbb{R}^{n\times n})}$, \\
\text{\quad \ \ } $\liminf_{L \rightarrow \infty} \pmb{\mathcal{R}}_L^{(4)}(\bold{flip}(\mathrm{c}_{L}))  \ge \|\bold{c}\|^3_{\mathcal{W}^{1,3}((0, 1) \times (0, 1); \mathbb{R}^{n})}$.

Note that we only need to show the $
\liminf_{L \rightarrow \infty} \pmb{\mathcal{R}}_L^{(1)}(\mathrm{U}_L)  \ge \|\bold{U}\|^2_{\mathcal{H}^{1}((0, 1); \mathbb{R}^{n\times n})}$ and $\liminf_{L \rightarrow \infty} \pmb{\mathcal{R}}_L^{(3)}(\bold{flip}(\mathrm{W}_{L})) \! \ge \! \|\bold{W}\|^3_{\mathcal{W}^{1,3}((0, 1) \times (0, 1); \mathbb{R}^{n\times n})}$, as the other cases are similar. 
By Definition~\ref{definition: linear extention}, the convergence of $\{ \bar{\pmb{\mathcal{I}}}_{L} {\mathrm{U}}_L\}_L$ in Eq.(\ref{equation: convergence of bar U&W}), Eq.(\ref{equation: bounded difference of U&W}) and Lemma~\ref{lemma: limit-inf-1}, we have
\begin{align}
&\liminf_{L \rightarrow \infty} \pmb{\mathcal{R}}_L^{(1)}(\mathrm{U}_L) \nonumber\\
\ge & \liminf_{L \rightarrow \infty} \tau \sum_{l=1}^{L}\|\mathrm{U}_{L}^{l}\|^{2} + \liminf_{L \rightarrow \infty}	\frac{1}{\tau} \sum_{l=1}^{L}\|\mathrm{U}_{L}^{l} - \mathrm{U}_{L}^{l-1}\|^{2} \nonumber \\
\ge &  \liminf_{L \rightarrow \infty} \int_{0}^{1} \| (\bar{\pmb{\mathcal{I}}}_{L}\mathrm{U}_L)(t)\|^2 dt + \liminf_{L \rightarrow \infty} \frac{1}{\tau^2} \sum_{l=2}^{L} \int_{t_L^{l-1}}^{t_L^{l}} \|(\bar{\pmb{\mathcal{I}}}_{L}\mathrm{U}_L)(t) - (\bar{\pmb{\mathcal{I}}}_{L}\mathrm{U}_L)(t - \tau)\|^{2}dt \nonumber \\
\ge & \int_{0}^{1} \| \bold{U}(t)\|^2 dt + \int_{0}^{1} \| \pmb{\mathcal{D}}_t(\bold{U})(t)\|^2 dt = \|\bold{U}\|^2_{\mathcal{H}^{1}((0, 1); \mathbb{R}^{n\times n})}. \nonumber
\end{align}
Similarly, we can estimate $\liminf_{L \rightarrow \infty} \pmb{\mathcal{R}}_L^{(3)}(\bar{\mathrm{W}}_L)$ as
\begin{equation}
    \begin{aligned}
	& \liminf_{L \rightarrow \infty} \pmb{\mathcal{R}}_L^{(5)}(\bar{\mathrm{W}}_L)  \nonumber\\
	\ge & \liminf_{L \rightarrow \infty} \tau^2 \sum_{l=1}^{L} \!\sum_{k=1}^{L}\|\bar{\mathrm{W}}_{L}^{l,k}\|^{3} \! \\
    & + \!\liminf_{L \rightarrow \infty} \tau^{-1} \Big(\sum_{l=2}^{L} \sum_{k=1}^{L} \! \|\bar{\mathrm{W}}_{L}^{l,k} \!- \!\bar{\mathrm{W}}_{L}^{l-1,k}\|^{3} \!+\! \sum_{l=1}^{L}\! \sum_{k=2}^{L}  \!\|\bar{\mathrm{W}}_{L}^{l,k} \!- \! \bar{\mathrm{W}}_{L}^{l,k-1}\|^{3} \Big) \nonumber \\
	= &  \liminf _{L \rightarrow \infty} \int_{0}^{1} \int_{0}^{1}	\|(\bar{\pmb{\mathcal{BI}}}_{L} \bar{\mathrm{W}}_L) (t,s) \|^3 dt ds   \\
	& + \liminf _{L \rightarrow \infty} \Big\{ L^3 \int_{{1}/{L}}^1 \int_{0}^1 \|(\bar{\pmb{\mathcal{BI}}}_{L} \bar{\mathrm{W}}_L) (t,s) -(\bar{\pmb{\mathcal{BI}}}_{L} \bar{\mathrm{W}}_L) \left(t-1/{L}, s\right)\|^3 \!  dt {d}s  \\
    & + L^3 \int_{0}^1 \int_{1/{L}}^1 \left\|(\bar{\pmb{\mathcal{BI}}}_{L} \bar{\mathrm{W}}_L)(t,s)-(\bar{\pmb{\mathcal{BI}}}_{L} \bar{\mathrm{W}}_L)\left(t, s-1/{L} \right)\right\|^3  dt ds \Big\} \\
	\ge &  \int_{0}^{1} \int_{0}^{1} \|\bold{W}(t,s)\|^3 dt ds +
	\int_0^1\int_0^1  \Big(\|\pmb{\mathcal{D}}_t( \bold{W} (t,s))\|^3 + \|\pmb{\mathcal{D}}_s(\bold{W} (t,s))\|^3 \Big) dt ds  \\
	= & \|\bold{W}\|^3_{\mathcal{W}^{1,3}((0, 1) \times (0, 1); \mathbb{R}^{n\times n})}, 
\end{aligned}
\end{equation}
where we use the convergence of $\bar{\pmb{\mathcal{BI}}}_{L} \bar{\mathrm{W}}_L$ in (\ref{equation: convergence of bar U&W}), Eq.(\ref{equation: bounded difference of U&W}) and Lemma~\ref{proposition: limit-inf-2} in the last inequality. This completes the proof.
\end{proof}

Next, we verify the ${\rm \limsup}$ condition of $\Gamma$-convergence for $\tilde{\pmb{\mathfrak{L}}}_{\mathcal{S}; L}$. For any $\pmb{\Theta} \in \Omega_{\pmb{\Theta}}$ and $L \in \mathbb{N}$, we define the learnable parameter $\Theta_L \in \Omega_{\Theta; L}$ by 
\begin{align}
	 (\mathrm{T}_j\!)_{L}^{0} &= \!\frac{1}{\tau} \!\int_{t_L^0}^{t_L^1}\! (\bold{T}_j\!)\!(t) dt, 
	(\mathrm{T}_j\!)_{L}^{l} \! =\! \frac{1}{2\tau} \!\int_{t_L^{l\!-\!1}}^{t_L^{l\!+\!1}} \!(\bold{T}_j\!)\!(t) dt,  1 \!\leq l\! \leq L\!-\!1, 
	(\mathrm{T}_j\!)_{L}^{L}\! =\! \frac{1}{\tau} \!\int_{t_L^{L\!-\!1}}^{t_L^L}\! (\bold{T}_j\!) \!(t) dt, \nonumber \\
	\mathrm{U}_L^{0} &= \frac{1}{\tau} \int_{t_L^0}^{t_L^1} \bold{U}(t) dt, 
	\mathrm{U}_L^{l}  = \frac{1}{2\tau} \int_{t_L^{l-1}}^{t_L^{l+1}} \bold{U}(t) dt,  1 \leq l \leq L-1, 
	\mathrm{U}_L^{L} = \frac{1}{\tau} \int_{t_L^{L-1}}^{t_L^L} \bold{U}(t) dt, \nonumber\\
	\mathrm{a}_L^{0} &= \frac{1}{\tau} \int_{t_L^0}^{t_L^1} \bold{a}(t) dt, 
	\mathrm{a}_L^{l}  = \frac{1}{2\tau} \int_{t_L^{l-1}}^{t_L^{l+1}} \bold{a}(t) dt,  1 \leq l \leq L-1, 
	\mathrm{a}_L^{L} = \frac{1}{\tau} \int_{t_L^{L-1}}^{t_L^L} \bold{a}(t) dt, \nonumber \\
	\mathrm{V}_L^{0} &= \frac{1}{\tau} \int_{t_L^0}^{t_L^1} \bold{V}(t) dt, 
	\mathrm{V}_L^{l}  = \frac{1}{2\tau} \int_{t_L^{l-1}}^{t_L^{l+1}} \bold{V}(t) dt, 1 \leq l \leq L-1, 
	\mathrm{V}_L^{L} = \frac{1}{\tau} \int_{t_L^{L-1}}^{t_L^L} \bold{V}(t) dt, \nonumber \\
	\mathrm{b}_L^{0} &= \frac{1}{\tau} \int_{t_L^0}^{t_L^1} \bold{b}(t) dt, 
	\mathrm{b}_L^{l}  = \frac{1}{2\tau} \int_{t_L^{l-1}}^{t_L^{l+1}} \bold{b}(t) dt,  1 \leq l \leq L-1, 
	\mathrm{b}_L^{L} = \frac{1}{\tau} \int_{t_L^{L-1}}^{t_L^L} \bold{b}(t) dt, \nonumber\\
	\mathrm{W}_L^{l,k} &= \frac{1}{\tau ^2} \int_{t_L^{l-1}}^{t_L^{l}} \int_{t_L^{k-1}}^{t_L^{k}} \bold{W}(t,s)dt ds,  1 \leq k \leq l \leq L, \label{recovery sequence} \\
	\mathrm{c}_L^{l,k} &= \frac{1}{\tau ^2} \int_{t_L^{l-1}}^{t_L^{l}} \int_{t_L^{k-1}}^{t_L^{k}} \bold{c}(t,s)dt ds,  1 \leq k \leq l \leq L. \nonumber
\end{align}

\begin{lemma}
	(Limsup condition) Let the assumptions in Theorem~\ref{theorem: main results} hold. Then, for every learnable parameter function $\pmb{\Theta} \in \mathcal{C}_{\pmb{\Theta}} $, there exists
a sequence $\{\pmb{\Theta}_L\}_{L \in \mathbb{N}}$ converging to $\pmb{\Theta}$ in $\mathcal{C}_{\pmb{\Theta}}$ such that
	$$ \limsup_{L \rightarrow \infty}\tilde{\pmb{\mathfrak{L}}}_{\mathcal{S};L}(\pmb{\Theta}_L) \leq \tilde{\pmb{\mathfrak{L}}}(\pmb{\Theta}) ,$$
	where $\tilde{\pmb{\mathfrak{L}}}_{\mathcal{S};L}$ and $\tilde{\pmb{\mathfrak{L}}}$ are given in (\ref{problem tilde P_L}) and (\ref{problem tilde P}), respectively.
 \label{lemma: limsup}
\end{lemma}
\begin{proof}
We only need to consider the case when $\tilde{\pmb{\mathfrak{L}}}(\pmb{\Theta}) < +\infty$, i.e., $\pmb{\Theta} \in \Omega_{\pmb{\Theta}}$. For every given $\pmb{\Theta} \in \Omega_{\pmb{\Theta}}$, let $\{\Theta_L\}_L$ be given by (\ref{recovery sequence}) associated with $\pmb{\Theta}$, and let $\bar{\mathrm{W}}_L := \bold{flip}(\mathrm{W}_L), \bar{\mathrm{c}}_L := \bold{flip}(\bold{c}_L)$, $\pmb{\Theta}_L := \hat{\pmb{\mathcal{I}}}_{L} \Theta_{L}$. We next prove that $\pmb{\Theta}_L$ meets the requirements.

We show that the sequence $\{\pmb{\Theta}_L\}_{L \in \mathbb{N}}$ converges to $\pmb{\Theta}$ in $\mathcal{C}_{\pmb{\Theta}}$ firstly. Note that we only need to show the convergence of $\{\hat{\pmb{\mathcal{I}}}_{L} \mathrm{U}_L\}_L, \{\hat{\pmb{\mathcal{BI}}}_{L} \bar{\mathrm{W}}_L\}_L$, as the other cases are similar.  For $\tilde{t} \in [t_L^{l-1}, t_L^{l}]$, $2\leq l \leq L-1$, we have
\begin{equation}
	\begin{aligned}
		\|(\hat{\pmb{\mathcal{I}}}_{L} \mathrm{U}_L)(\tilde{t})  \!-\! \bold{U}(\tilde{t})\|
		\leq & \frac{t_L^l - \tilde{t}}{\tau} \| \mathrm{U}_L^{l-1} - \bold{U}(\tilde{t}) \| + \frac{\tilde{t} - t_L^{l - 1}}{\tau} \| \mathrm{U}_L^{l} - \bold{U}(\tilde{t}) \| \\
		\leq & \Big\| \frac{1}{2\tau} \int_{t_L^{l-2}}^{t_L^{l}}\! \bold{U}(t) dt \!-\! \bold{U}(\tilde{t}) \Big\| \!+\! \Big\| \frac{1}{2\tau} \int_{t_L^{l\!-\!1}}^{t_L^{l\!+\!1}} \!\bold{U}(t) dt \!- \!\bold{U}(\tilde{t}) \Big\| 
		\!\leq \! 2 \pmb{\omega}_{\bold{U}}(2\tau), 
	\end{aligned}
 \nonumber
\end{equation}
where $\pmb{\omega}_{\bold{U}}$ is the modulus of continuity of $\bold{U}$. {Similarly, by Eq.(\ref{recovery sequence}), we have for $\tilde{t} \in [t_L^{0}, t_L^{1}]$ and $\tilde{t} \in [t_L^{L-1}, t_L^{L}]$ that} $\|(\hat{\pmb{\mathcal{I}}}_{L} \mathrm{U}_L)(\tilde{t})  - \bold{U}(\tilde{t})\| 
		\leq   2 \pmb{\omega}_{\bold{U}}(2\tau)$.
Consequently, we obtain $\hat{\pmb{\mathcal{I}}}_{L} \mathrm{U}_L \rightarrow \bold{U}$ in $\mathcal{C}([0,1], \mathbb{R}^{n\times n})$.
In addition, by the definition of $\mathrm{W}_L$ in Eq.\eqref{recovery sequence} and a similar proof of Morrey's inequality \cite[Theorem 11.35]{leoni2009first}, there exists constants $L_{\bold{W}}$ and  $\tilde{L}_{\bold{W}}$ independent of $\tau$, such that
\begin{align}
	&  \|(\hat{\pmb{\mathcal{BI}}}_{L} \bar{\mathrm{W}}_L)(\tilde{t}, \tilde{s})  - \bold{W}(\tilde{t}, \tilde{s})\| \nonumber\\
	\leq &\frac{t^k_L\! -\!\tilde{s} }{\tau}  \frac{t^l_L\!-\!\tilde{t}}{\tau} \Big\| \!  \frac{1}{\tau ^2} \!\int_{t_L^{l\!-\!2}}^{t_L^{l\!-\!1}} \int_{t_L^{k\!-\!2}}^{t_L^{k\!-\!1}} \bold{W}(t,s)dt ds \!-\!\bold{W}(\tilde{t}, \tilde{s}) \Big\| \!  + \!\frac{t^k_L \!-\!\tilde{s} }{\tau} \frac{\tilde{t}\!-\!t^{l\!-\!1}_L}{\tau} \! \Big\| \frac{1}{\tau ^2} \!\int_{t_L^{l-1}}^{t_L^{l}} \int_{t_L^{k-2}}^{t_L^{k-1}}   \nonumber \\
	& \bold{W}(t,s)dt ds -\bold{W}(\tilde{t}, \tilde{s})  \Big\| + \frac{\tilde{s}-t^{k-1}_L}{\tau} \frac{t^l_L-\tilde{t}}{\tau} \Big\|  \frac{1}{\tau ^2} \int_{t_L^{l-2}}^{t_L^{l-1}} \int_{t_L^{k-1}}^{t_L^{k}} \bold{W}(t,s)dt ds -\bold{W}(\tilde{t}, \tilde{s}) \Big\| \nonumber \\
    & + \frac{\tilde{s}-t^{k-1}_L}{\tau}  \frac{\tilde{t}-t^{l-1}_L}{\tau} \Big\|  \frac{1}{\tau ^2} \int_{t_L^{l-1}}^{t_L^{l}} \int_{t_L^{k-1}}^{t_L^{k}} \bold{W}(t,s)dt ds -\bold{W}(\tilde{t}, \tilde{s})  \Big\| \nonumber\\
	\leq &  \frac{1}{\tau ^2} \int_{t_L^{l-2}}^{t_L^{l}} \int_{t_L^{k-2}}^{t_L^{k}} \| \bold{W}(t,s) -\bold{W}(\tilde{t}, \tilde{s}) \|dt ds \nonumber\\
	\leq &   \frac{1}{\tau ^2}  \cdot \int_{t_L^{l-2}}^{t_L^{l}} \int_{t_L^{k-2}}^{t_L^{k}} L_{\bold{W}} |(t,s) -(\tilde{t}, \tilde{s})|^{1/3} dt ds \nonumber \\
	\leq & \tilde{L}_{\bold{W}} \cdot \tau ^ {1/3}, \ \forall \tilde{t} \in [t_L^{l-1}, t_L^{l}], \tilde{s} \in [t_L^{k-1}, t_L^{k}], 2 \leq k \leq l \leq L. \nonumber
\end{align}
Similarly, by Eq.\eqref{recovery sequence} and a similar proof of Morrey's inequality \cite[Theorem 11.35]{leoni2009first} again, we have for $ \tilde{t}\! \in \! [t_L^{l-1}, t_L^{l}], \tilde{s}\! \in\! [t_L^{0}, t_L^{1}], 1 \!\leq \!l \leq\! L,$ and $ \tilde{t}\! \in \![t_L^{0}, t_L^{1}], \tilde{s}\! \in\! [t_L^{k-1}, t_L^{k}], 1 \!\leq k \!\leq \! L, $ that $\|(\hat{\pmb{\mathcal{BI}}}_{L} \bar{\mathrm{W}}_L)(\tilde{t}, \tilde{s}) \! -\! \bold{W}(\tilde{t}, \tilde{s})\| 
\!\leq \!  \tilde{L}_{\bold{W}} \tau ^ {1/3}$. 
Hence, $\hat{\pmb{\mathcal{BI}}}_{L} (\bold{flip}(\mathrm{W}_{L}))$ $ \rightarrow \bold{W}$ in $\mathcal{L}^{\infty}([0,1]\times[0,1], \mathbb{R}^{n\times n})$ owing to the symmetry of $\bold{W}$ and $\bold{flip}(\mathrm{W}_{L})$.

We next show $\limsup_{L \rightarrow \infty}\pmb{\mathcal{R}}_L(\Theta_L) \leq \pmb{\mathcal{R}}(\pmb{\Theta})$.
Note that it is enough to show the following inequalities
\begin{itemize}
    \item[(i)] $
\limsup_{L \rightarrow \infty} \pmb{\mathcal{R}}_L^{(1)}(\mathrm{T}_{j,L})  \leq \|\bold{T}_j\|^2_{\mathcal{H}^{1}((0, 1); \mathbb{R}^{n\times n})}$, $j=1,2,3$;
\item[(ii)] $
\limsup_{L \rightarrow \infty} \pmb{\mathcal{R}}_L^{(1)}(\mathrm{U}_L)  \!\leq\! \|\bold{U}\|^2_{\mathcal{H}^{1}((0, 1); \mathbb{R}^{n\times n})}$, $\limsup_{L \rightarrow \infty} \pmb{\mathcal{R}}_L^{(2)}(\mathrm{a}_L)  \!\leq \!\|\bold{a}\|^2_{\mathcal{H}^{1}((0, 1); \mathbb{R}^{n})}$;
\item[(iii)] $\limsup_{L \rightarrow \infty} \pmb{\mathcal{R}}_L^{(1)}(\mathrm{V}_L) \! \leq \!\|\bold{V}\|^2_{\mathcal{H}^{1}((0, 1); \mathbb{R}^{n\times n})}$, $\limsup_{L \rightarrow \infty} \pmb{\mathcal{R}}_L^{(2)}(\mathrm{b}_L)\!  \leq \!\|\bold{b}\|^2_{\mathcal{H}^{1}(\!(0, 1); \mathbb{R}^{n}\!)}$;
\item[(iv)] $\limsup_{L \rightarrow \infty} \pmb{\mathcal{R}}_L^{(3)}(\bold{flip}(\mathrm{W}_{L}))  \leq \|\bold{W}\|^3_{\mathcal{W}^{1,3}((0, 1)\times(0, 1); \mathbb{R}^{n\times n})}$, \\
 $ \limsup_{L \rightarrow \infty} \pmb{\mathcal{R}}_L^{(4)}(\bold{flip}(\mathrm{c}_{L})) \leq \|\bold{c}\|^3_{\mathcal{W}^{1,3}((0, 1)\times(0, 1); \mathbb{R}^{n})}$.
\end{itemize}
Since the proof of parts (i) - (iii) are analogous and the two inequalities in (iv) are analogous, we only need to show $\limsup_{L \rightarrow \infty} \pmb{\mathcal{R}}_L^{(1)}(\mathrm{U}_L)  \leq \|\bold{U}\|^2_{\mathcal{H}^{1}((0, 1); \mathbb{R}^{n\times n})}$ and  $\limsup_{L \rightarrow \infty} \pmb{\mathcal{R}}_L^{(3)}(\bold{flip}(\mathrm{W}_{L}))  \leq \|\bold{W}\|^3_{\mathcal{W}^{1,3}((0, 1)\times(0, 1); \mathbb{R}^{n\times n})}$.

We now prove $\limsup_{L \rightarrow \infty} \pmb{\mathcal{R}}_L^{(1)}(\mathrm{U}_L)  \leq \|\bold{U}\|^2_{\mathcal{H}^{1}((0, 1); \mathbb{R}^{n\times n})}.$
On one hand, 
\begin{equation}
	\begin{aligned}
		\Big | \tau \sum_{l=1}^{L}\|\mathrm{U}_{L}^{l}\|^{2} \!- \!\int_{0}^{1} \| \bold{U}(t)\|^2 dt  \Big | 
		= &	 \tau \sum_{l=1}^{L-1} \big(\|\mathrm{U}_{L}^{l}\| \!+ \! \|\bold{U}(\xi_L^l) \| \big) \cdot \Big\|\frac{1}{2\tau} \int_{t_L^{l-1}}^{t_L^{l+1}} \bold{U}(t) dt \!- \! \bold{U}(\xi_L^l)  \Big\|  \\
		& +  \tau (\|\mathrm{U}_{L}^{L}\| +  \|\bold{U}(\xi_L^L) \| ) \cdot \Big\| \frac{1}{\tau} \int_{t_L^{L-1}}^{t_L^L} \bold{U}(t) dt-  \bold{U}(\xi_L^L)  \Big\| \\
		\leq & 2 \cdot \|\bold{U}\|_C \cdot \pmb{\omega}_{\bold{U}}(2\tau) + 2 \tau\|\bold{U}\|_C \cdot \pmb{\omega}_{\bold{U}}(\tau), 
	\end{aligned}
	\nonumber
\end{equation}     
where  $\xi_L^l \in {[t_L^{l-1}, t_L^{l}}]$, the first inequality is due to the mean value theorem for integrals.
On the other hand, by H{\"o}lder's inequality and \cite[Theorem 10.55]{leoni2009first}, we get
\begin{align}
	&	\frac{1}{\tau} \sum_{l=1}^{L}\|\mathrm{U}_{L}^{l} - \mathrm{U}_{L}^{l-1}\|^{2}  \\
    {\leq} & \frac{1}{4\tau^2} \int_{t_L^1}^{t_L^2} \| \bold{U}(t) \!-\! \bold{U}(t-\tau) \|^2  dt \!+\! \frac{1}{4\tau^2} \int_{t_L^{L-1}}^{t_L^L} \| \bold{U}(t) \!-\! \bold{U}(t-\tau) \|^2  dt \nonumber \\
		& + \frac{1}{2\tau^2} \!\sum_{l=2}^{L-1} \!\Big[\!  \int_{t_L^{l}}^{t_L^{l+1}} \big\| \bold{U}(t)\! - \!\bold{U}(t\!-\!\tau) \big\|^2\!dt  \!+ \! \int_{t_L^{l-1}}^{t_L^{l}}  \big\| \bold{U}(t) \!-\! \bold{U}(t-\tau)\big\| ^2 \!dt   \!\Big] \nonumber\\
		= & \frac{1}{\tau^2} \Big[ \frac{3}{4} \int_{t_L^1}^{t_L^2} \| \bold{U}(t) \!-\! \bold{U}(t-\tau) \|^2  dt \!+ \!\sum_{l=2}^{L-2}\int_{t_L^{l}}^{t_L^{l+1}}  \left\| \bold{U}(t)\! - \!\bold{U}(t-\tau)\right\| ^2 \! dt \nonumber\\
		& \quad  + \frac{3}{4} \int_{t_L^{L-1}}^{t_L^L} \| \bold{U}(t) - \bold{U}(t-\tau) \|^2  dt   \Big]  \nonumber \\
		\leq & {L^2} \int_{1/L}^{1} \| \bold{U}(t) - \bold{U}(t-\frac{1}{L}) \|^2  dt \leq  \int_{0}^{1} \| \pmb{\mathcal{D}}_t(\bold{U})(t)\|^2  dt. \nonumber
\end{align}
Combining the above two inequalities, we have
	\begin{align}
		\limsup_{L \rightarrow \infty} \pmb{\mathcal{R}}_L^{(1)}(\mathrm{U}_L) \leq &\lim_{L \rightarrow \infty} \tau \sum_{l=1}^{L}\|\mathrm{U}_{L}^{l}\|^{2} + \limsup_{L \rightarrow \infty} \frac{1}{\tau} \sum_{l=1}^{L}\|\mathrm{U}_{L}^{l} - \mathrm{U}_{L}^{l-1}\|^{2} \nonumber \\ 
		\leq & \int_{0}^{1} \| \bold{U}(t)\|^2 dt + \int_{0}^{1} \| \pmb{\mathcal{D}}_t(\bold{U})(t)\|^2  dt = \|\bold{U}\|^2_{\mathcal{H}^1((0, 1); \mathbb{R}^{n\times n})}. \nonumber
	\end{align}
We next show $\limsup_{L \rightarrow \infty} \pmb{\mathcal{R}}_L^{(3)}(\bold{flip}(\mathrm{W}_{L}))  \leq \|\bold{W}\|^3_{\mathcal{W}^{1,3}((0, 1)\times(0, 1); \mathbb{R}^{n\times n})}$. A direct calculation yields
\begin{align*}
	\pmb{\mathcal{R}}_L^{(3)}(\bar{\mathrm{W}}_L)\!=  & \tau^2 \!\sum_{l=1}^{L} \sum_{k=1}^{L}\|\bar{\mathrm{W}}_{L}^{l,k}\|^{3} \!+\!   \tau^{-1} \! \Big(\!\sum_{l=1}^{L} \sum_{k=1}^{L} \! \|\!\bar{\mathrm{W}}_{L}^{l,k}\! -\! \bar{\mathrm{W}}_{L}^{l-1,k}\|^{3} \! + \! \sum_{l=1}^{L} \!\sum_{k=1}^{L} \! \|\!\bar{\mathrm{W}}_{L}^{l,k}\! -\! \bar{\mathrm{W}}_{L}^{l,k-1}\!\|^{3} \!\Big) \\
	 = &\tau^2 \!\sum_{l=1}^{L}\! \sum_{k=1}^{L} \! \Big \|\frac{1}{\tau ^2} \!\int_{t_L^{l-1}}^{t_L^{l}} \int_{t_L^{k-1}}^{t_L^{k}} \!\bold{W}(t,s)dt ds\Big\|^3 \!+ \!\tau^{-1} \Big(  \sum_{l=2}^{L} \!\sum_{k=1}^{L} \! \Big\|\! \frac{1}{\tau ^2}\int_{t_L^{l-1}}^{t_L^{l}} \int_{t_L^{k\!-\!1}}^{t_L^{k}} \!\bold{W}(t,s) \\
	 &  -\! \bold{W}(t-\tau,s) dt ds \Big\|^{3} \! +  \!\sum_{l=1}^{L} \sum_{k=2}^{L} \! \Big\|\frac{1}{\tau ^2}\int_{t_L^{l-1}}^{t_L^{l}} \!\int_{t_L^{k-1}}^{t_L^{k}} \bold{W}(t,s)\! -\! \bold{W}(t,s-\tau) dt ds\Big\|^{3} \Big) \\
	\leq & \int_{0}^{1} \int_{0}^{1} \|\bold{W}(t,s)\|^3dt ds + {L^3} \int_{1/L}^{1} \int_{0}^{1} \|\bold{W}(t,s) - \bold{W}(t-1/L,s)\|^3dt ds \\
	& + {L^3} \int_{0}^{1}\int_{1/L}^{1} \|\bold{W}(t,s) - \bold{W}(t,s-1/L)\|^3 dt ds \\
	\leq & \|\bold{W}\|^3_{\mathcal{W}^{1,3}((0, 1) \times (0, 1); \mathbb{R}^{n\times n})},
\end{align*}
where the first inequality is owing to the H{\"o}lder's inequality and the second inequality is due to \cite[Theorem 10.55]{leoni2009first}.
Taking $L \rightarrow \infty$, we have $ \limsup_{L \rightarrow \infty} \pmb{\mathcal{R}}_L^{(3)}(\bar{\mathrm{W}}_L) \leq \|\bold{W}\|^3_{\mathcal{W}^{1,3}((0, 1) \times (0, 1); \mathbb{R}^{n\times n})}$.

The condition (\ref{convergence of parameter}) is satisfied with the given $\{\Theta_L\}_L$ due to the convergence of $\{\pmb{\Theta}_L\}_{L \in \mathbb{N}}$ and the proof of Lemma~\ref{lemma: liminf condition}. Therefore, we get 
\begin{equation}
	\frac{1}{M} \sum\limits_{m=1}^{M} \pmb{\ell} (\mathrm{x}^L(\mathrm{d}_m; \Theta_L), \mathrm{g}_m) \rightarrow \frac{1}{M} \sum\limits_{m=1}^{M} \pmb{\ell} (\bold{x}(1; \mathrm{d}_m;\pmb{\Theta}) , \mathrm{g}_m)  \text{ as } L \rightarrow \infty,
	\label{equation: limsup-data}
\end{equation}
by the continuity of $\pmb{\ell}$ (assumption (${A}_3$)) and Proposition~\ref{forward convergence}. Thus, combining $\limsup_{L \rightarrow \infty}\pmb{\mathcal{R}}_L(\Theta_L) \leq \pmb{\mathcal{R}}(\pmb{\Theta})$ and Eq.(\ref{equation: limsup-data}) completes the proof.
\end{proof}

Now, we are ready to give a proof for Theorem~\ref{theorem: main results}.
\begin{proof}
    (\textbf{Proof of Theorem~\ref{theorem: main results}})
    By Proposition~\ref{proposition: existence}, the optimal solutions of ($\mathcal{P}$) and $(\mathcal{P}_{L})$ exist in $\Omega_{\pmb{\Theta}}$ and $\Omega_{\Theta;{L}}$ respectively. Consequently, we obtain the existence of minimizers in $\mathcal{C}_{\pmb{\Theta}}$ of $\tilde{\pmb{\mathfrak{L}}}$ and $\tilde{\pmb{\mathfrak{L}}}_{\mathcal{S}; L}$, $L \in \mathbb{N}$, owing to Lemma~\ref{lemma: solution tilde_L and L_L}.
	Moreover, $\tilde{\pmb{\mathfrak{L}}}_{\mathcal{S}; L}$ $\Gamma$-converges to $\tilde{\pmb{\mathfrak{L}}}$ in $\mathcal{C}_{\pmb{\Theta}}$ due to Lemma~\ref{lemma: liminf condition} and Lemma~\ref{lemma: limsup}.
	Using the definition of $\tilde{\pmb{\mathfrak{L}}}_{\mathcal{S}; L}$,
	if  $\pmb{\Theta}_L^{*} \in \arg \min_{\pmb{\Theta} \in \mathcal{C}_{\pmb{\Theta}}}\tilde{\pmb{\mathfrak{L}}}_{\mathcal{S}; L}(\pmb{\Theta})$, 
 then there exists some $\Theta_L^* = ({\mathrm{T}}_L^*,{\mathrm{U}}_L^*, {\mathrm{a}}_L^*, {\mathrm{V}}_L^*, {\mathrm{b}}_L^*, {\mathrm{W}}_L^*, {\mathrm{c}}_L^*)  \in \Omega_{\Theta;{L}}$ such that $\pmb{\Theta}_L^{*} = \hat{\pmb{\mathcal{I}}}_{L} \Theta_L^{*}$ 
 for all $L \in \mathbb{N}$. It then follows from Lemma~\ref{lemma: relative compactness} that 
  there exists a subsequence of $\{\pmb{\Theta}_L^{*}\}_L$ converging to a $\tilde{\pmb{\Theta}} \in \Omega_{\pmb{\Theta}}$ in $\mathcal{C}_{\pmb{\Theta}}$. 
   Using the property of $\Gamma$-convergence \cite[Theorem~3.2]{thorpe2018deep}, we obtain
	$$\min_{\pmb{\Theta} \in \mathcal{C}_{\pmb{\Theta}}} \tilde{\pmb{\mathfrak{L}}}_{\mathcal{S}; L}(\pmb{\Theta}) \rightarrow \min_{\pmb{\Theta} \in \mathcal{C}_{\pmb{\Theta}}}\tilde{\pmb{\mathfrak{L}}}(\pmb{\Theta}), \ {\rm as } \ L \rightarrow \infty.$$
In addition, for any subsequence $\{\pmb{\Theta}_{L_k}^{*}\}_{L_k \in \mathbb{N}}$ converging to some $\pmb{\Theta}^{*}$ in $\mathcal{C}_{\pmb{\Theta}}$, we have $\pmb{\Theta}^{*} \in \arg \min_{\pmb{\Theta} \in \mathcal{C}_{\pmb{\Theta}} }\tilde{\pmb{\mathfrak{L}}}(\pmb{\Theta})$.
	
Therefore, for any sequence $\{{\Theta}_L^*\}_L$ with $\Theta_L^*$ being the minimiser of ($\mathcal{P}_L$), we derive 
 $$
 \begin{aligned}
  & \min _{\Theta_L \in \Omega_{{\Theta};L}}{\pmb{\mathfrak{L}}}_{\mathcal{S}; L}(\Theta_L) = {\pmb{\mathfrak{L}}}_{\mathcal{S}; L}(\Theta_L^*) = \tilde{\pmb{\mathfrak{L}}}_{\mathcal{S}; L}(\pmb{\Theta}_L^{*}) = \min_{\pmb{\Theta} \in \mathcal{C}_{\pmb{\Theta}}}\tilde{\pmb{\mathfrak{L}}}_{\mathcal{S}; L}(\pmb{\Theta}) \\
   \rightarrow & \min_{\pmb{\Theta} \in \mathcal{C}_{\pmb{\Theta}}} \tilde{\pmb{\mathfrak{L}}}(\pmb{\Theta}) = \tilde{\pmb{\mathfrak{L}}}(\pmb{\Theta}^{*}) = \pmb{\mathfrak{L}}(\pmb{\Theta}^{*}) = \min_{\pmb{\Theta} \in \Omega_{\pmb{\Theta}}} \pmb{\mathfrak{L}}(\pmb{\Theta}), \text{ as } L \rightarrow \infty.
 \end{aligned}
 $$
Besides, $\hat{\pmb{\mathcal{I}}}_{L} \Theta_L^{*} \in \arg \min_{\pmb{\Theta} \in \mathcal{C}_{\pmb{\Theta}}}\tilde{\pmb{\mathfrak{L}}}_{\mathcal{S}; L}(\pmb{\Theta})$
	by Lemma~\ref{lemma: solution tilde_L and L_L}. 
	It then follows that  
	any accumulation point of $\{ \hat{\pmb{\mathcal{I}}}_{L} \Theta_L^{*}\}_{L \in \mathbb{N}}$ minimizes $\tilde{\pmb{\mathfrak{L}}}$, and thus is an optimal solution of ($\mathcal{P}$).
\end{proof}

\section{Experiments}
\label{Sec:5}
In this section, we conduct experiments to test the behavior of the training loss of the DNL framework to verify our convergence results in some sense. We investigate how it changes as the layer number $L$ increases on image classification tasks on the SVHN and CIFAR-10 datasets.
The SVHN dataset contains approximately 73,257 training images and 26,032 test images, while CIFAR-10 consists of 50,000 training images and 10,000 test images. Each image has 3 color channels and a resolution of $32 \times 32$ pixels, yielding 3,072 features per sample. We use the standard dataset splits without additional preprocessing.

We implement the standard DenseNet as an instance of the DNL framework to examine the deep-layer limit behavior established in our theoretical result. Each network consists of four DNL blocks, each of which has $L$ layers with a growth rate of 6.
For the loss function in (\ref{discrete-time-control problem P_L}), we take $\pmb{\ell}$ as the cross entropy function, i.e.,
$$
\pmb{\ell}(\hat{\mathrm{g}},\mathrm{g}) = - 10^6 \times \sum_{i=1}^{n}\mathrm{g}[i] \log(\hat{\mathrm{g}}[i]), \ \hat{\mathrm{g}},\mathrm{g} \in \mathbb{R}^n.
$$
All models are trained for 50 epochs on the SVHN dataset and 150 epochs on CIFAR10 using SGD with an initial learning rate of 0.01 and momentum of 0.9. 
To reduce the influence of randomness, we repeat the training process five times for each hyperparameter setting and record the average training loss of the five trained models during testing.
Since our theoretical results concern the training problem rather than generalization, we only report training losses.
All experiments are implemented in PyTorch on an NVIDIA GeForce RTX 3090 GPU. 

We test the DNL block with layer numbers $L=6,8,10,12,14,16$, yielding a step size $\tau_L = 1/L$.
Figure~\ref{figure: classification-deep-limit} shows training losses versus epoch number for different $L$. On both SVHN and CIFAR10, the losses converge as $L$ increases, i.e., larger $L$ yields smaller training loss, and the improvement diminishes with further increases of $L$. This observation is consistent with the convergence result established in Theorem~\ref{theorem: main results}.

\begin{figure}[htbp]
	\centering
         \subfigure[SVHN dataset]{
		\includegraphics[scale=0.22]{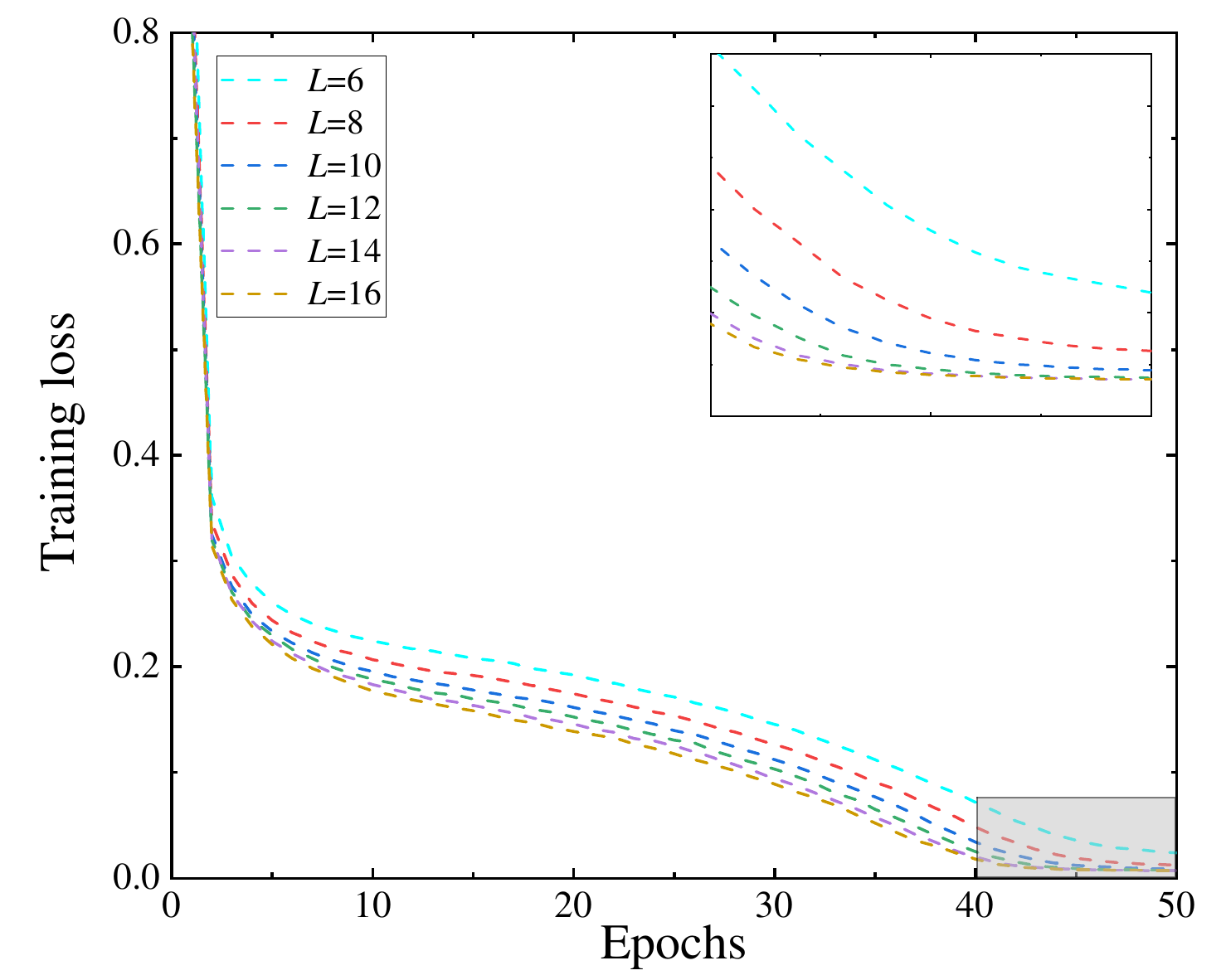} \label{1} 
	}
        \subfigure[CIFAR10 dataset]{
		\includegraphics[scale=0.22]{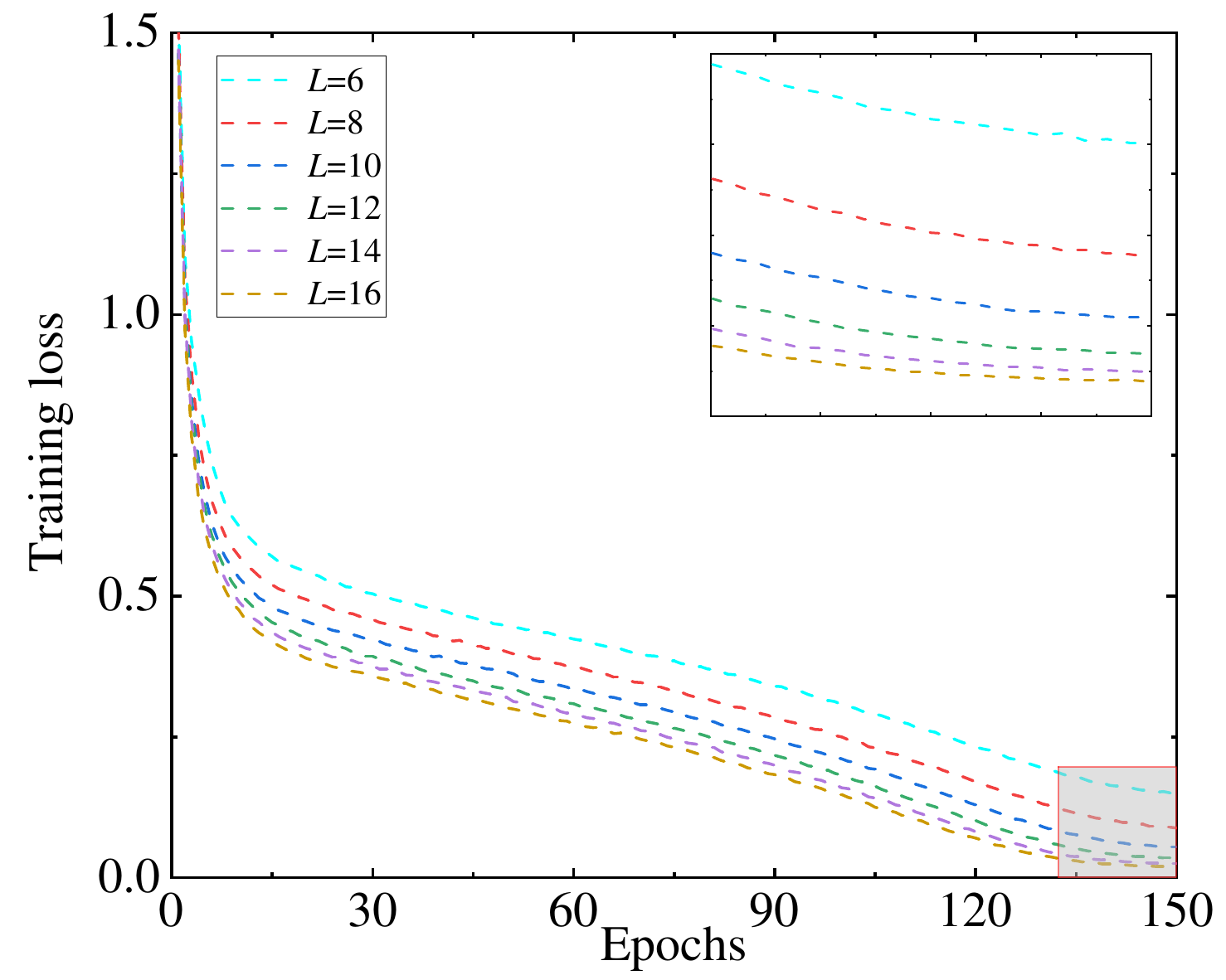} \label{2} 
	}
	\caption{The plot of training loss v.s. epoch number for the DenseNet with different layer numbers $L$ on the SVHN and CIFAR10 datasets. 
    The training losses decrease with increasing $L$, closely matching the theoretical convergent result.}
	\label{figure: classification-deep-limit}
\end{figure}

\section{Conclusion}
\label{Sec:6}
DNNs with dense layer connectivities form an important class of architectures in modern deep learning.
In this work, we provided a dynamical system modeling and convergence analysis for such DNN architectures. Our study was presented within a general DNL framework, containing both standard dense layer connections and general (local/non-local) feature transformations within layers.
In the deep-layer limit, we obtained a continuous-time formulation in the form of nonlinear integral equations and studied the associated learning problems from an optimal control perspective.
By employing a piecewise linear extension technique and $\Gamma$-convergence tool, we established convergence properties between their learning problems, including the convergence of optimal values and the subsequence convergence of minimizers.
Our theoretical result is applicable to a class of densely connected DNNs with various local/non-local feature transformations within layers.
These findings provide a theoretical foundation for understanding densely connected networks and suggest directions for future research, including the design of training algorithms and the study of loss landscapes based on continuous-time integral equations.

\backmatter

\bmhead{Supplementary information}
The supplementary material contains the complete proofs of Remark~\ref{remark: remark for assumptions}, Lemma~\ref{lemma: dis bound}, Proposition~\ref{proposition: forward-existence}, and Corollary~\ref{corollary: convergent rate} presented in the main text. Readers interested in the detailed derivations are referred to this file.

\bmhead{Acknowledgements}
This work was partially supported by the National Natural Science Foundation of China (grants 12271273, 11871035) and the Key Program (21JCZDJC00220) of the Natural Science Foundation of Tianjin, China.

\begin{appendices}

\section{$\Gamma$-convergence}
\label{appendix: definitions and inequalities}
\begin{definition}
	{\rm \cite{braides2002gamma}} Let $(\mathbb{U}, d_{\mathbb{U}})$ be a metric space, and $\pmb{\mathcal{E}}_{L}:\mathbb{U} \rightarrow [0,\infty]$ be a sequence of functionals. Then the sequence $\left\{\pmb{\mathcal{E}}_{L}\right\}_{L \in \mathbb{N}} \Gamma$-converges with respect to metric $d_{\mathbb{U}}$ to the functional $\mathcal{E}: \mathbb{U} \rightarrow[0, \infty]$ as $L \rightarrow \infty$ if the following inequalities hold:\\
	1.{ Liminf inequality}: For every ${\rm u} \in \mathbb{U}$ and every sequence $\left\{{\rm u}_L\right\}_{L \in \mathbb{N}}$ converging to ${\rm u}$,
	$$
	\liminf _{L \rightarrow \infty} \pmb{\mathcal{E}}_{L}\left({\rm u}_L\right) \geq \mathcal{E}({\rm u});
	$$
	2. {Limsup inequality}: For every ${\rm u} \in \mathbb{U}$ there exists a sequence $\left\{{\rm u}_L\right\}_{L \in \mathbb{N}}$ converging to ${\rm u}$ satisfying
	$$
	\limsup _{L \rightarrow \infty} \pmb{\mathcal{E}}_{L}\left({\rm u}_L\right) \leq \mathcal{E}({\rm u}) .
	$$
	\label{definition: gamma convergence}
\end{definition}
We say that $\mathcal{E}$ is the $\Gamma$-limit of the sequence of functionals $\left\{\pmb{\mathcal{E}}_{L}\right\}_{L \in \mathbb{N}}$ (with respect to the metric $d_{\mathbb{U}}$). 
A fundamental property of $\Gamma$-convergence is given in \cite[Theorem~3.2]{thorpe2018deep}.
\begin{theorem}
	{\rm \cite[Theorem~3.2]{thorpe2018deep}}
	Let $(\mathbb{U}, d_{\mathbb{U}})$ be a metric space and let $\pmb{\mathcal{E}}_{L}:\mathbb{U} \rightarrow [0,\infty]$ be a sequence of functionals. Let ${\rm u}_L$ be a sequence of almost minimisers for $\pmb{\mathcal{E}}_{L}$, i.e.
	$$\pmb{\mathcal{E}}_{L}\left({\rm u}_L\right) \leq \max \left\{\inf _{{\rm u} \in \mathbb{U}} \pmb{\mathcal{E}}_{L}\left({\rm u}_L\right)+\varepsilon_L,-\frac{1}{\varepsilon_L}\right\}$$ 
	for some $\varepsilon_L \rightarrow 0^{+}$. Assume that $\pmb{\mathcal{E}}$ is the $\Gamma$-limit of the sequence of functionals $\left\{\pmb{\mathcal{E}}_{L}\right\}_{L \in \mathbb{N}}$ and $\left\{{\rm u}_L\right\}_{L=1}^{\infty}$ are relatively compact. Then,
	$$	\inf _{{\rm u} \in \mathbb{U}} \pmb{\mathcal{E}}_{L}({\rm u}) \rightarrow \min _{{\rm u} \in \mathbb{U}} \pmb{\mathcal{E}}({\rm u}),$$
	where the minimum of $\pmb{\mathcal{E}}$ exists. Moreover, if a convergent subsequence ${\rm u}_{L_k} \rightarrow {\rm u}^*$, then ${\rm u}^*$ minimizes $\pmb{\mathcal{E}}$.
	\label{theorem: gamma-con basic thm}
\end{theorem}




\end{appendices}


\bibliography{sn-bibliography}


\begin{thebibliography}{40}
\ifx \bisbn   \undefined \def \bisbn  #1{ISBN #1}\fi
\ifx \binits  \undefined \def \binits#1{#1}\fi
\ifx \bauthor  \undefined \def \bauthor#1{#1}\fi
\ifx \batitle  \undefined \def \batitle#1{#1}\fi
\ifx \bjtitle  \undefined \def \bjtitle#1{#1}\fi
\ifx \bvolume  \undefined \def \bvolume#1{\textbf{#1}}\fi
\ifx \byear  \undefined \def \byear#1{#1}\fi
\ifx \bissue  \undefined \def \bissue#1{#1}\fi
\ifx \bfpage  \undefined \def \bfpage#1{#1}\fi
\ifx \blpage  \undefined \def \blpage #1{#1}\fi
\ifx \burl  \undefined \def \burl#1{\textsf{#1}}\fi
\ifx \doiurl  \undefined \def \doiurl#1{\url{https://doi.org/#1}}\fi
\ifx \betal  \undefined \def \betal{\textit{et al.}}\fi
\ifx \binstitute  \undefined \def \binstitute#1{#1}\fi
\ifx \binstitutionaled  \undefined \def \binstitutionaled#1{#1}\fi
\ifx \bctitle  \undefined \def \bctitle#1{#1}\fi
\ifx \beditor  \undefined \def \beditor#1{#1}\fi
\ifx \bpublisher  \undefined \def \bpublisher#1{#1}\fi
\ifx \bbtitle  \undefined \def \bbtitle#1{#1}\fi
\ifx \bedition  \undefined \def \bedition#1{#1}\fi
\ifx \bseriesno  \undefined \def \bseriesno#1{#1}\fi
\ifx \blocation  \undefined \def \blocation#1{#1}\fi
\ifx \bsertitle  \undefined \def \bsertitle#1{#1}\fi
\ifx \bsnm \undefined \def \bsnm#1{#1}\fi
\ifx \bsuffix \undefined \def \bsuffix#1{#1}\fi
\ifx \bparticle \undefined \def \bparticle#1{#1}\fi
\ifx \barticle \undefined \def \barticle#1{#1}\fi
\bibcommenthead
\ifx \bconfdate \undefined \def \bconfdate #1{#1}\fi
\ifx \botherref \undefined \def \botherref #1{#1}\fi
\ifx \url \undefined \def \url#1{\textsf{#1}}\fi
\ifx \bchapter \undefined \def \bchapter#1{#1}\fi
\ifx \bbook \undefined \def \bbook#1{#1}\fi
\ifx \bcomment \undefined \def \bcomment#1{#1}\fi
\ifx \oauthor \undefined \def \oauthor#1{#1}\fi
\ifx \citeauthoryear \undefined \def \citeauthoryear#1{#1}\fi
\ifx \endbibitem  \undefined \def \endbibitem {}\fi
\ifx \bconflocation  \undefined \def \bconflocation#1{#1}\fi
\ifx \arxivurl  \undefined \def \arxivurl#1{\textsf{#1}}\fi
\csname PreBibitemsHook\endcsname

\bibitem[\protect\citeauthoryear{Goodfellow et~al.}{2016}]{goodfellow2016deep}
\begin{bbook}
\bauthor{\bsnm{Goodfellow}, \binits{I.}},
\bauthor{\bsnm{Bengio}, \binits{Y.}},
\bauthor{\bsnm{Courville}, \binits{A.}}:
\bbtitle{Deep Learning}.
\bpublisher{MIT press},
\blocation{Cambridge, MA}
(\byear{2016})
\end{bbook}
\endbibitem

\bibitem[\protect\citeauthoryear{He et~al.}{2016}]{he2016deep}
\begin{bchapter}
\bauthor{\bsnm{He}, \binits{K.}},
\bauthor{\bsnm{Zhang}, \binits{X.}},
\bauthor{\bsnm{Ren}, \binits{S.}},
\bauthor{\bsnm{Sun}, \binits{J.}}:
\bctitle{Deep residual learning for image recognition}.
In: \bbtitle{Proceedings of the IEEE Conference on Computer Vision and Pattern
  Recognition},
pp. \bfpage{770}--\blpage{778}
(\byear{2016})
\end{bchapter}
\endbibitem

\bibitem[\protect\citeauthoryear{Huang et~al.}{2016}]{Huang2016DenselyCC}
\begin{botherref}
\oauthor{\bsnm{Huang}, \binits{G.}},
\oauthor{\bsnm{Liu}, \binits{Z.}},
\oauthor{\bsnm{Weinberger}, \binits{K.Q.}}:
Densely connected convolutional networks.
2017 IEEE Conference on Computer Vision and Pattern Recognition,
2261--2269
(2016)
\end{botherref}
\endbibitem

\bibitem[\protect\citeauthoryear{Ronneberger et~al.}{2015}]{ronneberger2015u}
\begin{bchapter}
\bauthor{\bsnm{Ronneberger}, \binits{O.}},
\bauthor{\bsnm{Fischer}, \binits{P.}},
\bauthor{\bsnm{Brox}, \binits{T.}}:
\bctitle{U-net: Convolutional networks for biomedical image segmentation}.
In: \bbtitle{Medical Image Computing and Computer-Assisted Intervention--MICCAI
  2015: 18th International Conference, Munich, Germany, October 5-9, 2015,
  Proceedings, Part III 18},
pp. \bfpage{234}--\blpage{241}
(\byear{2015}).
\bcomment{Springer}
\end{bchapter}
\endbibitem

\bibitem[\protect\citeauthoryear{Zhang et~al.}{2020}]{zhang2020dense}
\begin{bchapter}
\bauthor{\bsnm{Zhang}, \binits{H.}},
\bauthor{\bsnm{Qin}, \binits{K.}},
\bauthor{\bsnm{Zhang}, \binits{Y.}},
\bauthor{\bsnm{Li}, \binits{Z.}},
\bauthor{\bsnm{Xu}, \binits{K.}}:
\bctitle{Dense attention convolutional network for image classification}.
In: \bbtitle{Journal of Physics: Conference Series},
vol. \bseriesno{1651},
p. \bfpage{012184}
(\byear{2020}).
\bcomment{IOP Publishing}
\end{bchapter}
\endbibitem

\bibitem[\protect\citeauthoryear{Yao et~al.}{2022}]{yao2022dense}
\begin{barticle}
\bauthor{\bsnm{Yao}, \binits{C.}},
\bauthor{\bsnm{Jin}, \binits{S.}},
\bauthor{\bsnm{Liu}, \binits{M.}},
\bauthor{\bsnm{Ban}, \binits{X.}}:
\batitle{Dense residual transformer for image denoising}.
\bjtitle{Electronics}
\bvolume{11}(\bissue{3}),
\bfpage{418}
(\byear{2022})
\end{barticle}
\endbibitem

\bibitem[\protect\citeauthoryear{Ma et~al.}{2023}]{ma2023denseformer}
\begin{barticle}
\bauthor{\bsnm{Ma}, \binits{H.}},
\bauthor{\bsnm{Li}, \binits{X.}},
\bauthor{\bsnm{Yuan}, \binits{X.}},
\bauthor{\bsnm{Zhao}, \binits{C.}}:
\batitle{Denseformer: A dense transformer framework for person
  re-identification}.
\bjtitle{IET Computer Vision}
\bvolume{17}(\bissue{5}),
\bfpage{527}--\blpage{536}
(\byear{2023})
\end{barticle}
\endbibitem

\bibitem[\protect\citeauthoryear{Diakogiannis
  et~al.}{2020}]{diakogiannis2020resunet}
\begin{barticle}
\bauthor{\bsnm{Diakogiannis}, \binits{F.I.}},
\bauthor{\bsnm{Waldner}, \binits{F.}},
\bauthor{\bsnm{Caccetta}, \binits{P.}},
\bauthor{\bsnm{Wu}, \binits{C.}}:
\batitle{Resunet-a: A deep learning framework for semantic segmentation of
  remotely sensed data}.
\bjtitle{ISPRS Journal of Photogrammetry and Remote Sensing}
\bvolume{162},
\bfpage{94}--\blpage{114}
(\byear{2020})
\end{barticle}
\endbibitem

\bibitem[\protect\citeauthoryear{Cao et~al.}{2020}]{cao2020denseunet}
\begin{barticle}
\bauthor{\bsnm{Cao}, \binits{Y.}},
\bauthor{\bsnm{Liu}, \binits{S.}},
\bauthor{\bsnm{Peng}, \binits{Y.}},
\bauthor{\bsnm{Li}, \binits{J.}}:
\batitle{Denseunet: densely connected unet for electron microscopy image
  segmentation}.
\bjtitle{IET Image Processing}
\bvolume{14}(\bissue{12}),
\bfpage{2682}--\blpage{2689}
(\byear{2020})
\end{barticle}
\endbibitem

\bibitem[\protect\citeauthoryear{Tai et~al.}{2024}]{tai2024pottsmgnet}
\begin{barticle}
\bauthor{\bsnm{Tai}, \binits{X.-C.}},
\bauthor{\bsnm{Liu}, \binits{H.}},
\bauthor{\bsnm{Chan}, \binits{R.}}:
\batitle{Pottsmgnet: A mathematical explanation of encoder-decoder based neural
  networks}.
\bjtitle{SIAM Journal on Imaging Sciences}
\bvolume{17}(\bissue{1}),
\bfpage{540}--\blpage{594}
(\byear{2024})
\end{barticle}
\endbibitem

\bibitem[\protect\citeauthoryear{Vaswani et~al.}{2017}]{vaswani2017attention}
\begin{botherref}
\oauthor{\bsnm{Vaswani}, \binits{A.}},
\oauthor{\bsnm{Shazeer}, \binits{N.}},
\oauthor{\bsnm{Parmar}, \binits{N.}},
\oauthor{\bsnm{Uszkoreit}, \binits{J.}},
\oauthor{\bsnm{Jones}, \binits{L.}},
\oauthor{\bsnm{Gomez}, \binits{A.N.}},
\oauthor{\bsnm{Kaiser}, \binits{{\L}.}},
\oauthor{\bsnm{Polosukhin}, \binits{I.}}:
Attention is all you need.
Advances in Neural Information Processing Systems
\textbf{30}
(2017)
\end{botherref}
\endbibitem

\bibitem[\protect\citeauthoryear{Wang et~al.}{2017}]{Wang2017NonlocalNN}
\begin{botherref}
\oauthor{\bsnm{Wang}, \binits{X.}},
\oauthor{\bsnm{Girshick}, \binits{R.B.}},
\oauthor{\bsnm{Gupta}, \binits{A.K.}},
\oauthor{\bsnm{He}, \binits{K.}}:
Non-local neural networks.
Proceedings of the IEEE Conference on Computer Vision and Pattern Recognition,
7794--7803
(2017)
\end{botherref}
\endbibitem

\bibitem[\protect\citeauthoryear{Jia et~al.}{2020}]{jia2020nonlocal}
\begin{barticle}
\bauthor{\bsnm{Jia}, \binits{F.}},
\bauthor{\bsnm{Tai}, \binits{X.-C.}},
\bauthor{\bsnm{Liu}, \binits{J.}}:
\batitle{Nonlocal regularized cnn for image segmentation}.
\bjtitle{Inverse Problems \& Imaging}
\bvolume{14}(\bissue{5}),
\bfpage{891}--\blpage{911}
(\byear{2020})
\end{barticle}
\endbibitem

\bibitem[\protect\citeauthoryear{Meng et~al.}{2024}]{meng2024learnable}
\begin{barticle}
\bauthor{\bsnm{Meng}, \binits{J.}},
\bauthor{\bsnm{Wang}, \binits{F.}},
\bauthor{\bsnm{Liu}, \binits{J.}}:
\batitle{Learnable nonlocal self-similarity of deep features for image
  denoising}.
\bjtitle{SIAM Journal on Imaging Sciences}
\bvolume{17}(\bissue{1}),
\bfpage{441}--\blpage{475}
(\byear{2024})
\end{barticle}
\endbibitem

\bibitem[\protect\citeauthoryear{Liu et~al.}{2022}]{liu2022deep}
\begin{barticle}
\bauthor{\bsnm{Liu}, \binits{J.}},
\bauthor{\bsnm{Wang}, \binits{X.}},
\bauthor{\bsnm{Tai}, \binits{X.-C.}}:
\batitle{Deep convolutional neural networks with spatial regularization, volume
  and star-shape priors for image segmentation}.
\bjtitle{Journal of Mathematical Imaging and Vision}
\bvolume{64}(\bissue{6}),
\bfpage{625}--\blpage{645}
(\byear{2022})
\end{barticle}
\endbibitem

\bibitem[\protect\citeauthoryear{Gao and Pavel}{2017}]{gao2017properties}
\begin{botherref}
\oauthor{\bsnm{Gao}, \binits{B.}},
\oauthor{\bsnm{Pavel}, \binits{L.}}:
On the properties of the softmax function with application in game theory and
  reinforcement learning.
arXiv preprint arXiv:1704.00805
(2017)
\end{botherref}
\endbibitem

\bibitem[\protect\citeauthoryear{E}{2017}]{weinan2017proposal}
\begin{barticle}
\bauthor{\bsnm{E}, \binits{W.}}:
\batitle{A proposal on machine learning via dynamical systems}.
\bjtitle{Communications in Mathematics and Statistics}
\bvolume{1}(\bissue{5}),
\bfpage{1}--\blpage{11}
(\byear{2017})
\end{barticle}
\endbibitem

\bibitem[\protect\citeauthoryear{Haber and Ruthotto}{2017}]{haber2017stable}
\begin{barticle}
\bauthor{\bsnm{Haber}, \binits{E.}},
\bauthor{\bsnm{Ruthotto}, \binits{L.}}:
\batitle{Stable architectures for deep neural networks}.
\bjtitle{Inverse Problems}
\bvolume{34}(\bissue{1}),
\bfpage{014004}
(\byear{2017})
\end{barticle}
\endbibitem

\bibitem[\protect\citeauthoryear{Lu et~al.}{2018}]{Lu18Beyond}
\begin{bchapter}
\bauthor{\bsnm{Lu}, \binits{Y.}},
\bauthor{\bsnm{Zhong}, \binits{A.}},
\bauthor{\bsnm{Li}, \binits{Q.}},
\bauthor{\bsnm{Dong}, \binits{B.}}:
\bctitle{Beyond finite layer neural networks: Bridging deep architectures and
  numerical differential equations}.
In: \bbtitle{Proceedings of the 35th International Conference on Machine
  Learning}.
\bsertitle{Proceedings of Machine Learning Research},
vol. \bseriesno{80},
pp. \bfpage{3282}--\blpage{3291}.
\bpublisher{PMLR},
\blocation{Stockholm, Sweden}
(\byear{2018})
\end{bchapter}
\endbibitem

\bibitem[\protect\citeauthoryear{Gomez et~al.}{2017}]{gomez2017reversible}
\begin{botherref}
\oauthor{\bsnm{Gomez}, \binits{A.N.}},
\oauthor{\bsnm{Ren}, \binits{M.}},
\oauthor{\bsnm{Urtasun}, \binits{R.}},
\oauthor{\bsnm{Grosse}, \binits{R.B.}}:
The reversible residual network: Backpropagation without storing activations.
Advances in Neural Information Processing Systems
\textbf{30}
(2017)
\end{botherref}
\endbibitem

\bibitem[\protect\citeauthoryear{Zhang and Schaeffer}{2020}]{zhang2020forward}
\begin{barticle}
\bauthor{\bsnm{Zhang}, \binits{L.}},
\bauthor{\bsnm{Schaeffer}, \binits{H.}}:
\batitle{Forward stability of resnet and its variants}.
\bjtitle{Journal of Mathematical Imaging and Vision}
\bvolume{62}(\bissue{3}),
\bfpage{328}--\blpage{351}
(\byear{2020})
\end{barticle}
\endbibitem

\bibitem[\protect\citeauthoryear{Thorpe and van Gennip}{2023}]{thorpe2018deep}
\begin{barticle}
\bauthor{\bsnm{Thorpe}, \binits{M.}},
\bauthor{\bsnm{Gennip}, \binits{Y.}}:
\batitle{Deep limits of residual neural networks}.
\bjtitle{Research in the Mathematical Sciences}
\bvolume{10}(\bissue{1}),
\bfpage{6}
(\byear{2023})
\end{barticle}
\endbibitem

\bibitem[\protect\citeauthoryear{Chen et~al.}{2018}]{chen2018neural}
\begin{bchapter}
\bauthor{\bsnm{Chen}, \binits{R.T.Q.}},
\bauthor{\bsnm{Rubanova}, \binits{Y.}},
\bauthor{\bsnm{Bettencourt}, \binits{J.}},
\bauthor{\bsnm{Duvenaud}, \binits{D.K.}}:
\bctitle{Neural ordinary differential equations}.
In: \bbtitle{Advances in Neural Information Processing Systems},
vol. \bseriesno{31},
pp. \bfpage{6571}--\blpage{6583}.
\bpublisher{Curran Associates, Inc.},
\blocation{Red Hook, NY}
(\byear{2018})
\end{bchapter}
\endbibitem

\bibitem[\protect\citeauthoryear{Lu et~al.}{2019}]{lu2019understanding}
\begin{botherref}
\oauthor{\bsnm{Lu}, \binits{Y.}},
\oauthor{\bsnm{Li}, \binits{Z.}},
\oauthor{\bsnm{He}, \binits{D.}},
\oauthor{\bsnm{Sun}, \binits{Z.}},
\oauthor{\bsnm{Dong}, \binits{B.}},
\oauthor{\bsnm{Qin}, \binits{T.}},
\oauthor{\bsnm{Wang}, \binits{L.}},
\oauthor{\bsnm{Liu}, \binits{T.-Y.}}:
Understanding and improving transformer from a multi-particle dynamic system
  point of view.
arXiv preprint arXiv:1906.02762
(2019)
\end{botherref}
\endbibitem

\bibitem[\protect\citeauthoryear{Sherstinsky}{2020}]{sherstinsky2020fundamentals}
\begin{barticle}
\bauthor{\bsnm{Sherstinsky}, \binits{A.}}:
\batitle{Fundamentals of recurrent neural network (rnn) and long short-term
  memory (lstm) network}.
\bjtitle{Physica D: Nonlinear Phenomena}
\bvolume{404},
\bfpage{132306}
(\byear{2020})
\end{barticle}
\endbibitem

\bibitem[\protect\citeauthoryear{Song et~al.}{2021}]{song2021scorebased}
\begin{bchapter}
\bauthor{\bsnm{Song}, \binits{Y.}},
\bauthor{\bsnm{Sohl-Dickstein}, \binits{J.}},
\bauthor{\bsnm{Kingma}, \binits{D.P.}},
\bauthor{\bsnm{Kumar}, \binits{A.}},
\bauthor{\bsnm{Ermon}, \binits{S.}},
\bauthor{\bsnm{Poole}, \binits{B.}}:
\bctitle{Score-based generative modeling through stochastic differential
  equations}.
In: \bbtitle{International Conference on Learning Representations}
(\byear{2021})
\end{bchapter}
\endbibitem

\bibitem[\protect\citeauthoryear{Monga et~al.}{2021}]{monga2021algorithm}
\begin{barticle}
\bauthor{\bsnm{Monga}, \binits{V.}},
\bauthor{\bsnm{Li}, \binits{Y.}},
\bauthor{\bsnm{Eldar}, \binits{Y.C.}}:
\batitle{Algorithm unrolling: Interpretable, efficient deep learning for signal
  and image processing}.
\bjtitle{IEEE Signal Processing Magazine}
\bvolume{38}(\bissue{2}),
\bfpage{18}--\blpage{44}
(\byear{2021})
\end{barticle}
\endbibitem

\bibitem[\protect\citeauthoryear{Huang et~al.}{2024}]{huang2024on}
\begin{barticle}
\bauthor{\bsnm{Huang}, \binits{J.}},
\bauthor{\bsnm{Gao}, \binits{Y.}},
\bauthor{\bsnm{Wu}, \binits{C.}}:
\batitle{On dynamical system modeling of learned primal-dual with a linear
  operator $\mathcal{K}$: stability and convergence properties}.
\bjtitle{Inverse Problems}
\bvolume{40}(\bissue{7}),
\bfpage{075006}
(\byear{2024})
\end{barticle}
\endbibitem

\bibitem[\protect\citeauthoryear{Lin and Wu}{2025}]{lin2022deep}
\begin{botherref}
\oauthor{\bsnm{Lin}, \binits{X.}},
\oauthor{\bsnm{Wu}, \binits{C.}}:
Deep layer limit and stability analysis of the basic forward-backward-splitting
  induced network (i): feed-forward systems.
To appear in IMA Journal of Numerical Analysis
(2025)
\end{botherref}
\endbibitem

\bibitem[\protect\citeauthoryear{Lin and Wu}{2024}]{lin2024deep}
\begin{botherref}
\oauthor{\bsnm{Lin}, \binits{X.}},
\oauthor{\bsnm{Wu}, \binits{C.}}:
Deep layer limit and stability analysis of the basic forward-backward-splitting
  induced network (ii): learning problems.
Submitted
(2024)
\end{botherref}
\endbibitem

\bibitem[\protect\citeauthoryear{Haber et~al.}{2019}]{haber2019imexnet}
\begin{bchapter}
\bauthor{\bsnm{Haber}, \binits{E.}},
\bauthor{\bsnm{Lensink}, \binits{K.}},
\bauthor{\bsnm{Treister}, \binits{E.}},
\bauthor{\bsnm{Ruthotto}, \binits{L.}}:
\bctitle{Imexnet a forward stable deep neural network}.
In: \bbtitle{International Conference on Machine Learning},
pp. \bfpage{2525}--\blpage{2534}
(\byear{2019})
\end{bchapter}
\endbibitem

\bibitem[\protect\citeauthoryear{Ruthotto and Haber}{2020}]{ruthotto2020deep}
\begin{barticle}
\bauthor{\bsnm{Ruthotto}, \binits{L.}},
\bauthor{\bsnm{Haber}, \binits{E.}}:
\batitle{Deep neural networks motivated by partial differential equations}.
\bjtitle{Journal of Mathematical Imaging and Vision}
\bvolume{62}(\bissue{3}),
\bfpage{352}--\blpage{364}
(\byear{2020})
\end{barticle}
\endbibitem

\bibitem[\protect\citeauthoryear{Buades et~al.}{2005}]{buades2005review}
\begin{barticle}
\bauthor{\bsnm{Buades}, \binits{A.}},
\bauthor{\bsnm{Coll}, \binits{B.}},
\bauthor{\bsnm{Morel}, \binits{J.-M.}}:
\batitle{A review of image denoising algorithms, with a new one}.
\bjtitle{Multiscale modeling \& simulation}
\bvolume{4}(\bissue{2}),
\bfpage{490}--\blpage{530}
(\byear{2005})
\end{barticle}
\endbibitem

\bibitem[\protect\citeauthoryear{Wei et~al.}{2019}]{wei2019regularization}
\begin{botherref}
\oauthor{\bsnm{Wei}, \binits{C.}},
\oauthor{\bsnm{Lee}, \binits{J.D.}},
\oauthor{\bsnm{Liu}, \binits{Q.}},
\oauthor{\bsnm{Ma}, \binits{T.}}:
Regularization matters: generalization and optimization of neural nets vs their
  induced kernel.
Advances in Neural Information Processing Systems
\textbf{32}
(2019)
\end{botherref}
\endbibitem

\bibitem[\protect\citeauthoryear{Esteve et~al.}{2020}]{Esteve2020LargetimeAI}
\begin{botherref}
\oauthor{\bsnm{Esteve}, \binits{C.}},
\oauthor{\bsnm{Geshkovski}, \binits{B.}},
\oauthor{\bsnm{Pighin}, \binits{D.}},
\oauthor{\bsnm{Zuazua}, \binits{E.}}:
Large-time asymptotics in deep learning.
ArXiv
\textbf{abs/2008.02491}
(2020)
\end{botherref}
\endbibitem

\bibitem[\protect\citeauthoryear{Brunner}{2017}]{brunner2017volterra}
\begin{bbook}
\bauthor{\bsnm{Brunner}, \binits{H.}}:
\bbtitle{Volterra Integral Equations: An Introduction to Theory and
  Applications}.
\bsertitle{Cambridge Monographs on Applied and Computational Mathematics}.
\bpublisher{Cambridge University Press},
\blocation{Cambridge}
(\byear{2017})
\end{bbook}
\endbibitem

\bibitem[\protect\citeauthoryear{Dragomir}{2003}]{dragomir2002some}
\begin{bbook}
\bauthor{\bsnm{Dragomir}, \binits{S.S.}}:
\bbtitle{Some Gronwall Type Inequalities and Applications}.
\bpublisher{Nova Science Publishers},
\blocation{Hauppauge, NY}
(\byear{2003})
\end{bbook}
\endbibitem

\bibitem[\protect\citeauthoryear{Braides}{2002}]{braides2002gamma}
\begin{bbook}
\bauthor{\bsnm{Braides}, \binits{A.}}:
\bbtitle{Gamma-convergence for Beginners}
vol. \bseriesno{22}.
\bpublisher{Oxford University Press},
\blocation{Oxford}
(\byear{2002})
\end{bbook}
\endbibitem

\bibitem[\protect\citeauthoryear{Adams and Fournier}{2003}]{adams2003sobolev}
\begin{bbook}
\bauthor{\bsnm{Adams}, \binits{R.A.}},
\bauthor{\bsnm{Fournier}, \binits{J.J.F.}}:
\bbtitle{Sobolev Spaces}
vol. \bseriesno{140},
\bedition{2}nd edn.
\bpublisher{Elsevier Press},
\blocation{Amsterdam}
(\byear{2003})
\end{bbook}
\endbibitem

\bibitem[\protect\citeauthoryear{Leoni}{2009}]{leoni2009first}
\begin{bbook}
\bauthor{\bsnm{Leoni}, \binits{G.}}:
\bbtitle{A First Course in Sobolev Spaces}
vol. \bseriesno{105}.
\bpublisher{American Mathematical Society},
\blocation{Providence, RI}
(\byear{2009})
\end{bbook}
\endbibitem

\end{thebibliography}

\end{document}